\pdfoutput=1
\documentclass{duke_preprint}

\usepackage{amsmath,amsfonts,bm}

\def\eqref#1{equation~\ref{#1}}

\def\1{\bm{1}}

\DeclareMathAlphabet{\mathsfit}{\encodingdefault}{\sfdefault}{m}{sl}
\SetMathAlphabet{\mathsfit}{bold}{\encodingdefault}{\sfdefault}{bx}{n}

\DeclareMathOperator{\spec}{spec}

\usepackage{hyperref}
\usepackage{url}

\usepackage{amssymb,amsthm,mathtools}
\usepackage{booktabs}
\usepackage{graphicx}
\AddToHook{env/figure/begin}{\setlength{\abovecaptionskip}{3pt}\setlength{\belowcaptionskip}{0pt}}
\usepackage{float}
\renewcommand{\bottomfraction}{0.9}
\renewcommand{\floatpagefraction}{0.75}

\newif\ifappendixcontents
\let\linkedaddcontentsline\addcontentsline
\renewcommand{\addcontentsline}[3]{%
    \ifappendixcontents\linkedaddcontentsline{#1}{#2}{#3}\fi}
\renewcommand{\eqref}[1]{\textup{(\ref{#1})}}

\theoremstyle{plain}
\newtheorem{theorem}{Theorem}
\newtheorem{lemma}[theorem]{Lemma}

\newtheorem{corollary}[theorem]{Corollary}
\newtheorem*{gmrfstatement}{Theorem~\ref{thm:gmrf-locality}}
\newtheorem*{radiusstatement}{Corollary~\ref{crl:bounded-range}}
\newtheorem{definition}{Definition}
\newtheorem{assumption}{Assumption}

\usepackage[breakable,skins]{tcolorbox}
\tcbset{statementbox/.style={
    enhanced jigsaw, breakable, lines before break=6,
    colback=white, colframe=black!35,
    boxrule=0.6pt, sharp corners,
    boxsep=0pt, left=5pt, right=5pt, top=3pt, bottom=3pt,
    before skip=6pt, after skip=6pt,
    toprule at break=0pt, bottomrule at break=0pt, pad at break=3pt,
    colbacktitle=white, coltitle=black, fonttitle=\normalfont\bfseries,
    titlerule=0pt, toptitle=3pt, bottomtitle=2pt
}}
\tcolorboxenvironment{theorem}{statementbox,title after break={Theorem~\thetheorem\ (continued)}}
\tcolorboxenvironment{lemma}{statementbox,title after break={Lemma~\thelemma\ (continued)}}
\tcolorboxenvironment{proposition}{statementbox,title after break={Proposition~\theproposition\ (continued)}}
\tcolorboxenvironment{corollary}{statementbox,title after break={Corollary~\thecorollary\ (continued)}}
\tcolorboxenvironment{gmrfstatement}{statementbox,title after break={Theorem~\ref{thm:gmrf-locality}\ (continued)}}
\tcolorboxenvironment{radiusstatement}{statementbox,title after break={Corollary~\ref{crl:bounded-range}\ (continued)}}
\tcolorboxenvironment{definition}{statementbox,title after break={Definition~\thedefinition\ (continued)}}
\tcolorboxenvironment{assumption}{statementbox,title after break={Assumption~\theassumption\ (continued)}}
\tcolorboxenvironment{remark}{statementbox,title after break={Remark~\theremark\ (continued)}}

\newcommand{\mbb}[1]{\mathbb{#1}}
\newcommand{\mcl}[1]{\mathcal{#1}}

\newcommand{\mrm}[1]{\mathrm{#1}}

\newcommand{\opn}[1]{\operatorname{#1}}

\newcommand{\N}{\mathcal{N}}

\title{How Do DiTs Denoise? An Aggregator--Predictor View of Attention and MLPs}

\author[1]{Gyeonghun Kang}
\author[2]{William Xingxu Chen}
\author[3]{Xiang Cheng}
\affiliation[1]{Department of Statistical Science, Duke University\quad\email{gk149@duke.edu}}
\affiliation[2]{Department of Computer Science, Duke University\quad\email{wc187@duke.edu}}
\affiliation[3]{Department of Electrical and Computer Engineering, Duke University\quad\email{xiang.cheng@duke.edu}}
\date{\today}
\metadata[Code]{\url{https://github.com/hun-learning94/dit-denoise}}
\hypersetup{
    pdftitle={How Do DiTs Denoise? An Aggregator--Predictor View of Attention and MLPs},
    pdfauthor={Gyeonghun Kang, William Xingxu Chen, Xiang Cheng},
    pdfsubject={Diffusion transformers and local denoising}
}

\abstract{
How do diffusion transformers (DiTs) represent image content and learn rules for reusing it to generalize?
We develop an aggregator--predictor view of DiTs as \emph{input-adaptive local denoisers}, in which context aggregation and patch prediction provide functional roles for interpreting attention and MLPs.
We first establish that, for Gaussian Markov random field (GMRF) priors, the global Bayes denoiser can be approximated using only a subset of context patches, with error decaying exponentially in context radius.
We further show that patchwise linear denoisers, often invoked to explain diffusion generalization, recover the empirical local denoiser with sufficiently informative aggregation.
Within this view, AdaLN-modulated attention naturally parameterizes an input- and noise-dependent aggregator, while the MLP admits a learned key--value representation for tokenwise patch prediction.
Our experiments support this view.
DiTs trained on real images exhibit noise-dependent denoising contexts resembling those in the analytic GMRF setting, while attention mass exhibits a qualitatively corresponding broadening pattern.
In a controlled CelebA64 study, shared-memory patch predictors used in place of DiT MLPs achieve competitive denoising and generative performance.
Finally, cross-dataset component transplantation reveals compatibility patterns consistent with the proposed roles of QK in spatial weighting, VO in feature transport, and MLPs in patch prediction.
Together, these results provide a component-level account of DiT denoising and motivate architecture-specific fine-tuning and controlled intervention.
}

\begin{document}

\maketitle

\begingroup
\setlength{\intextsep}{4pt}
\begin{figure}[t]
    \setlength{\abovecaptionskip}{3pt}
    \centering
    \includegraphics[width=\textwidth]{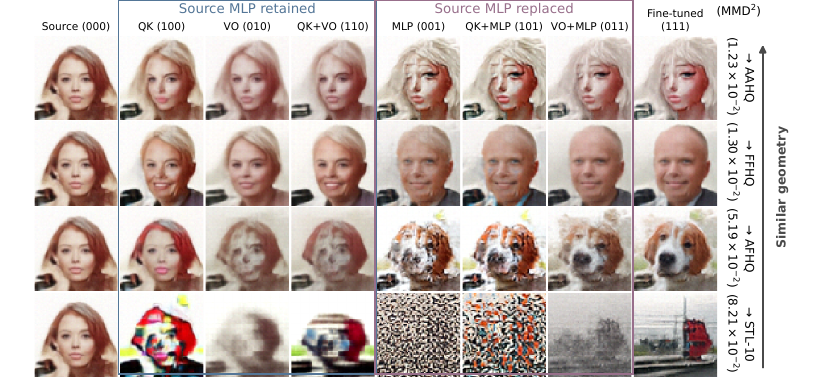}
    \vspace{-0.5cm}
    \caption{In each row, a CelebA source checkpoint is fine-tuned on the AAHQ, FFHQ, AFHQ, or STL-10 dataset.
    The source and fine-tuned generated images are shown in columns (000) and (111).
    In the middle columns, attention QK/VO and MLP weights of the source checkpoint are replaced with the corresponding fine-tuned weights; e.g., (001) means only MLP is replaced, while QK/VO are retained.
    \textbf{For face-to-face fine-tuning, replacing only QK generates source-like faces, while replacing MLP brings visual changes: cartoon faces, bald heads, or dogs.}
    Among the tested CelebA$\to$STL-10 hybrids, coherent target-like structure appears primarily in the joint endpoint (111); reverting VO (101) leaves fragmented texture.
    See Sec.~\ref{sec:component-roles-experiment} and Appendix~\ref{app:exp3}.}
    \label{fig:transfer-flagship}
\end{figure}
\endgroup

\section{Introduction}
Diffusion and flow-matching models are widely used for image generation \citep{song2021score,lipman2023flow}, with Diffusion Transformers (DiTs) becoming a standard backbone
\citep{peebles2023scalablediffusionmodelstransformers,esser2024scaling,xie2025sana,LiHe2026BackToBasics}.
Growing evidence connects their generalization to \emph{representations of images} and \emph{rules for reusing them}.
Diffusion features encode semantic and geometric information \citep{baranchuk2022label,tang2023emergent}, and aligning them with pretrained representations improves training \citep{yu2025representation}.
Models also generate combinations absent from the training data, as shown by analyses of synthetic hierarchical data \citep{okawa2023compositional,favero2025compositional}.
How components of DiTs realize these capabilities remains unclear.

Existing explanations connect diffusion-model generalization to geometry-adaptive harmonic or balanced representations \citep{kadkhodaie2024generalization, zhang2026generalization}, and inductive biases toward spatial locality and translation equivariance \citep{kamb2025analytic,niedoba2025towards}.
These accounts leave open how DiT components divide the denoising computation.
\citet{rosu2025softmaxscore} show that attention can implement exact score-based denoising, but they consider whole-image tokens, leaving interactions among spatial patches unresolved.
\citet{an2025inductive} associate DiT generalization with attention mimicking CNN-like spatial locality, but attention can also select input-dependent, spatially non-contiguous context (Sec.~\ref{sec:real-image-locality}).
We ask: \emph{How do attention, MLPs, and adaptive layer normalization (AdaLN) divide the work of representing image content and learning context-dependent rules for generalization?}

Understanding these component roles matters for intervention: domain adaptation, concept editing, and controlled generation require deciding which computations to modify and which to preserve.
In language models, interpreting MLPs as key--value memories has motivated methods for editing stored associations \citep{geva2021transformer,meng2022locating,meng2023mass}, with related analyses in multimodal transformers \citep{pan2024finding,basu2024understanding}.
A functional account of DiTs could similarly guide component-wise adaptation across domains and interventions that change image content while preserving desired spatial relationships.
Related work is reviewed in Appendix~\ref{sec:literature}.

We address this question by proposing an aggregator--predictor view of DiTs as \emph{input-adaptive local denoisers}, supported by a statistical foundation and a set of architectural and empirical observations:

\textbf{(1) A statistical foundation for local denoising.}
    We justify local denoiser, in which each patch is denoised using only a subset of context patches, by identifying its statistical target and establishing when it approximates the global Bayes denoiser.
    Specifically, for Gaussian Markov random field (GMRF) image priors, the global denoiser can be approximated by a radius-$r$ local denoiser with error decaying exponentially in $r$ (Thm.~\ref{thm:gmrf-locality}).
    We further connect existing patchwise retrieval-based explanations of generalization to this target by showing that, with sufficiently informative aggregation, a patchwise linear denoiser exactly recovers the empirical local denoiser (Thm.~\ref{thm:pld-is-local}).

\textbf{(2) An aggregator--predictor interpretation of DiT components.}
    We decompose the local denoising into \emph{context aggregation} followed by \emph{patch prediction} (Def.~\ref{dfn:agg-pre}).
    Under this view, the algebraic structure of AdaLN-modulated attention implements an input- and noise-dependent aggregation rule: QK parameterizes spatial compatibility weights, while VO transforms the corresponding features (Sec.~\ref{sec:attention-aggregator}).
    In turn, MLP admits a key--value memory interpretation and operates on the resulting contextualized token (Sec.~\ref{sec:patch-predictors}).
    These algebraic correspondences are exact for the isolated branches considered in Sec.~\ref{sec:dit-decomposition}; in a multilayer residual stack, the proposed roles interact and need not remain functionally separable (Appendix~\ref{app:residual-stack}).
    This view extends existing fixed spatial-locality accounts by allowing the relevant context to depend on the image and noise level and to include spatially non-contiguous patches.
    It further suggests the possibility of replacing DiT MLPs with more interpretable memory-based components, as supported by our experiments (Sec.~\ref{sec:shared-predictor-experiment}).

\textbf{(3) Empirical evidence consistent with the proposed component roles.}
    Across real-image datasets, DiT sensitivity fields exhibit noise-dependent context broadening resembling the GMRF setting, while attention mass shows a related qualitative trend (Sec.~\ref{sec:real-image-locality}).
    We further construct explicit shared-memory patch predictors that retain competitive denoising and generation performance in a controlled CelebA64 study (Sec.~\ref{sec:shared-predictor-experiment}).
    Finally, cross-domain component transplantation reveals systematic component-compatibility patterns associated with source--target spatial similarity. These patterns are consistent with, but do not uniquely identify, the proposed roles in context aggregation and patch prediction (Sec.~\ref{sec:component-roles-experiment}).

\section{Background and Problem Setup}
\label{sec:setup}
Let $X_0=(X_{0,1},\ldots,X_{0,M})\in\mbb{R}^{Md_p}$ be a clean image with $M$ patches, where $X_{0,m}\in\mbb{R}^{d_p}$ is patch $m$.
At time $t$, consider $X_t=a_tX_0+b_t\varepsilon$, with $a_t,b_t>0$ and $\varepsilon\sim\N(0,I_{Md_p})$ independent of $X_0$.
This formulation includes variance-exploding (VE) and variance-preserving (VP) diffusion \citep{song2021score}, as well as linear Gaussian interpolation paths used in flow matching \citep{lipman2023flow}.
We define the normalized observation $Y_t:=X_t/a_t=X_0+\sigma_t\varepsilon$ and the effective noise level $\sigma_t:=b_t/a_t$.
Under standard squared-loss objectives, population-optimal clean-image, noise, and velocity predictions are related by known affine transformations to the denoiser \citep{LiHe2026BackToBasics}
\begin{align}
    D_{m,t}^\ast(y):=\mbb{E}[X_{0,m}\mid Y_t=y,t],
    \qquad m\in\{1,\ldots,M\}.
    \label{eq:population-patch-denoiser}
\end{align}
Throughout the paper, unless otherwise specified, expectations involving $Y_t$ are over $X_0\sim p_0$ and the independent Gaussian corruption.
We call \eqref{eq:population-patch-denoiser} a population \emph{global denoiser} because the prediction conditions on all noisy patches.

For training images $\{x^{(i)}\}_{i=1}^n$ and the empirical prior $\hat p_0=n^{-1}\sum_{i=1}^n\delta_{x^{(i)}}$, the empirical version of the global denoiser \eqref{eq:population-patch-denoiser} is a weighted sum of whole images:
\begin{align}
    \hat D_{m,t}(y)&=\sum_{i=1}^n\opn{sm}_i\left\{-\frac{1}{2\sigma_t^2}
    \left\| y-x^{(j)} \right\|^2
    \right\}_{j\in[n]}x_m^{(i)},
    \label{eq:empirical-image-denoiser}
\end{align}
where the softmax $\opn{sm}$ is taken over the $n$ training images.
This creates a tension between training-loss minimization and generalization in diffusion models: under exact reverse dynamics and the corresponding terminal initialization, the empirical denoiser reproduces training images at the clean endpoint \citep{biroli2024dynamical,kamb2025analytic}.
Even when these conditions are not met, the shared mixture weights across patches keep the full-image prediction in the convex hull of whole training images, restricting its ability to recombine their local content independently \citep{kamb2025analytic,niedoba2025towards}.

This raises the question of why trained models often generalize beyond memorizing training images.
Several works attribute this to compositional generation, in which retrieval weights vary across patch locations \citep{niedoba2025towards,kamb2025analytic,lukoianov2025locality}.
We formalize the estimator underlying the compositional accounts as a \emph{patchwise linear denoiser (PLD)}.
PLD depends on $y$ only through the context patches $y_{\mcl{B}_{m,t}}$ and allows different output patches to combine content from different training images.
It is linear in how the context patches are aggregated.

\begin{definition}[Patchwise linear denoiser (PLD)]\label{def:pld}
    Fix a context set $\mcl{B}_{m,t}\subseteq\{1,\ldots,M\}$ and a matrix $S_{m,t}\in\mbb{R}^{d_h\times Md_p}$ with full row rank whose $l$th block of columns, $[S_{m,t}]_l\in\mbb{R}^{d_h\times d_p}$, is zero for $l\notin\mcl{B}_{m,t}$.
    Given a symmetric positive definite $\Gamma_{m,t}\in\mbb{R}^{d_h\times d_h}$, define the patchwise linear denoiser (PLD) by
    \begin{align}
        \hat D_{m,t}^{\mrm{pld}}(y)
        =\sum_{i=1}^n\opn{sm}_i\left\{-\frac{1}{2}
        \left\| S_{m,t}(y-x^{(j)}) \right\|^2_{\Gamma_{m,t}}
        \right\}_{j\in[n]}x_m^{(i)},
        \label{eq:patchwise-local-denoiser}
    \end{align}
    where $\|u\|_{\Gamma_{m,t}}^2:=u^\top\Gamma_{m,t}u$ and $x_m^{(i)}$ is patch $m$ of training image $i$.
\end{definition}

However, it is not clear when PLD achieves low population denoising risk or how a neural architecture can implement it.
In the following sections, we address three central questions: what statistical target does the PLD approximate, when is that target close to the population denoiser, and how can a DiT implement such an estimator?

\section{Local denoiser can approximate global denoiser}

We first show when global denoising can be accurately approximated using only a local subset of patches.
Suppose that the clean patches jointly follow a Markov random field (MRF) with respect to a graph $G_0=([M], E_0)$, and let $B_r(m)\subseteq [M]$ be a \emph{graph-radius context set}, containing all patches within graph distance $r$ of patch $m$.
This subset need not be spatially contiguous (Fig.~\ref{fig:exp1-1-synthetic}(a)).
Define
\begin{align*}
    D_{m,t}^\ast (y):= \mbb{E}\!\left[X_{0,m} \mid Y_t=y,t\right],\qquad 
    D_{m,t}^{\ast,r} (y):= \mbb{E}\!\left[X_{0,m} \mid Y_{t,B_r(m)}=y_{B_r(m)},t\right]
\end{align*}
as the \emph{global denoiser} and \emph{$B_r(m)$-local denoiser}, respectively.
We define their mean-squared discrepancy $\mathbb{E}\!\left\|D_{m,t}^\ast(Y_t) - D_{m,t}^{\ast,r}(Y_t)\right\|_2^2$ as the \emph{locality error} and say that $p_0$ is \emph{exponentially local} with respect to $G_0$ if the following holds.
\begin{assumption}[Exponential locality]\label{assm:exponential-locality}
  For every $t$, there exists a constant $c_t>0$ such that, for every patch $m$ and integer $r\geq 0$, $\mathbb{E}\left\| D_{m,t}^\ast(Y_t) - D_{m,t}^{\ast,r}(Y_t) \right\|_2^2 \lesssim_t \exp(-c_t r)$.
\end{assumption}

An explicit example is a Gaussian MRF image prior $X_0\sim\mcl{N}(0,Q^{-1})$ with positive definite block precision $Q\in\mbb{R}^{Md_p\times Md_p}$ supported on $G_0$.
Here, the global denoiser has the closed form $D_{m,t}^\ast(y)=[(I+\sigma_t^2Q)^{-1}y]_m$.
Although this denoiser generally depends on all noisy patches, analyticity of the inverse function and the graph sparsity of $Q$ allow approximation using only radius-$r$ context, with error decaying exponentially in $r$ \citep{benzi2007decay}.

\begin{figure}[t]
    \centering
    \includegraphics[width=0.9\textwidth]{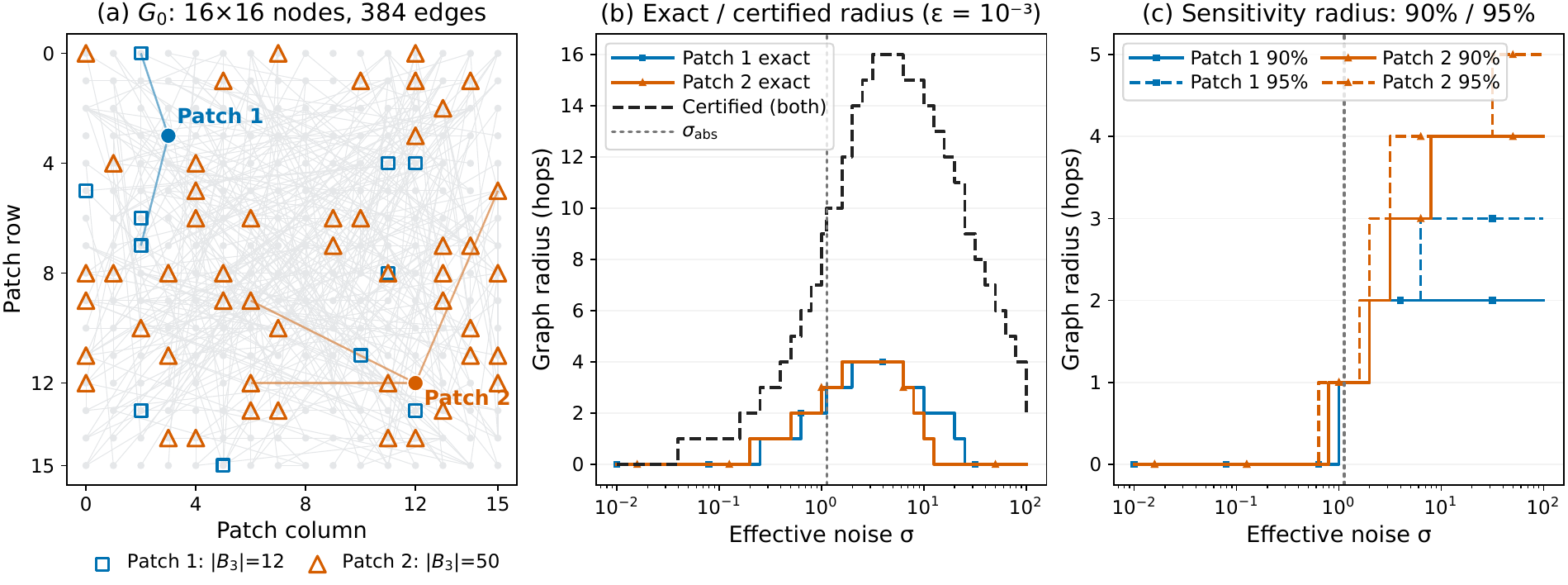}
    \vspace{-0.25cm}
    \caption{\textbf{Exact context set in a Gaussian MRF prior.}
    The known graph and precision matrix allow exact analysis of context set in this synthetic setting.
    (a) Radius-3 contexts $B_3(m)$ for two selected patches. Defined by graph distance, these sets need not be spatially contiguous.
    (b) Exact and certified radii at $\epsilon=10^{-3}$. Corollary~\ref{crl:bounded-range} guarantees that the certified radius is nondecreasing up to $\sigma_{\mrm{abs}}$ (dotted lines); the exact radii can be substantially smaller.
    (c) 90\%/95\% sensitivity radii (Sec.~\ref{sec:real-image-locality}), showing that relative influence spreads over a broader context as noise increases.
    Details in Appendix~\ref{app:exp1-1}.}
    \label{fig:exp1-1-synthetic}
\end{figure}

\begin{theorem}[Exponential locality of GMRF, informal]\label{thm:gmrf-locality}
A Gaussian MRF is exponentially local in the sense of Assump.~\ref{assm:exponential-locality}.
The full bound and conditions are in Appendix~\ref{app:gmrf-locality}.
\end{theorem}

Intuitively, noisier patches provide weaker evidence individually, making broader context useful for denoising.
The following guarantee formalizes this intuition for the bound; the optimal local denoiser may achieve the same tolerance with less context (Fig.~\ref{fig:exp1-1-synthetic}(b)).

\begin{corollary}[Bounded-range monotonicity of the locality radius, informal]\label{crl:bounded-range}
For any error tolerance $\epsilon>0$, the smallest radius $r_\epsilon(\sigma_t)$ certified by the explicit exponential bound is nondecreasing in noise on $0\leq\sigma_t\leq\sigma_{\mrm{abs}}$, where $\sigma_{\mrm{abs}}>0$ depends on $Q$.
The full statement, radius formula, and noise threshold are in Appendix~\ref{app:bounded-range}.
\end{corollary}

This guarantee concerns a bounded noise range: at extreme noise, the posterior mean approaches the zero prior mean and the absolute bound vanishes.
Fig.~\ref{fig:exp1-1-synthetic} illustrates the context $B_r(m)$ defined by graph distance and the radius guarantee in Cor.~\ref{crl:bounded-range} using a synthetic GMRF prior.
Similar locality patterns appear in DiTs trained on real images (Sec.~\ref{sec:real-image-locality}).

\subsection{Aggregator--predictor decomposition of local denoiser}

We next interpret PLD and DiTs as estimators of the same local denoiser.
Without assuming that the image prior is an MRF, fix a patch $m$, a time $t$, and a \emph{context set} $\mcl{B}_{m,t}$.
We define the following population target as the \emph{$\mcl{B}_{m,t}$-local denoiser}
\begin{align}\label{eqn:local-target}
    D_{m,t}^{\ast,\mcl{B}}(y):= \mbb{E}[X_{0,m}\mid Y_{t,\mcl{B}_{m,t}}=y_{\mcl{B}_{m,t}},t].
\end{align}
We decompose its estimation in two stages: \emph{local aggregation of the relevant context}, followed by \emph{patch prediction from the aggregated statistic}.

Specifically, a measurable map $\mcl{S}_{m,t}:\mbb{R}^{Md_p}\to\mbb{R}^{d_h}$ is a \emph{$\mcl{B}_{m,t}$-local aggregator} if $y_{\mcl{B}_{m,t}}=y'_{\mcl{B}_{m,t}}$ implies $\mcl{S}_{m,t}(y)=\mcl{S}_{m,t}(y')$ for all $y,y'$.
It depends only on the context patches and maps them, possibly nonlinearly, to a statistic in $\mbb{R}^{d_h}$.
A measurable map $g_{m,t}:\mbb{R}^{d_h}\to\mbb{R}^{d_p}$, called a \emph{patch predictor}, then predicts the clean patch from this statistic.
For a given local aggregator, a predictor recovering the local conditional mean exists if and only if $\opn{sf}(D_{m,t}^{\ast,\mcl{B}}(Y_t))\subseteq\opn{sf}(\mcl{S}_{m,t}(Y_t))$; that is, local aggregation preserves all information needed to recover the local conditional mean.
The aggregator may discard details that are unnecessary for determining this mean.
In that case, the tower property shows that such predictor can be chosen as $g_{m,t}^\ast(s)=\mbb{E}[X_{0,m}\mid\mcl{S}_{m,t}(Y_t)=s,t]$.

\begin{definition}[Aggregator--predictor decomposition]\label{dfn:agg-pre}
    A pair of measurable maps $(\mcl{S}_{m,t}^\ast,g_{m,t}^\ast)$ is an \emph{aggregator--predictor decomposition} of $D_{m,t}^{\ast,\mcl{B}}$ if $\mcl{S}_{m,t}^\ast$ is $\mcl{B}_{m,t}$-local and, for $P_{Y_t}$-almost every $y$, $D_{m,t}^{\ast,\mcl{B}}(y) = (g_{m,t}^\ast\circ\mcl{S}_{m,t}^\ast)(y)$.
\end{definition}

This decomposition is generally nonunique but provides a framework for separately estimating the aggregator and predictor of the local denoiser \eqref{eqn:local-target}.
The PLD is one instance; for a patch $m$ and time $t$ with $\sigma_t>0$, setting $\Gamma_{m,t}=(\sigma_t^2 S_{m,t}S_{m,t}^\top)^{-1}$ in Def.~\ref{def:pld}, the PLD is decomposed into $\mcl{S}_{m,t}^{\mrm{lin}}$, the \emph{linear aggregator}, and $g_{m,t}^{\mrm{rbf}}$, the \emph{radial basis function (RBF) predictor}:
\begin{align}\label{eqn:lin-local}
    \begin{split}
        \hat{D}_{m,t}^{\mrm{pld}}(y) &= \left(g_{m,t}^{\mrm{rbf}}\circ \mcl{S}^{\mrm{lin}}_{m,t}\right)(y),\\
    \mcl{S}_{m,t}^{\mrm{lin}}(y) &= S_{m,t}y,
    \qquad
    g_{m,t}^{\mrm{rbf}}(s) = \sum_{i=1}^n
    \opn{sm}_i\left\{-\frac{1}{2}
    \left\|s-\mcl{S}_{m,t}^{\mrm{lin}}(x^{(j)})\right\|^2_{(\sigma_t^2 S_{m,t} S_{m,t}^\top)^{-1}}
    \right\}_{j\in[n]}x_m^{(i)}.
    \end{split}
\end{align}

Bayes' rule identifies the predictor $g_{m,t}^{\mrm{rbf}}(s)$ with $\mbb{E}_{\hat p_0}[X_{0,m}\mid\mcl{S}_{m,t}^{\mrm{lin}}(Y_t)=s,t]$, the conditional mean given the aggregated statistic $s$ under the empirical prior; see Appendix~\ref{app-sec:pld-is-local-proof}.
This correspondence has been noted in prior work \citep{niedoba2025towards, kamb2025analytic, lukoianov2025locality}.
In our framework, these approaches correspond to using cropping matrices, prescribed neighborhood masks, or thresholded Wiener-filter sensitivity masks for $S_{m,t}$, respectively.
We go a step further and show that, with a sufficiently expressive aggregator, $\hat{D}_{m,t}^{\mrm{pld}}$ exactly recovers the local denoiser under the empirical observation law.

\begin{theorem}[PLD recovers the empirical local denoiser]\label{thm:pld-is-local}
Under the setup above, let $\sigma_t>0$ and take $\Gamma_{m,t}=(\sigma_t^2 S_{m,t}S_{m,t}^\top)^{-1}$ in Def.~\ref{def:pld}. If $\mcl{S}_{m,t}^{\mrm{lin}}$ preserves the empirical local conditional mean, i.e., $\opn{sf}(\mbb{E}_{\hat{p}_0}[X_{0,m}\mid Y_{t,\mcl{B}_{m,t}},t])\subseteq \opn{sf}(\mcl{S}_{m,t}^{\mrm{lin}}(Y_t))$, then, almost surely under the empirical observation law, $\hat{D}_{m,t}^{\mrm{pld}}(y) = \mbb{E}_{\hat{p}_0}[X_{0,m}\mid Y_{t,\mcl{B}_{m,t}}=y_{\mcl{B}_{m,t}},t]$.
\end{theorem}

The condition holds whenever the full local observation can be recovered from the aggregate, as with a patch-selection matrix.
A compressive linear aggregator can also suffice: when the training contexts restricted to $\mcl{B}_{m,t}$ lie in a lower-dimensional affine subspace, a fixed linear aggregator whose row space contains their difference subspace is sufficient; see \eqref{app-eqn:V-cond}.

Theorem~\ref{thm:pld-is-local} therefore connects the PLD, used to explain generalization by composition, to the population local denoiser \eqref{eqn:local-target}: with sufficient aggregation, retrieval recovers its empirical counterpart exactly.
However, requiring all training-context differences to lie in a single low-dimensional subspace is restrictive.
This motivates the nonlinear DiT aggregator considered next.

\section{DiTs as Adaptive Local Denoisers}
\label{sec:dit-decomposition}

We next relate DiT attention and MLPs to the aggregator and predictor roles of the local-denoising framework.
We consider a single layer without residual connections $\hat{D}_{m,t}^{\mrm{dit}}(y) = (g_{t}^{\mrm{mlp}}\circ \mcl{S}^{\mrm{attn}}_{m,t})(y)$, where the attention aggregator $\mcl{S}^{\mrm{attn}}_{m,t}$ and MLP predictor $g_{t}^{\mrm{mlp}}$ are defined in the following sections.
Appendix~\ref{app:attention-aggregator} gives the derivation and residual-stack extension; Fig.~\ref{fig:dit-component-map} summarizes the component mapping.
Note that in a residual stack, VO and MLP updates can alter later QK routing, and attention can affect later MLP updates, so the branch-level algebra does not necessarily imply a functionally separable decomposition of the full network.

\subsection{Attention as an input-adaptive aggregator}
\label{sec:attention-aggregator}
PLD's linear aggregator \eqref{eqn:lin-local} uses a fixed context rule across inputs.
A more flexible denoiser can adapt its context to the input, drawing on different patches as their relevance changes with the image and noise level.
We express this through an \emph{input-adaptive aggregator}
\begin{align}\label{eqn:adaptive-local}
  \mcl{S}_{m,t}(Y_t) = \sum_{l=1}^M S_{ml}(Y_t;t)\Phi(Y_{t,l};t,l).
\end{align}
Here $\Phi(Y_{t,l};t,l)\in\mbb{R}^{d_h}$ represents patch $l$ together with its position and noise level, while $S_{ml}(Y_t;t)\in\mbb{R}^{d_h\times d_h}$ determines how it contributes to the context for denoising patch $m$.

DiT attention, incorporating patch, position, and time embeddings together with AdaLN conditioning, provides a factorized parameterization of such an input-adaptive aggregator.
Let $E\in\mbb{R}^{d_h\times d_p}$ be the patch projection, $e_l\in\mbb{R}^{d_h}$ the positional embedding, and $\phi_{t,l}(Y_t):=\Phi^{\mrm{attn}}(EY_{t,l};e_l,t)$ the feature formed by patch embedding, position embedding, and AdaLN.
For each of $\tilde{H}$ equal-width heads, $\tilde{W}_h^Q$, $\tilde{W}_h^K$, and $\tilde{W}_h^V$ are the query (Q), key (K), and value (V) projections, and $\tilde{W}_h^O$ is the output (O) projection.
Writing the attention AdaLN residual gate as $\tilde{\alpha}_t\in\mbb{R}^{d_h}$, the attention branch constitutes an \emph{attention aggregator} defined as
\begin{align}\label{eqn:attention}
\begin{split}
    \mcl{S}_{m,t}^{\mrm{attn}}(Y_t)&=\sum_{l=1}^M S_{ml}^{\mrm{attn}}(Y_t;t)\phi_{t,l}(Y_t),\\
    S_{ml}^{\mrm{attn}}(Y_t;t)
    &=\sum_{h=1}^{\tilde{H}}
    \opn{sm}_{l}\left\{
    \frac{\langle \tilde{W}_h^Q\phi_{t,m}(Y_t),\tilde{W}_h^K\phi_{t,j}(Y_t)\rangle}{(d_h/\tilde{H})^{1/2}}
    \right\}_{j\in[M]}
    \opn{diag}(\tilde{\alpha}_t)\tilde{W}_h^O\tilde{W}_h^V.
\end{split}
\end{align}

\textbf{A single attention branch separates QK spatial weighting from VO value transformation and mixing}~\eqref{eqn:attention}.
The QK projections determine the compatibility scores of patches through $(\tilde{W}_h^Q)^\top \tilde{W}_h^K$, and softmax converts these scores into spatial weights.
By smoothly suppressing patches with low relative scores, attention performs soft thresholding and forms an approximately local aggregator when its weights concentrate on a suitable context.
In this branch-level sense, QK parameterizes which patch positions receive greater or smaller weight in the constructed context, depending on image content, position, and noise level.
Appendix~\ref{app:attention-aggregator-3} details the local approximation.

VO transforms and combines the selected features through the shared, time-dependent channel operators $\opn{diag}(\tilde{\alpha}_t)\tilde{W}_h^O\tilde{W}_h^V$.
Attention thus combines $\tilde{H}$ shared channel operators with input-dependent spatial weights, rather than learning an independent operator for every patch pair.
This decomposition motivates an aggregator interpretation: for fixed input features, QK directly controls relative spatial weights, while VO applies patch-shared channel operators to the weighted features.
The downstream impact of an attention weight also depends on the value features, channel operators, and residual interactions.
These roles interact through learned features and across layers, but the distinction motivates examining QK and VO separately in Sec.~\ref{sec:component-roles-experiment}.

\subsection{MLP and shared-memory patch predictors}
\label{sec:patch-predictors}
Within the single-layer model, the DiT MLP $g_t^{\mrm{mlp}}$ acts on the attention-aggregated context $s=\mcl{S}_{m,t}^{\mrm{attn}}(Y_t)$ in the patch-predictor role; in a residual stack, its output is a tokenwise feature update.
The framework also allows alternative patch predictors, such as the RBF predictor $g_{m,t}^{\mrm{rbf}}$ in \eqref{eqn:lin-local}, which explicitly retrieves from training patches.
In a multilayer residual stack (Appendix~\ref{app:residual-stack}), these predictors can share memory across layers, just as $g_{m,t}^{\mrm{rbf}}$ would consult the same training patches at each layer.
This motivates a sequence of \emph{shared-memory patch predictors} connecting explicit patch retrieval to an ordinary MLP.
We present simplified forms below; Appendix~\ref{app:predictor-implementation} gives the full parameterizations, including biases, output gates, compatibility projections, and multiple heads.

We first note that $g_t^{\mrm{mlp}}$ and $g_{m,t}^{\mrm{rbf}}$ admit similar key--value memory interpretations \citep{geva2021transformer}.
If coupled with the attention aggregator, $g_{m,t}^{\mrm{rbf}}$ retrieves from patch-specific values $\{x_m^{(i)}\}_{1:n}$ using contextualized training images $\mcl{S}_{m,t}^{\mrm{attn}}(x^{(i)})$ as keys.
A transformer MLP of hidden width $R$ likewise consists of learnable key--value pairs $\{k^{(r)},v^{(r)}\in\mbb{R}^{d_h}\}_{r=1}^R$.
In DiTs, AdaLN modulates the memory query to produce time-dependent retrieval weights:
\begin{align}\label{eqn:mlp-memory}
    g^{\mrm{mlp}}_{t}(s)
    &=\sum_{r=1}^{R}\opn{gelu}\!\left(\opn{AdaLN}_t(s)^\top k^{(r)}\right)v^{(r)}.
\end{align}
Compared with $g_{m,t}^{\mrm{rbf}}$ in \eqref{eqn:lin-local}, there are three differences: (i) the memory is learned rather than drawn from training patches, (ii) the same bank of memory is used across patches, and (iii) the retrieval weights are unnormalized and may have either sign.
These differences motivate three patch-predictor variants that start from $g_{m,t}^{\mrm{rbf}}$ and progressively resemble $g_t^{\mrm{mlp}}$: we first make the memory learnable, then share it across patches, and finally replace normalized retrieval weights with unnormalized, potentially signed coefficients.

All predictors act on the same contextualized patch $s=\mcl{S}_{m,t}^{\mrm{attn}}(Y_t)$.
Let $W^{Q},W^{K},W^{V}$ be layerwise projections, with softmax taken over the $R$ memory entries.
\begin{align}
    g^{\mrm{lpld}}_{m,t}(s)\label{eqn:predictor-lpld}
    &=\sum_{r=1}^{R}
    \opn{sm}_r\left\{\opn{AdaLN}_t(s)^\top(W^{Q})^\top W^{K}\mcl{S}_{m,t}^{\mrm{attn}}(z^{(j)})\right\}_{j\in[R]}
    W^{V}z_m^{(r)},\\
    g^{\mrm{sca}}_{t}(s)\label{eqn:predictor-sca}
    &=\sum_{r=1}^{R}
    \opn{sm}_r\left\{\opn{AdaLN}_t(s)^\top(W^{Q})^\top W^{K}k^{(j)}\right\}_{j\in[R]}
    W^{V}v^{(r)},\\
    g^{\mrm{smlp}}_{t}(s)\label{eqn:predictor-smlp}
    &=\sum_{r=1}^{R}
    \opn{gelu}\!\left(\opn{AdaLN}_t(s)^\top(W^{Q})^\top W^{K}k^{(r)}\right)W^{V}v^{(r)}.
\end{align}

Compared to PLD, the first variant \eqref{eqn:predictor-lpld} replaces the training images with a learned memory bank $\{z^{(r)}\in \mbb{R}^{Md_h'}\}_{r=1}^R$ shared across layers, and RBF retrieval with softmax.
Just as a training image $x^{(i)}$ consists of $M$ patches, each memory entry $z^{(r)}$ consists of $M$ learnable patches $z_m^{(r)}\in\mbb{R}^{d_h'}$.
The memory width $d_h'$ may equal the context width $d_h$ or differ from it; a learned compatibility projection shared across layers accommodates differing widths.
Together, these choices define the \textbf{learned PLD (LPLD)} predictor $g^{\mrm{lpld}}_{m,t}$.
LPLD uses the attention-aggregated memory $z^{(r)}$ to form keys, whereas the value is taken directly from its $m$th patch.

Storing a separate set of memory patches at each patch position can be redundant, especially given that patchification is an architectural choice.
In \eqref{eqn:predictor-sca} and \eqref{eqn:predictor-smlp}, the patch-specific memory bank in $g^{\mrm{lpld}}_{m,t}$ is replaced with an untied key--value bank $\{k^{(r)},v^{(r)} \in \mbb{R}^{d_h'}\}_{r=1}^R$ shared across patches and layers.
Coupled with the softmax retrieval, this gives a \textbf{shared cross-attention (SCA)} predictor $g^{\mrm{sca}}_{t}$.
Finally, we replace normalized softmax weights with unnormalized GELU coefficients for a more flexible retrieval rule.
This gives a \textbf{shared MLP (SMLP)} predictor $g^{\mrm{smlp}}_{t}$.
Absorbing the projection matrices into the keys and values allows SMLP to recover the MLP form \eqref{eqn:mlp-memory}.
For all three variants, sharing the memory bank across layers can reduce the parameter count, allowing us to allocate more capacity to predictor heads at matched parameter budgets.

\section{Experiments}\label{sec:experiments}
Our experiments examine whether trained DiT denoisers exhibit locality patterns resembling those of the oracle GMRF denoiser (Sec.~\ref{sec:real-image-locality}), and whether, as suggested by our aggregator--predictor framework, layerwise MLPs can be replaced with more interpretable, shared-memory patch predictors (Sec.~\ref{sec:shared-predictor-experiment}).
Finally, we use component transplantation to examine whether compatibility patterns among jointly adapted QK, VO, and MLP parameters are consistent with this framework (Sec.~\ref{sec:component-roles-experiment}).
Experimental details and additional results are in Appendix~\ref{app:experimental-details}.

\subsection{Real-image denoising echoes GMRF locality}
\label{sec:real-image-locality}

\begin{figure}[t]
    \centering
    \includegraphics[width=0.95\textwidth]{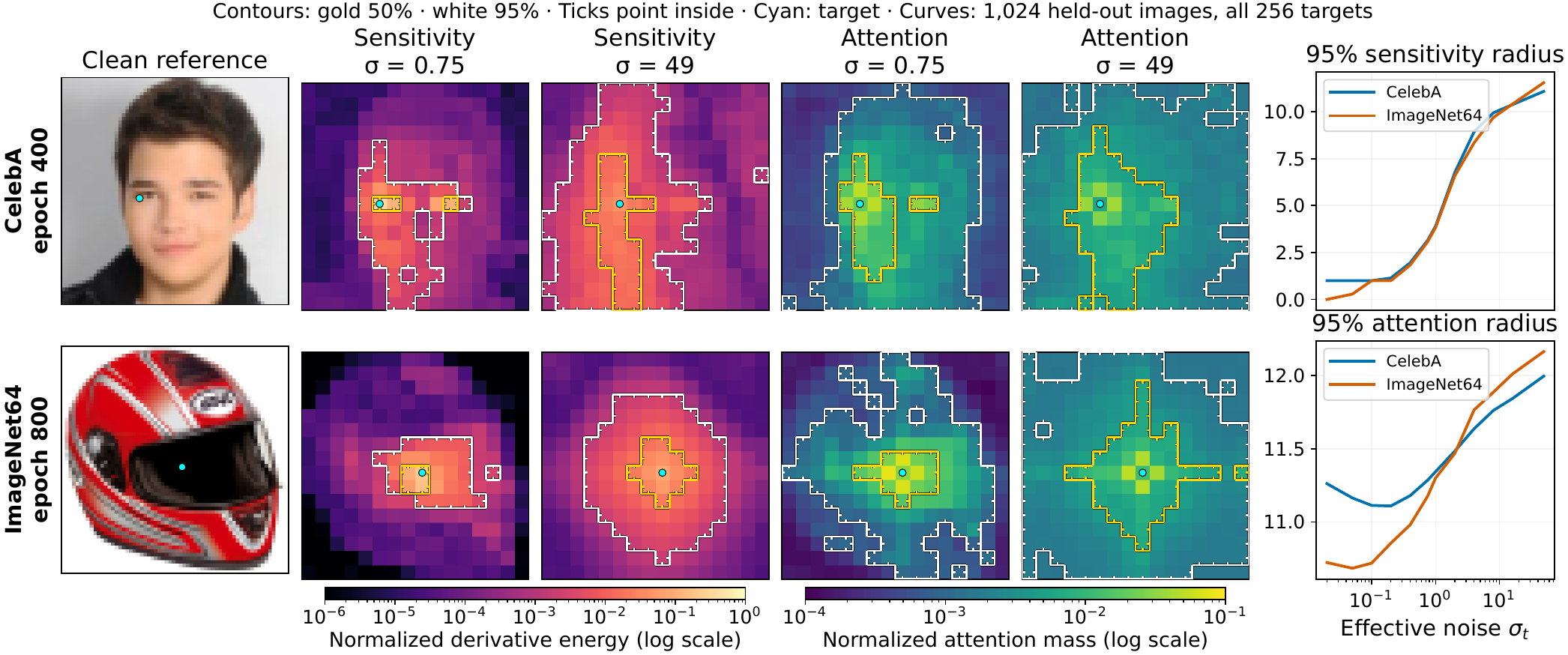}
    \vspace{-0.2cm}
    \caption{\textbf{Sensitivity-defined contexts broaden with noise, with attention mass showing a related trend.}
    Gold/white contours enclose 50\%/95\% of normalized squared sensitivity or attention mass; ticks point inward and cyan dots mark target patches.
    Right: 95\% radii averaged over target patches; Details in Appendix~\ref{app:exp1-2}.}
    \label{fig:exp1-2-real}
\end{figure}

An exact analysis as in Fig.~\ref{fig:exp1-1-synthetic} is unavailable for real images.
Instead, we identify context sets using \emph{sensitivity weights}, which quantify how the prediction of patch $m$ responds to perturbations of input patch $l$, based on a trained DiT's Jacobian, similar to \citet{niedoba2025towards}.
The population denoiser's Jacobian measures posterior dependence between clean patches given the noisy patches \citep{dytso2021general}.
Based on these weights, we define the \emph{$q$-sensitivity radius} as the smallest radius of a square centered on patch $m$ containing at least a fraction $q$ of the total weight.
The \emph{sensitivity context set} contains the fewest patches carrying at least that fraction.

\textbf{Denoising sensitivity exhibits GMRF-like broadening; attention shows a related trend.}
For the synthetic GMRF where the true Jacobian is available, Fig.~\ref{fig:exp1-1-synthetic}(c) shows that sensitivity spreads over broader contexts as noise increases.
Denoising real images echoes this pattern: Fig.~\ref{fig:exp1-2-real} shows concentrated sensitivity at low noise and broader, sometimes disconnected contexts at higher noise, together with increasing sensitivity radii.
Notably, replacing sensitivity weights with attention mass yields analogous attention contexts and radii, which also generally broaden with noise.
Together, these observations support the GMRF setting as a tractable reference model for denoising locality and the interpretation of attention as an input-adaptive context aggregator.
Sensitivity and attention mass measure different quantities, so the comparison concerns a shared qualitative trend rather than identical context sets.
Additional results appear in Appendix~\ref{app:exp1-2}.

\begin{figure}[!b]
    \centering
    \includegraphics[width=0.95\textwidth]{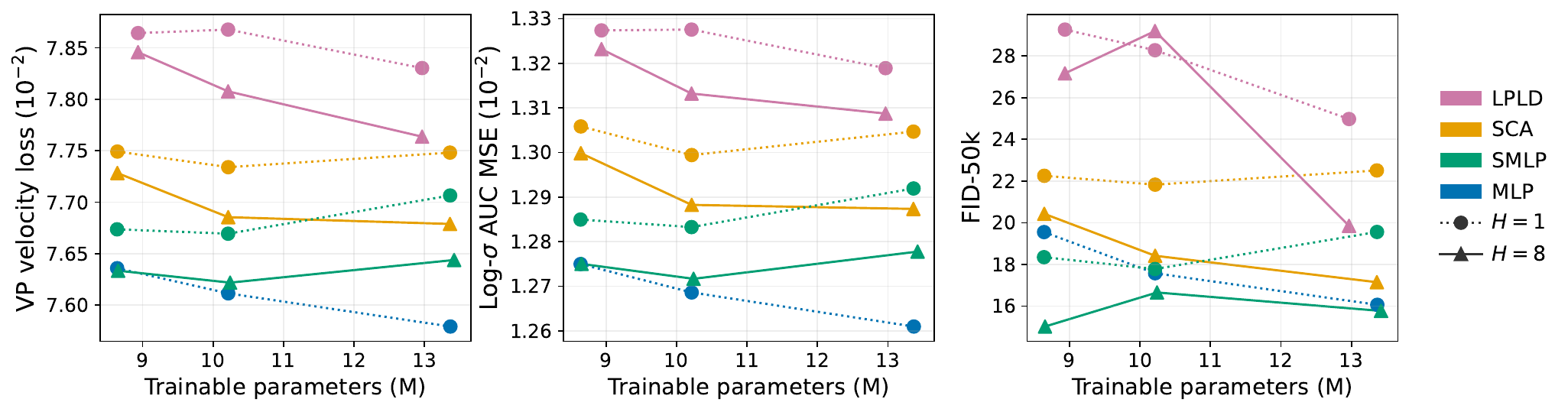}
    \vspace{-0.2cm}
    \caption{\textbf{Shared-memory patch predictors at comparable parameter budgets.}
    Models are trained on CelebA64 for 400 epochs; see Appendix~\ref{app:exp2} for details.}
    \label{fig:shared-predictor-efficiency}
\end{figure}

\subsection{Shared-memory predictors can replace MLPs in a controlled setting}
\label{sec:shared-predictor-experiment}
We replace DiT MLPs with LPLD, SCA, and SMLP on CelebA64.
We compare similar parameter budgets using a common backbone architecture, with modest models and one training seed per configuration for unconditional $64\times64$ face generation.
Figure~\ref{fig:shared-predictor-efficiency} reports held-out velocity loss, denoising error, and FID at epoch 400.
Three patterns appear; first, the patch-shared SCA and SMLP banks outperform the position-specific LPLD bank at comparable parameter budgets, showing that competitive prediction need not assign separate memory for every patch position.
Second, SMLP improves on SCA, supporting unnormalized GELU retrieval as an effective alternative to softmax.
Third, SMLP with $H=8$ achieves competitive denoising and lower FID than the ordinary MLP, despite sharing its base memory across layers.
Multiple predictor heads generally improve performance, suggesting that multiple projections of the shared bank can be useful.
Together, these results demonstrate that an explicit, layer-shared memory bank can support competitive patch prediction in this controlled setting. 
We note that these results by themselves do not establish that layerwise MLP parameters generally encode a single shared memory or that this design transfers unchanged to larger conditional DiTs.
Training details and limitations are in Appendix~\ref{app:exp2}.

\begin{figure}[!htb]
    \centering
    \includegraphics[width=0.95\textwidth]{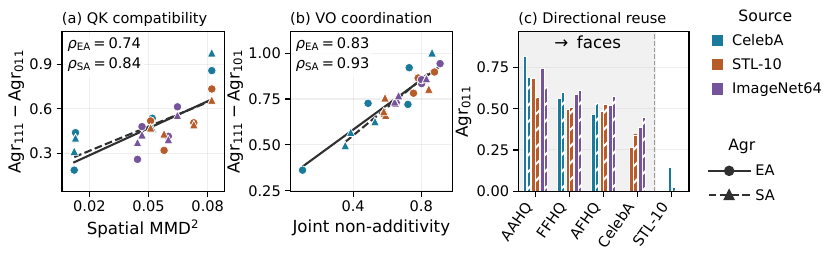}
    \vspace{-0.35cm}
    \caption{\textbf{QK, VO, and MLP component transplants.}
    (a) Source-QK compatibility is higher across spatially similar datasets.
    (b) VO reversion costs more as joint non-additivity increases.
    (c) Natural-image source QK retains partial agreement on face targets.
    See Appendix~\ref{app:exp3-endpoint-agreement}.}
    \label{fig:transfer-quantitative}
\end{figure}

\subsection{Component transplants support the aggregator--predictor view}
\label{sec:component-roles-experiment}

\textbf{Transplanting jointly adapted components probes compatibility and coadaptation.}
To examine compatibility patterns in relation to the proposed component roles, we fine-tune a pretrained source model on a target dataset.
We then transplant combinations of the adapted components into the source checkpoint. 
When the datasets have similar spatial structure, our framework suggests that source QK may remain compatible with adapted VO and MLP parameters.
Larger spatial shifts may be associated with stronger coadaptation.
We study twelve transfer directions between face and natural-image datasets.
For clarity, we jointly fine-tune only QK, VO, and MLP, leaving AdaLN and all other parameters fixed, assuming that the learned time conditioning transfers across domains. 
Each of the eight resulting hybrids is labeled by QK/VO/MLP bits: 0 for source weights and 1 for jointly adapted weights.
These transplants probe the compatibility and coadaptation of parameters learned jointly; they do not measure the performance obtainable by training QK-, VO-, or MLP-only models. 
Separately, under the same target-domain update budget, QK-only and QK+VO-only fine-tuning generally fall further behind joint adaptation on larger domain shifts (Appendix~\ref{app:exp3-trajectories}).

Following work on patch self-similarity and spatial normalization \citep{tumanyan2022splicing,singh2025spatialalignment}, we measure spatial distance between datasets using MMD$^2$ of patch self-similarity descriptors.
We compare each hybrid's outputs with the fully adapted model's using matched initial noise.
We measure this agreement using DINOv2 global features for endpoint agreement (EA) and patch self-similarity patterns for spatial agreement (SA). 
Both scores assign 0 to the source and 1 to the fully adapted model.
Full protocols and results are in Appendix~\ref{app:exp3}.

\textbf{QK compatibility is higher for spatially similar datasets, while VO reversion penalties track component coadaptation.}
With source QK retained (011), CelebA$\to$AAHQ/FFHQ/AFHQ produces recognizable faces. 
In contrast, among the tested CelebA$\to$STL-10 hybrids, coherent target-like structure emerges primarily only when all jointly adapted components are present (111; Figs.~\ref{fig:transfer-flagship} and~\ref{fig:exp3-extra-celeba64}).
Consistent with this pattern, the QK reversion penalty tends to increase with spatial distance (Fig.~\ref{fig:transfer-quantitative}(a)), which also accompanies stronger component coadaptation (Appendix~\ref{app:exp3-results}). 
Additionally, restoring source VO (101) yields fragmented textures in the natural-image-to-face transfers (Figs.~\ref{fig:transfer-stl10} and~\ref{fig:transfer-imagenet64}).
Its reversion penalty increases with joint non-additivity—the gain from joint adaptation beyond the sum of isolated component changes (Fig.~\ref{fig:transfer-quantitative}(b)).

\textbf{Source QK from broader natural-image domains retains partial compatibility on face targets.}
With source QK retained (011), STL-10 hybrids recover human and animal face structure, including a blurred face for STL-10$\to$CelebA, while joint adaptation produces clearer results (Fig.~\ref{fig:transfer-stl10}).
This contrasts with the poorly formed scene in the reverse CelebA$\to$STL-10 direction (Fig.~\ref{fig:transfer-flagship}); additional examples in Appendix~\ref{app:exp3-generated-cubes}.
Figure~\ref{fig:transfer-quantitative}(c) supports this asymmetry: source QK from STL-10 retains greater endpoint agreement on CelebA than in the reverse direction, while QK from natural-image domains retains partial agreement on face targets at levels comparable to CelebA QK.

\section{Conclusion and limitations}
\label{sec:conclusion}
We presented an aggregator--predictor view of diffusion transformers as input-adaptive local denoisers.
Our theory justifies local denoising under GMRF priors and connects patchwise linear denoisers to the empirical local denoising target.
Within this framework, QK determines spatial context weights, VO transports features, and MLPs perform contextualized patch prediction.
Experiments provide complementary evidence through noise-dependent denoising contexts, competitive shared-memory predictors, and cross-domain component compatibility patterns.
Together, these findings offer a statistical and architectural perspective on DiT denoising and motivate further work on interpretable predictors and targeted model adaptation.

Our experiments focus on modest-sized, unconditional DiTs operating on $64\times64$ images across a limited collection of face and natural-image datasets.
In particular, the shared-memory predictor comparison uses CelebA64 with one training seed per configuration, leaving its robustness across datasets, model scales, and training seeds unresolved.
The transfer study examines twelve directions while keeping AdaLN and global parameters fixed; its component transplants probe compatibility among jointly adapted weights and do not uniquely identify the proposed functional roles.
Testing these patterns in larger conditional models, at higher resolutions, and under broader adaptation protocols is an important next step.

\subsection*{AI use statement}
Generative AI tools were used as research assistants for literature review, refining the presentation of the work, writing and cleaning up code, assisting with experiment execution, cleaning up proofs, providing feedback on the manuscript, and polishing the text.
All AI-assisted results were independently reviewed and verified by the authors.
The authors made the final scientific and methodological decisions and take full responsibility for the contents of the paper.

\subsection*{Reproducibility statement}
Assumptions and proofs are in Appendix~\ref{app:proofs}.
Appendix~\ref{app:experimental-details} documents the datasets, model configurations, training and evaluation protocols, and metric definitions for all three experiments.
We provide \href{https://github.com/hun-learning94/dit-denoise}{code and reproduction instructions} for dataset preparation, training, evaluation, and regeneration of the figures and tables, together with the small numerical summaries used to render the reported plots.

\bibliography{references}
\bibliographystyle{plainnat}

\appendix
\appendixcontentstrue
\renewcommand{\contentsname}{Appendix contents}
\tableofcontents

\section{Related Literature}
\label{sec:literature}

\paragraph{Representations and compositional reuse.}
Diffusion models learn internal features useful for semantic segmentation and discriminative probing \citep{baranchuk2022label,xiang2023denoising}, and semantic and geometric correspondence can emerge without correspondence supervision \citep{tang2023emergent}.
\citet{chen2025deconstructing} connect this ability to classical denoising autoencoders by simplifying diffusion models and studying which components support self-supervised representation learning.
\citet{han2025feature} provide theoretical and empirical evidence that denoising encourages representations covering multiple data features, including features that classifiers may neglect.
REPA connects representation quality to generative training by aligning hidden states with pretrained visual features in diffusion and flow-based transformers \citep{yu2025representation}.
\citet{singh2025spatialalignment} further identify the spatial structure of the alignment target as a strong predictor of generation performance.

Compositional generation provides evidence of reuse: \citet{okawa2023compositional} demonstrate unseen combinations in a synthetic task, and \citet{favero2025compositional} analyze the learning of composition rules in a hierarchical grammar, where longer-range dependencies require more data.
The extent of this ability depends on the task; \citet{han2025hidden} find that diffusion models can learn coarse relationships while failing finer dependencies, including reflection and lighting consistency.
These findings motivate the introduction's representation-and-reuse framing.
In our denoising account, a reuse rule has a specific meaning: it determines which context patches supply evidence for predicting a clean patch from a new noisy image.

\paragraph{Memorization and generalization.}
Under a finite empirical prior, the squared-error denoising objective is minimized by a weighted average of training images, with weights determined by whole-image comparisons as in \eqref{eq:empirical-image-denoiser}.
With the corresponding terminal initialization, exact reverse dynamics recover the empirical distribution \citep{pidstrigach2022score,yoon2023diffusion,gu2023memorization,biroli2024dynamical,kamb2025analytic}.
Several explanations address why learned denoisers nevertheless generalize.
Training dynamics can favor generalizing solutions before memorization, with larger datasets delaying memorization \citep{bonnaire2025implicit}; large learning rates can also steer training away from the memorizing solution \citep{wu2025taking}.
Other accounts characterize the learned estimator: \citet{kadkhodaie2024generalization} connect generalization to geometry-adaptive harmonic representations, while \citet{zhang2026generalization} prove a distinction between sample storage and balanced representations of local statistics for a two-layer ReLU denoising autoencoder, and investigate analogous structure in larger models.
Our framework concerns how DiT components implement denoising computations compatible with such statistical structure.

\paragraph{Local denoising and its statistical target.}
Local empirical denoisers explain novel generation through patch recombination.
\citet{niedoba2025towards} use cropped neighborhoods to approximate trained denoisers; \citet{kamb2025analytic} derive estimators from locality and translation equivariance; and \citet{lukoianov2025locality} choose context from the sensitivity of the optimal linear denoiser, emphasizing the role of data statistics.
These constructions motivate the \emph{patchwise linear denoiser (PLD)} in Section~\ref{sec:setup}.
In \eqref{eq:patchwise-local-denoiser}, a softmax comparison through $S_{m,t}$ determines how stored training patches $x_m^{(i)}$ contribute to a prediction.
Setting $S_{m,t}=I$ and $\Gamma_{m,t}=\sigma_t^{-2}I$ recovers whole-image retrieval; local contexts allow the weights to differ across patches and hence permit recombination.

Our statistical contribution identifies a population target for this construction.
Theorem~\ref{thm:gmrf-locality} bounds the discrepancy between population local and global denoisers under a GMRF prior.
Theorem~\ref{thm:pld-is-local} identifies when sufficient aggregation lets a PLD recover the empirical local conditional mean exactly; this identity alone is not a finite-sample bound on population risk.
Together, the results motivate learning both the context aggregator and the patch predictor.
The stored patches and context comparisons in the PLD provide an explicit example of content and reuse; Section~\ref{sec:dit-decomposition} develops their learned counterparts in a DiT.

\paragraph{Architectural accounts of denoising.}
\citet{mei2024unet} show how U-Net encoder--decoder paths and skip connections can implement belief-propagation denoising in hierarchical graphical models, providing a direct precedent for interpreting architecture through a statistical denoising algorithm.
For transformers, \citet{rosu2025softmaxscore} show that attention can implement exact score-based denoising; their image experiments use whole-image tokens, although an anisotropic extension permits patchwise pixel updates.
\citet{an2025inductive} associate DiT generalization with local attention in early layers and show benefits from restricting attention windows with limited training data.
\citet{gan2025unleashing} link concentrated large feature activations to AdaLN and use channel modulation to improve DiT features for correspondence.
These works explain architectural or conditioning mechanisms relevant to denoising and representation learning.
Our account connects attention, MLPs, and AdaLN through context aggregation followed by patch prediction, allowing context to vary with the input and noise level, including the non-contiguous patterns in Section~\ref{sec:real-image-locality}.

\paragraph{MLPs as learned memory.}
Language-model MLPs have been interpreted as key--value memories that store and retrieve learned associations \citep{petroni2019language,geva2021transformer,geva2022transformer}; \citet{sukhbaatar2019augmenting} express them as attention over persistent memory vectors.
This interpretation has motivated locating and editing facts through MLP weights \citep{dai2022knowledge,meng2022locating,meng2023mass}, with related analyses and interventions in multimodal models \citep{schwettmann2023multimodal,pan2024finding,basu2024understanding,pan2024unike}.
At the level of the generative process, \citet{hoover2023memory} survey connections between diffusion dynamics and retrieval in energy-based associative memories.
Our DiT account assigns a memory role specifically to the MLP's patch prediction from aggregated context (Section~\ref{sec:patch-predictors}).
Replacing layerwise MLPs with explicit memory shared across layers tests this interpretation at the level of the predictor component (Section~\ref{sec:shared-predictor-experiment}).
The representation may be distributed across learned memory entries; the interpretation does not require individual entries to correspond to named concepts.

\paragraph{Localization, editing, and component interventions.}
Several analyses of text-to-image diffusion localize visual attributes in text-encoder self-attention or U-Net cross-attention \citep{basu2024localizing,basu2024mechanistic}, including neurons associated with concepts or memorized training images \citep{liu2023cones,hintersdorf2024finding}.
For DiTs, \citet{zarei2025localizing} identify blocks associated with specific concepts and use localized updates for personalization and unlearning.
Their component ablations also intervene on self-attention, normalization, and MLP outputs, with cross-attention giving the strongest localization signal in the tested PixArt setting.
Thus, component intervention already has direct precedents in DiTs.

Related editing methods illustrate the practical value of these distinctions.
Custom Diffusion adapts selected parameters for concept customization \citep{kumari2023customdiffusion}; Plug-and-Play Diffusion Features injects spatial features and self-attention to preserve layout during translation \citep{tumanyan2023plugandplay}; and \citet{shin2025exploring} analyze text--image attention blocks to design editing interventions for multimodal DiTs.
\citet{turri2026few} further show that transferring high-magnitude hidden-state channels between prompt-conditioned trajectories can steer image semantics in DiT and related flow-based models.
These activation interventions complement our parameter transplants, which ask which QK, VO, and MLP computations remain compatible across image domains, guided by the aggregator--predictor account (Section~\ref{sec:component-roles-experiment}).
The transplants probe jointly adapted components with AdaLN fixed, so their outcomes measure reuse and coadaptation under that intervention protocol (Appendix~\ref{app:exp3}).

\paragraph{Gaussian and Markov image priors.}
The GMRF prior in Theorem~\ref{thm:gmrf-locality} combines joint Gaussian statistics with conditional independence between non-neighboring patches given their graph neighbors.
Markov image models developed from spatial statistics \citep{besag1974spatial} into methods for restoration, texture modeling, and segmentation \citep{geman1984stochastic,li2009markov}, with higher-order models extending interactions to larger neighborhoods and patches \citep{roth2005fields,weiss2007what,chen2014insights}.
Gaussian MRFs \citep{rue2005gaussian} have been used for texture modeling \citep{chellappa1985classification}, Bayesian restoration \citep{shouno2015bayesian,yasuda2018linear,monma2022hierarchical}, and learned conditional denoising \citep{tappen2007learning}.
The Gaussian optimal denoiser is the Wiener filter \citep{wiener1949extrapolation,andrews1977digital}; its tractability also supports exact analyses of diffusion sampling \citep{pierret2025diffusion} and approximations to learned scores at high noise \citep{wang2024unreasonable}.
We use this tractable prior with patches as variables on an arbitrary graph to establish a locality guarantee, and examine the corresponding qualitative patterns on real images separately.

\section{Proofs}
\label{app:proofs}

\subsection{Notation}
\label{app:notation}

For a positive integer $M$, let $[M]:=\{1,\ldots,M\}$. Vectors in $\mbb{R}^{Md_p}$ are partitioned into $M$ blocks of dimension $d_p$; $x_m$ denotes block $m$ of a vector $x$, and $A_{ml}$ denotes the $(m,l)$ block of a conformably partitioned matrix $A$. 
The identity matrix is denoted by $I$, with its dimension determined by context, and $A^\top$ denotes the transpose of $A$.

For a vector $x$, $\|x\|_2$ is its Euclidean norm. 
For a matrix $A$, $\|A\|_2:=\sup_{\|x\|_2=1}\|Ax\|_2$ is the induced operator norm and $\|A\|_F$ is the Frobenius norm. 
We write $\spec(A)$ for the spectrum of $A$, $\lambda_{\mrm{min}}(A)$ and $\lambda_{\mrm{max}}(A)$ for the smallest and largest eigenvalues of a real symmetric matrix, and $\opn{tr}(A)$ for its trace. 
For real symmetric matrices, $A\preceq B$ means that $B-A$ is positive semidefinite.

For the graph $G_0=([M],E_0)$, $d_{G_0}(m,l)$ is the length of a shortest path between vertices $m$ and $l$, and $B_r(m):=\{l\in[M]:d_{G_0}(m,l)\leq r\}$ is the closed graph ball of radius $r$ centered at $m$. 
For a block vector $x$, $x_{B_r(m)}$ collects the blocks indexed by $B_r(m)$. Finally, $\mbb{E}$ denotes expectation and $\N(\mu,\Sigma)$ denotes a Gaussian distribution with mean $\mu$ and
covariance $\Sigma$.

For the conditional-mean decompositions, work at fixed time $t$ on a probability space $(\Omega,\mcl{F},\mbb{P})$ carrying $X_0$ and the independent Gaussian noise in the corruption model, and assume $\mbb{E}\|X_0\|_2^2<\infty$.
All Euclidean spaces are equipped with their Borel $\sigma$-fields, and measurability of maps between them is understood in this sense.
For a random vector $Z$, write $\opn{sf}(Z)\subseteq\mcl{F}$ for the $\sigma$-field it generates, and let $P_{Y_t}$ denote the law of $Y_t$.
Inclusions between $\sigma$-fields are understood modulo $\mbb{P}$-null sets.

\subsection{Background facts}
\label{app:background}

\subsubsection{Chebyshev polynomials and the reciprocal expansion}
\label{app:chebyshev}
The Chebyshev polynomials are a family of orthogonal polynomials on $[-1,1]$ that are useful for constructing near-minimax polynomial approximations. 
We use the Chebyshev polynomials of the first kind, defined by the recurrence $T_0(s) = 1$, $T_1(s) = s$, and $T_{r+1}(s) = 2sT_r(s) - T_{r-1}(s)$ for $r\ge 1$.
Equivalently, $T_r(\cos \theta) = \cos(r\theta)$ for $\theta \in [0,\pi]$. 
$T_r$ has degree $r$ and $|T_r(s)|\le 1$ for $s\in [-1,1]$.
For general background, see \citet{wikipedia:chebyshev_polynomials}.
The ordinary generating function for $T_r$ is given by \citep[Eq.~18.12.8]{nistdlmf2026}
\begin{align}\label{app-eqn:chebyshev-generating}
    \sum_{r=0}^{\infty} T_r(s) z^r = \frac{1-sz}{1-2sz+z^2},\qquad |z|<1,\quad s\in[-1,1].
\end{align}

We now introduce a useful lemma on the Chebyshev approximation of the reciprocal function with exponentially decaying error in the degree $r$.
\begin{lemma}
    Let $0<h_-<h_+$ and $r\ge 0$.
    There exists a polynomial $\psi_r(x)$ of degree at most $r$ such that
    \begin{align}\label{app-eqn:chebyshev-reciprocal}
        \sup_{x\in[h_{-},h_+]} \left|\frac{1}{x} - \psi_r(x)\right| = \frac{2}{\sqrt{h_- h_+}(1-\rho)}\rho^{r+1},
    \end{align}
    where $\rho = (\sqrt{h_+/h_-}-1)/(\sqrt{h_+/h_-}+1)$ is the Chebyshev convergence factor.
\end{lemma}
\begin{proof}
    Let $s(x)=(2x-h_+-h_-)/(h_+-h_-)$ be the linear transformation that maps $[h_-,h_+]$ to $[-1,1]$.
    Since $T_0(s)=1$ and the series converges absolutely and uniformly over $s\in [-1,1]$, from \eqref{app-eqn:chebyshev-generating}, we have
    \begin{align*}
        1+2\sum_{r=1}^\infty z^r T_r(s) = \frac{1-z^2}{1-2sz+z^2},\qquad |z|<1,\quad s\in[-1,1].
    \end{align*}
    Substituting $z=-\rho$ and $s=s(x)$ and simplifying gives
    \begin{align*}
        \frac{1}{x} = \frac{1}{\sqrt{h_-h_+}}\left(1 + 2\sum_{r=1}^\infty (-\rho)^r T_r(s(x))\right),\qquad x\in[h_-,h_+].
    \end{align*}
    Now, define the degree-$r$ truncated Chebyshev polynomial as
    \begin{align}\label{app-eqn:chebyshev-reciprocal-truncated}
        \psi_r(x) := \frac{1}{\sqrt{h_-h_+}}\left(1 + 2\sum_{k=1}^r (-\rho)^k T_k(s(x))\right).
    \end{align}
    Then, since $|T_k(s(x))|\leq 1$, we have, for $x\in[h_-,h_+]$,
    \begin{align*}
        \left|\frac{1}{x} - \psi_r(x)\right| = 
        \frac{2}{\sqrt{h_- h_+}}\left|\sum_{k=r+1}^\infty (-\rho)^k T_k(s(x))\right| \leq
        \frac{2}{\sqrt{h_- h_+}}\sum_{k=r+1}^\infty \rho^k =
        \frac{2}{\sqrt{h_- h_+}(1-\rho)}\rho^{r+1}.
    \end{align*}
    At $x=h_-$, we have $s(h_-)= -1$ and $T_k(-1)=(-1)^k$, so $(-\rho)^k T_k(s(h_-)) = \rho^k$ for all $k\ge 1$, and hence the upper bound is tight.
\end{proof}

\section{Main proofs}

\subsection{Full statement and proof of Theorem~\ref{thm:gmrf-locality}}\label{app:gmrf-locality}

\begin{gmrfstatement}[Exponential locality of GMRF; full statement]
  Suppose $X_0\sim\N(0,Q^{-1})$, where $Q\in\mbb{R}^{Md_p\times Md_p}$ is symmetric positive definite and its off-diagonal patch blocks vanish outside the edges of $G_0=([M],E_0)$.
  Let $X_t=a_tX_0+b_t\varepsilon$ with $a_t,b_t>0$ and independent $\varepsilon\sim\N(0,I_{Md_p})$, and set $Y_t=X_t/a_t$ and $\sigma_t=b_t/a_t$.
  For $B_r(m)=\{l:d_{G_0}(m,l)\leq r\}$, define
  \[
    D_{m,t}^\ast(y)=\mbb{E}[X_{0,m}\mid Y_t=y,t],\qquad
    D_{m,t}^{\ast,r}(y)=\mbb{E}[X_{0,m}\mid Y_{t,B_r(m)}=y_{B_r(m)},t].
  \]
  Let $q_-:=\lambda_{\min}(Q)$ and $q_+:=\lambda_{\max}(Q)$ satisfy $0<q_- < q_+<\infty$.
  Then for every patch $m$ and every integer radius $r\geq 0$,
  \begin{align}\label{eqn:chebyshev-locality}
        \mathbb{E}\left\|D_{m,t}^\ast(Y_t) - D_{m,t}^{\ast,r}(Y_t)\right\|_2^2 \leq 
        \frac{d_p}{q_-}
        \left[
        \frac{\left(1-1/\sqrt{\chi_t}\right)^2 }{1+\sigma_t^2 q_-}
        \right]
        \rho_t^{2r},
    \end{align}
    where $\chi_t = (1+\sigma_t^2 q_+)/(1+\sigma_t^2 q_-)$ is the condition number of $I+\sigma_t^2Q$ and $\rho_t = (\sqrt{\chi_t}-1)/(\sqrt{\chi_t}+1)$ is the Chebyshev convergence factor.
    In the notation of Assump.~\ref{assm:exponential-locality}, the bound is $C_t\exp(-c_t r)$ with
    \[
        C_t=\frac{d_p(1-1/\sqrt{\chi_t})^2}{q_-(1+\sigma_t^2q_-)},
        \qquad c_t=2\log(1/\rho_t)>0.
    \]
    If $q_-=q_+$, then $Q=q_-I$ and radius zero is exact.
\end{gmrfstatement}

\begin{proof}
Since $Y_t=X_0+\sigma_t\varepsilon$, Bayes' rule gives posterior precision
$Q+\sigma_t^{-2}I=\sigma_t^{-2}(I+\sigma_t^2Q)$ and posterior mean
\begin{align*}
    D_{m,t}^\ast(y)=\left[(I+\sigma_t^2Q)^{-1}y\right]_m.
\end{align*}
Let $H_t:=I+\sigma_t^2Q$. The matrix $H_t$ is real symmetric positive definite, and
\begin{align*}
    \spec(H_t)\subseteq[1+\sigma_t^2q_-,1+\sigma_t^2q_+].
\end{align*}
Write its spectral decomposition as $H_t=U\Lambda_tU^\top$, where $U$ is orthogonal and $\Lambda_t$ is diagonal.
We first approximate the scalar inverse function with the degree-$r$ polynomial in \eqref{app-eqn:chebyshev-reciprocal-truncated}, and then apply that polynomial to $H_t$.

Take $h_{-,t}:=1+\sigma_t^2 q_-$ and $h_{+,t}:=1+\sigma_t^2 q_+$, so that $\chi_t = h_{+,t}/h_{-,t}$. 
Then by the Chebyshev approximation in \eqref{app-eqn:chebyshev-reciprocal}, the degree-$r$ Chebyshev polynomial $\psi_{r,t}(x)$ as given in \eqref{app-eqn:chebyshev-reciprocal-truncated} (with $h_-$ and $h_+$ replaced by $h_{-,t}$ and $h_{+,t}$) satisfies
\begin{align*}
\|H_t^{-1} - \psi_{r,t}(H_t)\|_2 &=
\|U \left(\Lambda_t^{-1} - \psi_{r,t}(\Lambda_t)\right) U^\top\|_2
\\&=
\max_{\lambda \in \spec(H_t)} \left|\frac{1}{\lambda} -\psi_{r,t}(\lambda)\right|\\
&\leq \frac{2\rho_t^{r+1}}{\sqrt{h_{-,t} h_{+,t}}(1-\rho_t)}.
\end{align*}
On the other hand, $H_t$ has the same off-diagonal block sparsity pattern as $Q$, so, for every integer $k\geq0$, $[H_t^k]_{ml}=0$ whenever $d_{G_0}(m,l)>k$.
Indeed, each nonzero term in the blockwise expansion of $[H_t^k]_{ml}$ corresponds to a length-$k$ walk from $m$ to $l$, with diagonal blocks allowing the walk to remain at its current vertex.
Because $\psi_{r,t}$ has degree at most $r$, the predictor
\begin{align*}
    \widetilde{D}_{m,t}^{r}(y):=[\psi_{r,t}(H_t)y]_m
\end{align*}
depends only on $y_{B_r(m)}$ and is therefore measurable with respect to $Y_{t,B_r(m)}$.
Graph connectivity determines which patches lie within each radius, while the conditioning of $I+\sigma_t^2Q$ controls the decay rate: worse conditioning slows the certified decay and can require broader context.

By the tower property,
$D_{m,t}^{\ast,r}(Y_t)=\mbb{E}[D_{m,t}^\ast(Y_t)\mid Y_{t,B_r(m)},t]$.
The projection property of conditional expectation therefore implies
\begin{align*}
    \mbb{E}\|D_{m,t}^\ast(Y_t)-D_{m,t}^{\ast,r}(Y_t)\|_2^2
    \leq
    \mbb{E}\|D_{m,t}^\ast(Y_t)-\widetilde{D}_{m,t}^{r}(Y_t)\|_2^2.
\end{align*}
Construct $P_m \in \mbb{R}^{d_p\times Md_p}$ so that it selects patch $m$ of a noisy image $y\in \mbb{R}^{Md_p}$, i.e., $P_m y = y_m$.
Then $D_{m,t}^\ast(Y_t)-\widetilde{D}_{m,t}^{r}(Y_t)=P_m\left(H_t^{-1}-\psi_{r,t}(H_t)\right)Y_t$.
Since $Y_t$ is zero-mean Gaussian with covariance $Q^{-1}+\sigma_t^2 I$, we have
\begin{align*}
    \mbb{E}\|D_{m,t}^\ast(Y_t)-\widetilde{D}_{m,t}^{r}(Y_t)\|_2^2 
    &= \opn{tr} \bigg(
        P_m
        \left(H_t^{-1}-\psi_{r,t}(H_t)\right)
        (Q^{-1}+\sigma_t^2 I)
        \left(H_t^{-1}-\psi_{r,t}(H_t)\right)
        P_m^\top
        \bigg)\\
    &= \left\| P_m \left(H_t^{-1}-\psi_{r,t}(H_t)\right) (Q^{-1}+\sigma_t^2 I)^{1/2} \right\|_F^2\\
    &\leq \|Q^{-1}+\sigma_t^2 I\|_2 
    \left\| P_m \left(H_t^{-1}-\psi_{r,t}(H_t)\right)\right\|_F^2.
\end{align*}
Since $P_m(H_t^{-1}-\psi_{r,t}(H_t)) \in \mbb{R}^{d_p\times Md_p}$ has $d_p$ rows, the Frobenius--operator norm inequality and the Chebyshev approximation bound give
\begin{align*}
    \left\| P_m \left(H_t^{-1}-\psi_{r,t}(H_t)\right)\right\|_F^2
    &\leq d_p \left\| P_m \left(H_t^{-1}-\psi_{r,t}(H_t)\right)\right\|_2^2\\
    &\leq d_p \left\| H_t^{-1}-\psi_{r,t}(H_t)\right\|_2^2\\
    &\leq d_p \frac{4\rho_t^{2r+2}}{h_{-,t} h_{+,t}(1-\rho_t)^2}.
\end{align*}
Also, since $Q^{-1}+\sigma_t^2 I$ is real symmetric positive definite,
\begin{align*}
    \|Q^{-1}+\sigma_t^2 I\|_2
    =\lambda_{\mrm{max}}(Q^{-1}+\sigma_t^2 I)
    =q_-^{-1}+\sigma_t^2.
\end{align*}
Combining the above bounds, we have
\begin{align*}
    \mbb{E}\|D_{m,t}^\ast(Y_t)-\widetilde{D}_{m,t}^{r}(Y_t)\|_2^2
    &\leq d_p(q_-^{-1}+\sigma_t^2)
    \frac{4\rho_t^{2r+2}}{h_{-,t}h_{+,t}(1-\rho_t)^2}.
\end{align*}
Finally, the identities
\begin{align*}
q_-^{-1}+\sigma_t^2=\frac{h_{-,t}}{q_-},
\qquad
\frac{2\rho_t}{\sqrt{h_{-,t}h_{+,t}}(1-\rho_t)}
=\frac{1-1/\sqrt{\chi_t}}{h_{-,t}}
\end{align*}
reduce the preceding display to \eqref{eqn:chebyshev-locality}.
\end{proof}

\clearpage
\subsection{Full statement and proof of Corollary~\ref{crl:bounded-range}}\label{app:bounded-range}

\begin{radiusstatement}[Bounded-range monotonicity of the locality radius; full statement]
Under the hypotheses of the full Thm.~\ref{thm:gmrf-locality} above, with $0<q_-<q_+$,
define $\sigma_{\mrm{abs}}^2:=(q_-^{1/3}+q_+^{1/3})/(q_-^{2/3}q_+^{2/3})$.
For $\epsilon>0$ and $0<\sigma_t \leq \sigma_{\mrm{abs}}$, the smallest radius $r_\epsilon(\sigma_t)$ such that the upper bound in \eqref{eqn:chebyshev-locality} is at most $\epsilon$ is the least nonnegative integer satisfying
\begin{align}\label{eqn:chebyshev-locality-radius}
    r_\epsilon(\sigma_t) \ge \frac{1}{2}\log 
    \left(\frac{d_p \left(1-1/\sqrt{\chi_t}\right)^2}{\epsilon q_-(1+\sigma_t^2 q_-)}\right)
    \bigg /
    \log \left(\frac{1}{\rho_t}\right)
\end{align}
The resulting $r_\epsilon(\sigma_t)$ is nondecreasing in $\sigma_t$ on the same range. At $\sigma_t=0$, radius zero is exact.
\end{radiusstatement}

\begin{proof}
    We show that the upper bound in \eqref{eqn:chebyshev-locality}
    is nondecreasing in $\sigma_t$ for $0<\sigma_t \leq \sigma_{\mrm{abs}}$.
    First, for $\sigma_t>0$,
    \begin{align*}
        \frac{1}{\chi_t}\frac{d\chi_t}{d\sigma_t} = \frac{2\sigma_t(q_+ - q_-)}{(1+\sigma_t^2 q_+)(1+\sigma_t^2 q_-)} \ge 0,
    \end{align*}
    and hence $d\chi_t/d\sigma_t\ge 0$.
    Next,
    \begin{align*}
        \frac{d}{d\sigma_t}\log \left[
        \frac{\left(1-1/\sqrt{\chi_t}\right)^2 }{(1+\sigma_t^2 q_-)}
        \right] &=
        \frac{1}{\sqrt{\chi_t}-1}\left(\frac{1}{\chi_t}\frac{d\chi_t}{d\sigma_t}\right) - \frac{2\sigma_t q_-}{1+\sigma_t^2 q_-} \\&=
        \frac{2\sigma_t}{1+\sigma_t^2 q_-}\left[
            \frac{q_+-q_-}{(1+\sigma_t^2 q_+)(\sqrt{\chi_t}-1)} - q_-
        \right]
    \end{align*}
    is nonnegative when $q_+-q_- \geq q_-(1+\sigma_t^2 q_+)(\sqrt{\chi_t}-1)$. Using
    \begin{align*}
        1+\sigma_t^2q_+
        =\chi_t\frac{q_+-q_-}{q_+-q_-\chi_t},
        \qquad
        q_+-q_-\chi_t=\frac{q_+-q_-}{1+\sigma_t^2q_-},
    \end{align*}
    this condition is equivalent to $q_+\geq q_-\chi_t^{3/2}$, which simplifes to $\sigma_t^2 \leq \sigma_{\mrm{abs}}^2$.
    
    For the rest of the term in the upper bound, it is straightforward to see that $d\rho_t/d\chi_t\ge 0$, and since $d\chi_t/d\sigma_t\ge 0$, we have $d\rho_t/d\sigma_t\ge 0$.
    Consequently, if $0<\sigma_1\leq \sigma_2 \leq \sigma_{\mrm{abs}}$, every radius whose upper bound is at most $\epsilon$ at $\sigma_2$ also has upper bound at most $\epsilon$ at $\sigma_1$. Hence $r_\epsilon(\sigma_2) \geq r_\epsilon(\sigma_1)$.
\end{proof}

\subsection{Proof of Theorem~\ref{thm:pld-is-local}}\label{app-sec:pld-is-local-proof}
\begin{proof}
Fix $m$ and $t$ with $\sigma_t>0$, and use the stacked-vector convention throughout.
Under the empirical prior, let $J$ be uniform on $[n]$ and set $X_0=x^{(J)}$.
Since $Y_t=X_0+\sigma_t\varepsilon$, the linear aggregate satisfies
$\mcl{S}_{m,t}^{\mrm{lin}}(Y_t)\mid J=i\sim
\mcl{N}\!\left(S_{m,t}x^{(i)},\sigma_t^2 S_{m,t}S_{m,t}^\top\right)$.
The covariance is positive definite because $\sigma_t>0$ and $S_{m,t}$ has full row rank.
With the prescribed precision $\Gamma_{m,t}=(\sigma_t^2 S_{m,t}S_{m,t}^\top)^{-1}$, the Gaussian log-likelihood is $-\tfrac12\|s-S_{m,t}x^{(i)}\|_{\Gamma_{m,t}}^2$ up to a term independent of $i$.
The uniform prior and common Gaussian normalizing factor cancel in Bayes' rule, giving
\begin{align*}
    \opn{Pr}_{\hat{p}_0} \left[ J=i \mid \mcl{S}_{m,t}^{\mrm{lin}}(Y_t)=s, t\right] =
    \opn{sm}_i\left\{ -\frac{1}{2}
    \left\| s - \mcl{S}_{m,t}^{\mrm{lin}}(x^{(j)}) \right\|^2_{\Gamma_{m,t}}
    \right\}_{j\in[n]}.
\end{align*}
Consequently,
\begin{align*}
    \mbb{E}_{\hat{p}_0}\left[
    X_{0,m} \mid \mcl{S}_{m,t}^{\mrm{lin}}(Y_t)=s, t
    \right] = \sum_{i=1}^n \opn{Pr}_{\hat{p}_0} \left[ J=i \mid \mcl{S}_{m,t}^{\mrm{lin}}(Y_t)=s, t\right] x_{m}^{(i)} = g_{m,t}^{\mrm{rbf}}(s).
\end{align*}
Because the aggregate is a measurable function of $Y_{t,\mcl{B}_{m,t}}$, the tower property gives the first equality below.
The assumed sigma-field inclusion makes the inner conditional mean measurable with respect to the aggregate, giving the second equality:
\begin{align*}
    \mbb{E}_{\hat{p}_0}\left[X_{0,m} \mid \mcl{S}_{m,t}^{\mrm{lin}}(Y_t), t\right] &=
    \mbb{E}_{\hat{p}_0}\left[
    \mbb{E}_{\hat{p}_0}\left[X_{0,m} \mid Y_{t, \mcl{B}_{m,t}}, t\right]
    \mid \mcl{S}_{m,t}^{\mrm{lin}}(Y_t), t\right]\\&=
    \mbb{E}_{\hat{p}_0}\left[X_{0,m} \mid Y_{t, \mcl{B}_{m,t}}, t\right].
\end{align*}
Together with $\hat{D}_{m,t}^{\mrm{pld}}=g_{m,t}^{\mrm{rbf}}\circ\mcl{S}_{m,t}^{\mrm{lin}}$, this proves the theorem.

To make the inclusion condition concrete, restrict $S_{m,t}$ to its context columns, denoted by $[S_{m,t}]_{\mcl{B}_{m,t}}$.
Locality implies $S_{m,t}y=[S_{m,t}]_{\mcl{B}_{m,t}}y_{\mcl{B}_{m,t}}$.
Bayes' rule for the full context gives
\begin{align*}
    \opn{Pr}_{\hat{p}_0} \left[ J=i \mid Y_{t,\mcl{B}_{m,t}}=y_{\mcl{B}_{m,t}}, t\right] =
    \opn{sm}_i\left\{ -\frac{1}{2\sigma_t^2}
    \left\| y_{\mcl{B}_{m,t}}-x^{(j)}_{\mcl{B}_{m,t}} \right\|^2_2
    \right\}_{j\in[n]}\\ =
    \opn{sm}_i\left\{ \frac{2y_{\mcl{B}_{m,t}}^\top\left(x^{(j)}_{\mcl{B}_{m,t}}-x^{(1)}_{\mcl{B}_{m,t}}\right)-\|x^{(j)}_{\mcl{B}_{m,t}}\|_2^2+\|x^{(1)}_{\mcl{B}_{m,t}}\|_2^2}{2\sigma_t^2}
    \right\}_{j\in[n]}.
\end{align*}
The second equality subtracts the log-likelihood of sample $1$ from every logit, which leaves the softmax unchanged.
Thus, the posterior weights depend on $y_{\mcl{B}_{m,t}}$ only through its inner products with the training-context differences.
Suppose that
\begin{align}\label{app-eqn:V-cond}
    V_{m,t}:=\opn{span}\left\{
        x^{(j)}_{\mcl{B}_{m,t}}-x^{(1)}_{\mcl{B}_{m,t}}:j\in[n]
    \right\}\subseteq\opn{row}\!\left([S_{m,t}]_{\mcl{B}_{m,t}}\right).
\end{align}
For each $j$, there is then a vector $a_j\in\mbb{R}^{d_h}$ such that $x^{(j)}_{\mcl{B}_{m,t}}-x^{(1)}_{\mcl{B}_{m,t}}=[S_{m,t}]_{\mcl{B}_{m,t}}^\top a_j$.
Hence $y_{\mcl{B}_{m,t}}^\top(x^{(j)}_{\mcl{B}_{m,t}}-x^{(1)}_{\mcl{B}_{m,t}})=a_j^\top [S_{m,t}]_{\mcl{B}_{m,t}}y_{\mcl{B}_{m,t}}=a_j^\top S_{m,t}y$.
Every posterior weight is therefore a continuous function of $S_{m,t}y$, as is their weighted sum $\mbb{E}_{\hat p_0}[X_{0,m}\mid Y_{t,\mcl{B}_{m,t}}=y_{\mcl{B}_{m,t}},t]$.
This proves the required sigma-field inclusion: the local conditional mean can be recovered from the aggregate even when the full context cannot.

This construction is strictly compressive when $0<v_{m,t}:=\dim V_{m,t}<|\mcl{B}_{m,t}|d_p$: choose the $v_{m,t}$ rows of $[S_{m,t}]_{\mcl{B}_{m,t}}$ to be an orthonormal basis of $V_{m,t}$, so that $d_h=v_{m,t}<|\mcl{B}_{m,t}|d_p$.
All clean training contexts have the same projection onto $V_{m,t}^\perp$.
Under isotropic Gaussian corruption, the discarded coordinates consequently contribute the same likelihood factor for every sample index, which cancels from the posterior weights.
For example, if $|\mcl{B}_{m,t}|d_p=2$ and $x_{\mcl{B}_{m,t}}^{(j)}=(u_j,c)$ with fixed $c$ and varying $u_j$, then $[S_{m,t}]_{\mcl{B}_{m,t}}=(1,0)$ retains only the first noisy coordinate and still determines all posterior weights.
If $v_{m,t}=0$, all training contexts coincide and the local conditional mean is constant.
In general $v_{m,t}\leq\min(n-1,|\mcl{B}_{m,t}|d_p)$, but $v_{m,t}$ can equal $|\mcl{B}_{m,t}|d_p$; in that case this row-space condition requires retaining the full context dimension and provides no strict compression.
We note that the condition \eqref{app-eqn:V-cond} is sufficient, not necessary for the theorem: it preserves the entire posterior over sample indices, whereas the theorem requires only preservation of the target-patch conditional mean.
\end{proof}

\section{From the aggregator--predictor construction to implementation}
\label{app:attention-aggregator}

This section connects the aggregator--predictor construction in Secs.~\ref{sec:dit-decomposition} to the implemented architectures.
We first expand a single attention branch applied to the noisy patches $Y_t$, then explain how aggregation and patch prediction extend to residual stacks, and finally give the full predictor parameterizations.
All patches, tokens, embeddings, biases, and aggregates are column vectors, consistent with Appendix~\ref{app:notation}; projection matrices act from the left.
The width $d_h$ denotes the full token and aggregate dimension, and the number of attention heads $\tilde{H}$ divides $d_h$.

\begin{figure}[H]
    \centering
    \includegraphics[width=\textwidth]{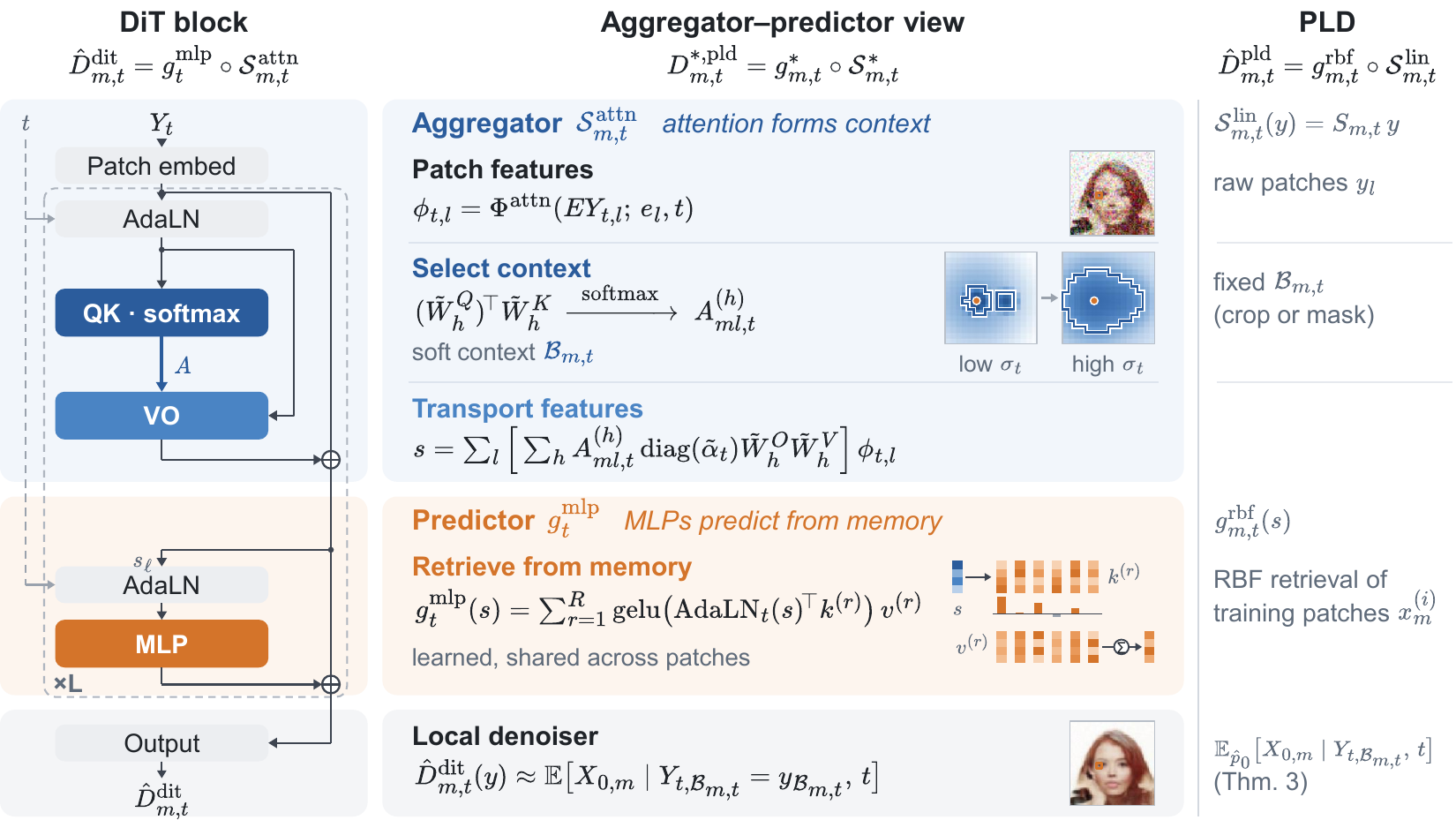}
    \caption{\textbf{Component map for the aggregator--predictor view.}
    At the branch level, AdaLN-conditioned QK supplies spatial weights, VO transforms the weighted features, and the MLP admits a learned-memory interpretation for tokenwise prediction.
    The patchwise linear denoiser (PLD) instead uses a prescribed context and retrieves training patches.}
    \label{fig:dit-component-map}
\end{figure}

\subsection{Patch, position, and time conditioning}
\label{app:attention-aggregator-1}

The initial token at patch $l$ is $EY_{t,l}+e_l$, where $E\in\mbb{R}^{d_h\times d_p}$ and $e_l\in\mbb{R}^{d_h}$.
A shared patch-projection bias, if present, can be included in each $e_l$.
If the network receives $X_t=a_tY_t$ or another known scalar preconditioning of $Y_t$, include that time-dependent scalar before the patch projection; the same derivation applies.

Let $\tau(t)\in\mbb{R}^{d_h}$ denote a fixed sinusoidal time embedding.
The time-embedding and the following AdaLN parameters can be written as
\begin{align}\label{app-eqn:attention-time}
 e_t^{\mrm{time}}
 &=W_2^{\mrm{time}}\opn{SiLU}\!\left(W_1^{\mrm{time}}\tau(t)+b_1^{\mrm{time}}\right)+b_2^{\mrm{time}},\\
 \begin{bmatrix}
 \tilde{\beta}_t^\top & \tilde{\gamma}_t^\top & \tilde{\alpha}_t^\top &
 \beta_t^\top & \gamma_t^\top & \alpha_t^\top
 \end{bmatrix}^{\!\top}
 &=W^{\mrm{ada}}e_t^{\mrm{time}}+b^{\mrm{ada}}.
 \label{app-eqn:attention-adaln}
\end{align}
Here $W_1^{\mrm{time}}\in\mbb{R}^{4d_h\times d_h}$, $W_2^{\mrm{time}}\in\mbb{R}^{d_h\times4d_h}$, and their biases have dimensions $4d_h$ and $d_h$, respectively.
Also $W^{\mrm{ada}}\in\mbb{R}^{6d_h\times d_h}$ and $b^{\mrm{ada}}\in\mbb{R}^{6d_h}$, so every shift $\beta$, scale adjustment $\gamma$, and residual gate $\alpha$ is a column in $\mbb{R}^{d_h}$.
The tilded rows condition spatial attention; the untilded rows condition the predictor branch.
Layer indices are suppressed in this single-layer expansion; in a stack, the AdaLN map and both sets of conditioning parameters are layer-specific.

With these notations, the feature entering attention can be written as
\begin{align}\label{app-eqn:attention-feature}
 \phi_{t,l}(Y_t)
 :=\Phi^{\mrm{attn}}(EY_{t,l};e_l,t)
 =\bigl(1+\tilde{\gamma}_t\bigr)\odot
 \opn{LN}(EY_{t,l}+e_l)+\tilde{\beta}_t.
\end{align}
Layer normalization acts over the $d_h$ channels of this token, with a positive numerical stabilizer and no learned affine parameters; $\odot$ denotes elementwise multiplication.
Consequently, at fixed $t$, this feature depends only on patch $l$ and its position.
The gate $\tilde{\alpha}_t$ is applied after attention output mixing.

\subsection{Spatial weights and shared channel operators}
\label{app:attention-aggregator-2}

For head $h$, the Q, K, and V matrices have shape $(d_h/\tilde{H})\times d_h$, and the output projection matrix $\tilde{W}_h^O$ has shape $d_h\times(d_h/\tilde{H})$.
Including the QK biases $\tilde{b}_h^Q,\tilde{b}_h^K\in\mbb{R}^{d_h/\tilde{H}}$, we define the scalar spatial weights
\begin{align}\label{app-eqn:attention-weights}
 A_{ml,t}^{(h)}(Y_t)
 =\opn{sm}_{l}\left\{
 \frac{\langle \tilde{W}_h^Q\phi_{t,m}(Y_t)+\tilde{b}_h^Q,
                   \tilde{W}_h^K\phi_{t,j}(Y_t)+\tilde{b}_h^K\rangle}{\sqrt{d_h/\tilde{H}}}
 \right\}_{j\in[M]}.
\end{align}
These weights are positive and sum to one over $l$.
Without QK biases, the numerator is $\phi_{t,m}(Y_t)^\top\allowbreak(\tilde{W}_h^Q)^\top \allowbreak\tilde{W}_h^K\phi_{t,l}(Y_t)$.
The matrix $(\tilde{W}_h^Q)^\top \tilde{W}_h^K\in\mbb{R}^{d_h\times d_h}$ therefore controls compatibility between query and context features.

Rearranging the finite head and patch sums gives
\begin{equation}\label{app-eqn:attention-rearrangement}
    \begin{split}
    \mcl{S}_{m,t}^{\mrm{attn}}(Y_t)
    &=\opn{diag}(\tilde{\alpha}_t)
    \sum_{h=1}^{\tilde{H}} \tilde{W}_h^O\left(\sum_{l=1}^M
    A_{ml,t}^{(h)}(Y_t)\tilde{W}_h^V\phi_{t,l}(Y_t)\right)\\
    &=\sum_{l=1}^M
    \left[\sum_{h=1}^{\tilde{H}} A_{ml,t}^{(h)}(Y_t)
    \opn{diag}(\tilde{\alpha}_t)\tilde{W}_h^O\tilde{W}_h^V\right]\phi_{t,l}(Y_t).
    \end{split}
\end{equation}
Each bracketed matrix in the second line is $d_h\times d_h$, and each summand is a column in $\mbb{R}^{d_h}$.
Setting the QK biases to zero recovers \eqref{eqn:attention}; with biases, the same decomposition holds using the weights in \eqref{app-eqn:attention-weights}.
If included, the value biases $\tilde{b}_h^V\in\mbb{R}^{d_h/\tilde{H}}$ and output bias $\tilde{b}^O\in\mbb{R}^{d_h}$ add the column offset $\opn{diag}(\tilde{\alpha}_t)\bigl(\sum_h \tilde{W}_h^O\tilde{b}_h^V+\tilde{b}^O\bigr)$, because each attention row sums to one.
This offset depends on time but not on the noisy image or patch location, so it does not change the spatial selection mechanism.

As explained in Sec.~\ref{sec:attention-aggregator}, we can read off the roles of QK and VO from \eqref{app-eqn:attention-rearrangement}.
At fixed $t$, every $S_{ml}^{\mrm{attn}}(Y_t;t)$ lies in the linear subspace spanned by the time-conditioned VO operators, $\opn{span}\{\opn{diag}(\tilde{\alpha}_t)\tilde{W}_h^O\tilde{W}_h^V:h\in[\tilde{H}]\}$, whose dimension is at most $\min(\tilde{H},d_h^2)$.
Thus, each head supplies a patch-shared, time-conditioned channel operator $\opn{diag}(\tilde{\alpha}_t)\tilde{W}_h^O\tilde{W}_h^V$ and an \emph{independently normalized} spatial weighting.
Before output mixing, a head forms a convex combination of its projected patch features $\tilde{W}_h^V\phi_{t,l}(Y_t)$.
However, the full branch \eqref{app-eqn:attention-rearrangement} need not be a convex combination of the patch features, since the channel operators are unconstrained.
With the value vectors and channel operators held fixed, the attention matrix $A_h$ determines how the head distributes across the value vectors associated with different patches.
QK therefore directly controls which patches receive more or less exposure in the constructed context, even though the downstream effect of an individual weight also depends on value magnitudes, channel operators, and residual interactions.

\subsection{Soft locality of Attention Aggregator}
\label{app:attention-aggregator-3}

Section~\ref{app:attention-aggregator-2} explained how, for fixed features, QK controls the relative logits and hence which patches receive appreciable weight.
Accordingly, low-scoring patches are suppressed smoothly by exponentiation and normalization, rather than removed by a literal threshold.
This provides an input-adaptive analogue of a local selector, but without explicit thresholding unmasked attention remains globally supported and need not solely concentrate on a local context set.

To relate this interpretation to our definition of locality, fix a nonempty context $\mcl{B}_{m,t}$ containing $m$ and mask every head to that context before softmax.
Define $\mcl{S}_{m,t}^{\mrm{attn},\mcl{B}}$ using the same features, channel operators, and biases, but with weights proportional to $A_{ml,t}^{(h)}(Y_t)$ over $l\in\mcl{B}_{m,t}$ and zero otherwise.
The renormalization cancels the original global softmax denominator.
Since each retained logit depends only on patches $m$ and $l$, this masked branch is exactly $\mcl{B}_{m,t}$-local.
The following bound records the expected squared error of this hard-threshold masking approximation.
The bound depends on discarded mass over the observation law, together with the feature and channel-operator magnitudes.
The notation $\lesssim_t$ hides a constant depending only on the fixed network parameters and time $t$, not on the observation or the context.

\begin{theorem}[Expected squared error of context masking]\label{app-thm:attention-local-error}
Fix $t$, a patch $m$, the attention parameters, and a nonempty context $\mcl{B}_{m,t}\subseteq[M]$ containing $m$.
Then $\mcl{S}_{m,t}^{\mrm{attn},\mcl{B}}$ is $\mcl{B}_{m,t}$-local, and
\begin{equation}\label{app-eqn:attention-local-error}
 \mbb{E}\left\|\mcl{S}_{m,t}^{\mrm{attn}}(Y_t)
 -\mcl{S}_{m,t}^{\mrm{attn},\mcl{B}}(Y_t)\right\|_2^2
 \lesssim_t\sum_{h=1}^{\tilde{H}}\mbb{E}\left[\left(
 \sum_{l\notin\mcl{B}_{m,t}}A_{ml,t}^{(h)}(Y_t)\right)^2\right].
\end{equation}
All expectations are over the observation law of $Y_t$ at the fixed time $t$.
\end{theorem}

\begin{proof}
Before applying the time-conditioned scale and shift, the layer-normalized feature has norm at most $\sqrt{d_h}$, so \eqref{app-eqn:attention-feature} gives
\[
 \max_{l\in[M]}\|\phi_{t,l}(Y_t)\|_2
 \leq\sqrt{d_h}\|1+\tilde{\gamma}_t\|_\infty+\|\tilde{\beta}_t\|_2.
\]
Fix an observation $Y_t$.
For each head, let $\mu_h^{\mrm{in}}$ and $\mu_h^{\mrm{out}}$ be the attention-weighted averages of $\phi_{t,l}(Y_t)$ over the retained and discarded patches, respectively, each normalized by its total attention mass.
The retained mass is positive because the context is nonempty and softmax weights are positive.
If the discarded set is empty, take $\mu_h^{\mrm{out}}=0$, in which case the approximation is exact.
Note that both averages satisfy the same norm bound as the features.
Also, the masked head average is $\mu_h^{\mrm{in}}$, since the retained weights are renormalized to sum to one.
Thus the difference between the full and masked head averages is
\[
\begin{aligned}
 &\sum_{l=1}^M A_{ml,t}^{(h)}(Y_t)\phi_{t,l}(Y_t)-\mu_h^{\mrm{in}}\\
 &\quad=\left(1-\sum_{l\notin\mcl{B}_{m,t}}A_{ml,t}^{(h)}(Y_t)\right)\mu_h^{\mrm{in}}
 +\left(\sum_{l\notin\mcl{B}_{m,t}}A_{ml,t}^{(h)}(Y_t)\right)\mu_h^{\mrm{out}}-\mu_h^{\mrm{in}}\\
 &\quad=\left(\sum_{l\notin\mcl{B}_{m,t}}A_{ml,t}^{(h)}(Y_t)\right)
 (\mu_h^{\mrm{out}}-\mu_h^{\mrm{in}}).
\end{aligned}
\]
The value- and output-bias offsets are the same in both branches, since both the full and masked attention rows sum to one, so these offsets cancel.
The triangle inequality and the spectral-norm bound therefore give
\[
\begin{split}
 &\left\|\mcl{S}_{m,t}^{\mrm{attn}}(Y_t)
 -\mcl{S}_{m,t}^{\mrm{attn},\mcl{B}}(Y_t)\right\|_2\\
 &\quad\leq\sum_{h=1}^{\tilde{H}}
 \left\|\opn{diag}(\tilde{\alpha}_t)\tilde{W}_h^O\tilde{W}_h^V\right\|_2
 \left(\sum_{l\notin\mcl{B}_{m,t}}A_{ml,t}^{(h)}(Y_t)\right)
 \|\mu_h^{\mrm{out}}-\mu_h^{\mrm{in}}\|_2\\
 &\quad\leq 2\left(\sqrt{d_h}\|1+\tilde{\gamma}_t\|_\infty+\|\tilde{\beta}_t\|_2\right)
 \sum_{h=1}^{\tilde{H}}
 \left\|\opn{diag}(\tilde{\alpha}_t)\tilde{W}_h^O\tilde{W}_h^V\right\|_2
 \sum_{l\notin\mcl{B}_{m,t}}A_{ml,t}^{(h)}(Y_t).
\end{split}
\]
Squaring and taking expectations gives
\[
\begin{split}
 &\mbb{E}\left\|\mcl{S}_{m,t}^{\mrm{attn}}(Y_t)
 -\mcl{S}_{m,t}^{\mrm{attn},\mcl{B}}(Y_t)\right\|_2^2\\
 &\quad\leq 4\left(\sqrt{d_h}\|1+\tilde{\gamma}_t\|_\infty+\|\tilde{\beta}_t\|_2\right)^2
 \mbb{E}\left[\left(\sum_{h=1}^{\tilde{H}}
 \left\|\opn{diag}(\tilde{\alpha}_t)\tilde{W}_h^O\tilde{W}_h^V\right\|_2
 \sum_{l\notin\mcl{B}_{m,t}}A_{ml,t}^{(h)}(Y_t)\right)^2\right].
\end{split}
\]
Applying Cauchy--Schwarz to the head sum yields \eqref{app-eqn:attention-local-error}, with implicit constant
\[
 4\left(\sqrt{d_h}\|1+\tilde{\gamma}_t\|_\infty+\|\tilde{\beta}_t\|_2\right)^2
 \sum_{h=1}^{\tilde{H}}\left\|\opn{diag}(\tilde{\alpha}_t)\tilde{W}_h^O\tilde{W}_h^V\right\|_2^2.
\]
Finally, after masking, the softmax denominator contains only retained logits.
Each such logit depends only on patches $m$ and $l\in\mcl{B}_{m,t}$, and each retained feature depends only on its own patch.
Thus the masked branch depends only on $Y_{t,\mcl{B}_{m,t}}$, establishing locality.
\end{proof}

\subsection{General predictor parameterizations}
\label{app:predictor-implementation}
We now write the full formulation of the shared-memory patch predictors at layer $\ell\in [L]$, where $L$ denotes the number of layers, expanding the AdaLN conditioning and restore the affine biases, output gates, separate output projections, and multiple heads, suppressed in Sec.~\ref{sec:patch-predictors}.
We let $s$ denote the input to a patch predictor at layer $\ell$.
Its time-conditioned version is $\opn{AdaLN}_{\ell,t}(s)=(1+\gamma_{\ell,t})\odot\opn{LN}_{\ell}(s)+\beta_{\ell,t}$.
The AdaLN scale and shift, together with the channelwise output gate $\alpha_{\ell,t}$, are supplied by the untilded rows in \eqref{app-eqn:attention-adaln}.

For compact notation, suppressing the layer index and biases, define the retrieval functions
\begin{align*}
    \kappa_t^{\mrm{sm}}(s,k)&:=\exp\!\left(\opn{AdaLN}_t(s)^\top k\right),
    &\kappa_t^{\mrm{gelu}}(s,k)&:=\opn{gelu}\!\left(\opn{AdaLN}_t(s)^\top k\right).
\end{align*}
Softmax retrieval normalizes $\kappa_t^{\mrm{sm}}$ over memory entries, whereas GELU retrieval uses $\kappa_t^{\mrm{gelu}}$ directly.
The full expressions below make the headwise projections, scaling, biases, and output gates explicit.

Throughout the predictors, we fix the same token dimension $d_h$.
The predictors are specified by the following design choices:
\begin{center}
    \begin{tabular}{ll}
    \toprule
    Notation & Design choice \\
    \midrule
    $H$ & Number of predictor heads \\
    $R$ & Number of learnable memory entries \\
    $d_h'$ & Width of learnable memory \\
    \bottomrule
    \end{tabular}
\end{center}
The ordinary MLP illustrates these choices.
It has a single predictor head $H=1$, and $R=4d_h$ layer-specific keys and values $\{k_{\ell}^{(r)},v_{\ell}^{(r)}\in\mbb{R}^{d_h'}\}_{r=1}^{R}$ of width $d_h'=d_h$:
\begin{equation}\label{app-eqn:predictor-mlp}
    g_{\ell,t}^{\mrm{mlp}}(s)
    =\alpha_{\ell,t}\odot\left[
    \sum_{r=1}^{R}\opn{gelu}\!\left(\opn{AdaLN}_{\ell,t}(s)^\top k_{\ell}^{(r)}+b_{\ell}^{(r)}\right)v_{\ell}^{(r)}
    +b_{\ell}^{O}\right].
\end{equation}
These keys and values are shared across patches within a layer but are learned independently at each layer.
If $H>1$, as with multi-head attention, each predictor head projects queries, keys, and values to width $d_h/H$, assuming $d_h$ is divisible by $H$.

Table~\ref{tab:predictor-parameters} compares the memory banks and layerwise parameter counts for each predictor across all $L$ layers.
The spatial-attention backbone and AdaLN parameters are common to all four models and are excluded.

\begin{table}[!htbp]
\centering
\begin{tabular}{cccc}
\toprule
Predictor & Memory bank & Bank parameters & Layerwise parameters \\
\midrule
LPLD $g_{\ell,m,t}^{\mrm{lpld}}$ & $\{z^{(r)}\}_{r=1}^{R}$ & $RMd_h'$ & $3d_h^2+d_hd_h'+4d_h$ \\
SCA $g_{\ell,t}^{\mrm{sca}}$ & $\{k^{(r)},v^{(r)}\}_{r=1}^{R}$ & $2Rd_h'$ & $2d_h(d_h+d_h')+5d_h+H$ \\
SMLP $g_{\ell,t}^{\mrm{smlp}}$ & $\{k^{(r)},v^{(r)}\}_{r=1}^{R}$ & $2Rd_h'$ & $2d_h(d_h+d_h')+Hd_h'+2d_h$ \\
MLP $g_{\ell,t}^{\mrm{dit}}$ & $\{k_\ell^{(r)},v_\ell^{(r)}\}_{\ell,r}$ & $2LRd_h$ & $0$ \\
\bottomrule
\end{tabular}
\caption{Predictor parameter comparison. The layerwise column includes Q/K/V/output projections, their biases, and SCA's headwise gain parameters at each layer. For LPLD, a learned compatibility map adds $d_hd_h'$ parameters once; this cost is zero when $A=I$.}
\label{tab:predictor-parameters}
\end{table}

\subsubsection{Learned PLD (LPLD)}
At each layer, LPLD uses cross-attention to retrieve from a memory bank $\{z^{(r)}\in\mbb{R}^{Md_h'}\}_{r=1}^{R}$ shared across layers.
Each entry $z^{(r)}=[z_1^{(r)},\ldots,z_M^{(r)}]$ contains $M$ memory patches $z_m^{(r)}\in\mbb{R}^{d_h'}$, whose width $d_h'$ may differ from the token width $d_h$.
The query is projected from the conditioned layer input $\opn{AdaLN}_{\ell,t}(s)$, while the value is projected directly from the raw memory patch $z_m^{(r)}$.
To form the key, a layer-shared compatibility map $A\in\mbb{R}^{d_h\times d_h'}$ first maps each memory patch to token width $d_h$.
We then add positional embeddings, apply the same attention aggregator (including attention AdaLN) used for the input tokens, and apply the same predictor AdaLN used to form the query.
We write $\mcl{S}_{\ell,m,t}^{\mrm{attn}}(Az^{(r)})$ for this memory attention context.
With $H$ predictor heads, the full predictor is
\begin{align}\label{app-eqn:predictor-softmax}
    q_{\ell,h,t}(s)
    &=W_{\ell,h}^{Q}\opn{AdaLN}_{\ell,t}(s)+b_{\ell,h}^{Q},\notag\\
    k_{\ell,h,m,t}^{(r)}
    &=W_{\ell,h}^{K}\opn{AdaLN}_{\ell,t}\!\left(
    \mcl{S}_{\ell,m,t}^{\mrm{attn}}(Az^{(r)})\right)+b_{\ell,h}^{K},\notag\\
    v_{\ell,h,m}^{(r)}
    &=W_{\ell,h}^{V}z_m^{(r)}+b_{\ell,h}^{V},\notag\\
    g_{\ell,m,t}^{\mrm{lpld}}(s)
    &=\alpha_{\ell,t}\odot\left[
    \sum_{h=1}^{H}W_{\ell,h}^{O}
    \left(\sum_{r=1}^{R}
    \opn{sm}_r\left\{
    \frac{q_{\ell,h,t}(s)^\top k_{\ell,h,m,t}^{(j)}}{\sqrt{d_h/H}}
    \right\}_{j\in[R]}
    v_{\ell,h,m}^{(r)}\right) + b_{\ell}^{O}\right].
\end{align}
Here $A$ acts patchwise, $W_{\ell,h}^{Q},W_{\ell,h}^{K}\in\mbb{R}^{(d_h/H)\times d_h}$, and $W_{\ell,h}^{V}\in\mbb{R}^{(d_h/H)\times d_h'}$.
When $d_h'=d_h$, we may take $A=I$ without adapter parameters; otherwise, $A$ is learned.
The shared spatial bank costs $RMd_h'$ parameters once, plus $d_hd_h'$ if the compatibility map $A$ is learned.

\subsubsection{Shared cross-attention (SCA)}
SCA uses untied key and value banks $\{k^{(r)},v^{(r)}\in\mbb{R}^{d_h'}\}_{r=1}^{R}$ shared across all layers and patch positions.
With $H$ predictor heads, the full predictor is
\begin{align}\label{app-eqn:predictor-sca}
 q_{\ell,h,t}(s)&=W_{\ell,h}^{Q}\opn{AdaLN}_{\ell,t}(s)+b_{\ell,h}^{Q},\notag\\
 k_{\ell,h}^{(r)}&=W_{\ell,h}^{K}k^{(r)}+b_{\ell,h}^{K},\notag\\
 v_{\ell,h}^{(r)}&=W_{\ell,h}^{V}v^{(r)}+b_{\ell,h}^{V},\notag\\
 g_{\ell,t}^{\mrm{sca}}(s)
 &=\alpha_{\ell,t}\odot\left[
 \sum_{h=1}^{H}W_{\ell,h}^{O}
 \left(\eta_{\ell,h,t}(s)\sum_{r=1}^{R}
 \opn{sm}_r\left\{
 \frac{q_{\ell,h,t}(s)^\top k_{\ell,h}^{(j)}}{\sqrt{d_h/H}}
 \right\}_{j\in[R]}v_{\ell,h}^{(r)}\right)
 +b_{\ell}^{O}\right].
\end{align}
Here $W_{\ell,h}^{Q}\in\mbb{R}^{(d_h/H)\times d_h}$ and $W_{\ell,h}^{K},W_{\ell,h}^{V}\in\mbb{R}^{(d_h/H)\times d_h'}$.
The positive headwise gain is
\[
\eta_{\ell,h,t}(s)=\opn{softplus}(a_{\ell,h}^{\top}q_{\ell,h,t}(s)+b_{\ell,h}^{\eta}),
\]
with learned $a_{\ell,h}\in\mbb{R}^{d_h/H}$ and scalar $b_{\ell,h}^{\eta}$.
The shared libraries cost $2Rd_h'$ parameters once, while their projections and the query/output maps remain layer-specific.

\subsubsection{Shared MLP (SMLP)}
SMLP uses the same shared-library structure and $H$ predictor heads as SCA but replaces normalized softmax retrieval with unnormalized GELU retrieval.
With $H$ predictor heads, the full predictor is
\begin{align}\label{app-eqn:predictor-gelu}
 q_{\ell,h,t}(s)&=W_{\ell,h}^{Q}\opn{AdaLN}_{\ell,t}(s)+b_{\ell,h}^{Q},\notag\\
 k_{\ell,h}^{(r)}&=W_{\ell,h}^{K}k^{(r)},\notag\\
 v_{\ell,h}^{(r)}&=W_{\ell,h}^{V}v^{(r)},\notag\\
 g_{\ell,t}^{\mrm{smlp}}(s)
 &=\alpha_{\ell,t}\odot\left[
 \sum_{h=1}^{H}W_{\ell,h}^{O}
 \left(\sum_{r=1}^{R}\opn{gelu}\!\left\{
 \frac{q_{\ell,h,t}(s)^\top k_{\ell,h}^{(r)}}{\sqrt{d_h/H}}
 +(b_{\ell,h}^{K})^\top k^{(r)}\right\}v_{\ell,h}^{(r)}\right)
 +b_{\ell}^{O}\right].
\end{align}
K and V projections are bias-free; the Q and O projections have ordinary affine biases.
$(b_{\ell,h}^{K})^\top$ is an additional learned row of the key projection that generates the GELU bias from the shared key prototype.
Note that all predictor heads are evaluated for every token: there is no expert router, softmax normalization, or separate retrieval gain.
Both the single-head and multihead forms retain the MLP's unnormalized GELU expansion while sharing the underlying representations across depth.
The shared libraries cost $2Rd_h'$ parameters once.
Including the layer-specific weight matrices but excluding biases and AdaLN, the cost is
$2Rd_h'+2Ld_h(d_h+d_h')$, compared with $2LRd_h$ for independent width-$R$ MLPs.

\subsection{From a single layer to a residual stack}
\label{app:residual-stack}
The construction in Sec.~\ref{sec:dit-decomposition} isolates aggregation and prediction in a single layer without residual connections.
In this section, we explain its extension to a \emph{residual stack}, a multilayer architecture with residual connections.

Let $h_\ell$ denote the hidden token sequence entering layer $\ell$ at diffusion time $t$.
We absorb normalization, AdaLN, and output gates into the attention branch $\mcl{S}_{\ell,t}^{\mrm{attn}}$ and the tokenwise predictor $g_{\ell,t}$ (an MLP in DiT).
The two residual additions give
\begin{align}\label{app-eqn:residual-stack}
 s_\ell&=h_\ell+\mcl{S}_{\ell,t}^{\mrm{attn}}(h_\ell),\notag\\
 h_{\ell+1}&=s_\ell+g_{\ell,t}(s_\ell)
 =h_\ell+\underbrace{\mcl{S}_{\ell,t}^{\mrm{attn}}(h_\ell)
 +g_{\ell,t}\!\left(h_\ell+\mcl{S}_{\ell,t}^{\mrm{attn}}(h_\ell)\right)}_{F_{\ell,t}(h_\ell)}.
\end{align}
The aggregator interpretation also applies to $s_\ell$: it combines the token's existing state with attention-selected context, so the identity path retains a contribution from the token itself.
The predictor then maps this contextualized state to a residual feature update.

We emphasize that \eqref{app-eqn:residual-stack} motivates interpreting the residual stack as \emph{iterative denoising in feature space}: successive aggregation and prediction updates refine the internal residual stream from which the final clean-image estimate is read out.
This view is consistent with interpretations of residual networks as unrolled iterative estimation and feature refinement \citep{greff2017highway,jastrzebski2018residual}.
This interpretation also motivates consulting a shared memory across depth, providing a reusable bank of learned representations at each refinement step.

However, the clean factorization in a single layer does not necessarily extend to a residual stack.
Patch projection and positional embedding are applied once at network entry; later layers operate on already contextualized tokens.
Although the shared-channel-operator decomposition of attention still holds layerwise, a token can depend on many original patches, and VO or predictor updates can alter later QK routing.
The division of labor between QK and VO is therefore less clear-cut than in a single layer.
Similarly, the MLP in one layer can affect attention in subsequent layers, and attention can in turn affect subsequent MLP updates.

\begingroup
\raggedbottom
\setcounter{topnumber}{4}
\setcounter{bottomnumber}{4}
\setcounter{totalnumber}{6}
\renewcommand{\topfraction}{0.95}
\renewcommand{\bottomfraction}{0.95}
\renewcommand{\textfraction}{0.05}
\renewcommand{\floatpagefraction}{0.65}
\setlength{\floatsep}{6pt plus 2pt minus 2pt}
\setlength{\textfloatsep}{8pt plus 2pt minus 2pt}
\setlength{\intextsep}{8pt plus 2pt minus 2pt}
\makeatletter
\setlength{\@fptop}{0pt}
\setlength{\@fpsep}{6pt plus 2pt minus 2pt}
\makeatother
\section{Experimental details}
\label{app:experimental-details}

\paragraph{Datasets.}
Table~\ref{tab:datasets} summarizes the seven datasets used in our real-image experiments.
All real images are represented as $64\times64$ RGB arrays with pixel values scaled to $[-1,1]$; training is unconditional and does not use class, identity, or attribute labels.

\begin{table}[!htbp]
    \centering
    \small
    \setlength{\tabcolsep}{5pt}
    \begin{tabular}{@{}p{0.37\textwidth}p{0.33\textwidth}rr@{}}
        \toprule
        Dataset & Content & Training & Evaluation \\
        \midrule
        \raggedright CelebA~\citep{liu2015faceattributes} & Aligned celebrity photographs & 162,770 & 19,962 \\
        \addlinespace[3pt]
        \raggedright FFHQ~\citep{karras2019stylebased} & Photographic faces & 60,000 & 5,000 \\
        \addlinespace[3pt]
        \raggedright AAHQ~\citep{liu2021blendgan} & Artistic portraits & 15,366 & 1,921 \\
        \addlinespace[3pt]
        \raggedright AFHQ-v2~\citep{choi2020starganv2,karras2021aliasfree} & Cat, dog, and wild-animal faces & 14,335 & 1,467 \\
        \addlinespace[3pt]
        \raggedright iCartoonFace~\citep{zheng2020cartoon} & Cartoon-face recognition crops & 388,710 & 22,496 \\
        \addlinespace[3pt]
        \raggedright STL-10~\citep{coates2011analysis} & Natural-object images & 104,500 & 8,000 \\
        \addlinespace[3pt]
        \raggedright ImageNet64~\citep{chrabaszcz2017downsampled} & Downsampled natural images & 1,281,167 & 50,000 \\
        \bottomrule
    \end{tabular}
    \caption{\textbf{Datasets used throughout the paper.}
    Counts refer to prepared images after filtering.
    Evaluation uses the held-out splits described below; the ImageNet64 evaluation set is its official validation split.}
    \label{tab:datasets}
\end{table}

We retain CelebA's official train/test splits and center-crop its images to $178\times178$ before bilinear resizing.
For FFHQ, we retain the official 60,000-image training partition and divide the remaining 10,000 images equally into validation and test sets.
For AAHQ, crops from the same source artwork remain in the same split.
AFHQ-v2 pools all three animal categories, retains the official evaluation split, and removes exact pixel duplicates.
iCartoonFace uses its official train/test splits after filtering invalid crops and exact pixel duplicates; identities may overlap across splits.
STL-10 training combines 100,000 unlabeled images with 4,500 labeled training images, reserving the remaining 500 labeled images for validation and retaining the official test split.
The synthetic GMRF experiment instead uses the analytically specified prior in Appendix~\ref{app:exp1-1} and does not train a DiT.

\paragraph{Common training protocol.}
\label{app:common-training}
All real-image DiTs divide each image into $M=256$ nonoverlapping $4\times4$ patches, giving a $16\times16$ patch grid with $d_p=3\cdot4^2=48$ coordinates per patch.
Following the direct-clean-image formulation in Sec.~\ref{sec:setup}, we use the linear Gaussian interpolation
\begin{align}
    X_t=tX_0+(1-t)\varepsilon,\qquad
    \varepsilon\sim\N(0,I),\qquad
    \opn{logit}(t)\sim\N(-0.8,0.8^2).
    \label{eq:shared-predictor-path}
\end{align}
Thus $a_t=t$, $b_t=1-t$, and the normalized observation $Y_t=X_t/t$ has noise level $\sigma_t=(1-t)/t$.
Let $f_\theta(X_t,t)$ denote the model's direct clean-image prediction.
The implementation's VP velocity objective uses $c(t)=\max(1-t,0.05)$ and minimizes
\begin{align}
    \mathcal L(\theta)
    &=\frac{1}{Md_p}\mbb E\left\|
        \frac{f_\theta(X_t,t)-X_t}{c(t)}
        -\frac{X_0-X_t}{c(t)}\right\|_2^2 \nonumber \\
    &=\frac{1}{Md_p}\mbb E\left[
        \frac{\|f_\theta(X_t,t)-X_0\|_2^2}{c(t)^2}\right].
    \label{eq:shared-predictor-loss}
\end{align}
The denominator floor is applied to both predicted and target velocities.
Away from the floored region, the target velocity is $X_0-\varepsilon$.
The floor changes the time weighting but retains the conditional mean as the pointwise population optimum.
Unless specified otherwise below, we use AdamW with a constant learning rate of $2\times10^{-4}$, $(\beta_1,\beta_2)=(0.9,0.999)$, zero weight decay, gradient clipping at norm 1, and global batch size 32.
We maintain an exponential moving average (EMA) of the weights with decay 0.999 and use EMA weights for all reported measurements and generated samples.
The experiment-specific settings below give the model sizes, training durations, seeds, precision, and any departures from these defaults.

\subsection{Experiment 1-1: Exact denoising context under synthetic GMRF priors}
\label{app:exp1-1}
We conduct experiments to illustrate the concept of local context and examine the prediction of Cor.~\ref{crl:bounded-range}: within a specified noise interval, the radius certifying an upper bound on locality error, which we call the \emph{certified radius}, is nondecreasing as the effective noise level increases.

\subsubsection{Synthetic Gaussian MRF prior.}
We consider the Gaussian MRF image prior introduced in Thm.~\ref{thm:gmrf-locality}, for a single-channel image consisting of $16\times16$ scalar patches of dimension $d_p=1$.
We draw a graph $G_0$ by sampling $384$ distinct unordered pairs uniformly without replacement from the $\binom{256}{2}$ possible pairs, and independently draw each retained edge's connection strength as $\omega_{ij}\sim\operatorname{Uniform}(0.5,1.5)$.
We choose $384$ edges to give an average of three neighbors per patch: each undirected edge contributes to the degree of both endpoints, so the mean degree is $2\times384/256=3$.
Then, let $A_{ij}=A_{ji}=\omega_{ij}$ for each sampled edge $\{i,j\}$ and zero otherwise.
We define the prior by
\begin{align}
    X_0\sim\mathcal N(0,Q^{-1}),\qquad
    Q=0.2 I+\operatorname{diag}(A\mathbf{1})-A.
\end{align}
Here $\mathbf{1}$ is the all-ones vector, so $A\mathbf{1}$ contains the row sums of $A$. We use the weighted graph Laplacian $\operatorname{diag}(A\mathbf{1})-A$ so that the prior penalizes differences between connected patches, with larger edge weights encouraging stronger agreement.
The positive diagonal shift ensures that $Q$ is positive definite even if the graph is disconnected.
We evaluate 41 logarithmically spaced noise levels from $10^{-2}$ to $10^2$, together with zero and the corollary's endpoint $\sigma_{\mrm{abs}}$, giving 43 noise levels in total.

\subsubsection{Exact and certified radii for absolute locality error}
Since we know the population prior and its graph, we can compute the global population denoiser and the exact local denoiser for each graph-radius context $B_r(m)=\{l:d_{G_0}(m,l)\le r\}$.
Here $d_{G_0}$ is the unweighted shortest-path distance; vertices in different components have infinite distance and zero cross-component influence.
This allows us to compare the minimum radius achieving a given locality error with the radius certifying its upper bound in Cor.~\ref{crl:bounded-range}.
Specifically, observe $Y_\sigma=X_0+\sigma\varepsilon$ with independent $\varepsilon\sim\mathcal N(0,I)$, and write $J_\sigma=(I+\sigma^2Q)^{-1}$ for the population denoiser's Jacobian.
The global denoiser and the local denoiser on $B_r(m)$ are
\begin{align}
    D^\ast_{m,\sigma}(y)&=(J_\sigma y)_m,\\
    D^{\ast,r}_{m,\sigma}(y)&=(Q^{-1})_{m,B_r(m)}
    \bigl[(Q^{-1})_{B_r(m),B_r(m)}+\sigma^2I\bigr]^{-1}y_{B_r(m)}.
\end{align}
Gaussian conditioning gives the locality error in closed form:
\begin{align}
    &\mathbb E\bigl|D^\ast_{m,\sigma}(Y_\sigma)-D^{\ast,r}_{m,\sigma}(Y_\sigma)\bigr|^2\nonumber\\
    &\quad=(J_\sigma Q^{-1})_{mm}
    -(Q^{-1})_{m,B_r(m)}
    \bigl[(Q^{-1})_{B_r(m),B_r(m)}+\sigma^2I\bigr]^{-1}
    (Q^{-1})_{B_r(m),m}.
    \label{eq:exp1-1-exact-error}
\end{align}
Based on this expression, we compute the \emph{exact radius} for an absolute tolerance $\epsilon>0$ as
\begin{align}
    r_{\mrm{ext}}(m,\sigma;\epsilon)
    :=\min\left\{r\in\mathbb Z_{\ge0}:
    \mathbb E\bigl|D^\ast_{m,\sigma}(Y_\sigma)-D^{\ast,r}_{m,\sigma}(Y_\sigma)\bigr|^2\le\epsilon\right\}.
\end{align}
Corollary~\ref{crl:bounded-range} also gives the certified radius.
For this calculation we use $q_-=0.2$ and $q_+=0.2+2\max_i\sum_j A_{ij}\approx21.420774$.
The former is the exact minimum eigenvalue of $Q$, while the latter is a conservative upper bound on its maximum eigenvalue; the theorem's bound remains valid with this spectral enclosure.
Using $\chi_t$ and $\rho_t$ as defined in Appendix~\ref{app:gmrf-locality}, evaluated at $\sigma_t=\sigma$, the \emph{certified radius} is
\begin{align}
    r_{\mrm{crt}}(\sigma;\epsilon)
    :=\min\left\{r\in\mathbb Z_{\ge0}:
    \frac{(1-1/\sqrt{\chi_t})^2}{q_-(1+\sigma^2q_-)}\rho_t^{2r}\le\epsilon\right\}.
\end{align}
At zero noise, both radii are zero.
The certified radius is common for all target patches $m$.
We use tolerances $\epsilon=10^{-3}$ in the main figure.

\subsubsection{Sensitivity radius as a measure of statistical dependence}
Along with these exact and certified radii, we define a sensitivity radius based on the input--output Jacobian, which measures how perturbing each observed patch affects the predicted target patch.
For this Gaussian prior, $(J_\sigma)_{ml}=\partial D^\ast_{m,\sigma}(y)/\partial y_l$, with $J_\sigma=(I+\sigma^2Q)^{-1}$.
We square its entries to measure sensitivity energy and normalize over input patches as $(J_\sigma)_{ml}^2/[\sum_j(J_\sigma)_{mj}^2]$.
This normalized derivative energy is the fraction of the target $m$ patch's total squared sensitivity attributable to input patch $l$.
It measures relative influence, so a common reduction in the magnitude of all derivatives does not by itself change the normalized map.
Using this normalization, we define the \emph{$q$-sensitivity radius}:
\begin{align}
    r_{\mrm{sns}}(m,\sigma;q)
    :=\min\left\{r\in\mathbb Z_{\ge0}:
    \sum_{l\in B_r(m)}(J_\sigma)_{ml}^2 \bigg/ \left[\sum_j(J_\sigma)_{mj}^2\right]
    \ge q\right\}.
\end{align}
The self-patch is included, and we evaluate $q\in\{0.90,0.95\}$.
Unlike the exact and certified radii, this radius concerns relative derivative energy rather than an absolute locality-error tolerance, so the three radii need not agree.
For any prior with finite second moments, the population denoiser satisfies the posterior-covariance identity \citep{dytso2021general}
\begin{align}\label{eq:denoiser-covariance}
    \frac{\partial D_{m,t}^\ast(y)}{\partial y_l}
    =\sigma_t^{-2}\opn{Cov}(X_{0,m},X_{0,l}\mid Y_t=y,t),
    \qquad \sigma_t>0.
\end{align}
Thus, its Jacobian reflects the dependence between clean patches that remains after observing the noisy image.
For the Gaussian prior considered here, this Jacobian is independent of $y$.

\subsubsection{Interpretation of Fig.~\ref{fig:exp1-1-synthetic}}

\paragraph{Fig.~\ref{fig:exp1-1-synthetic}(a) Local context need not be spatially contiguous.}
The first panel shows the radius-3 local contexts of the two selected patches, containing 12 and 50 nodes, respectively, including the targets.
Because these sets are defined by graph distance, they need not be square or spatially contiguous: a patch that is distant on the image grid can still be nearby in $G_0$.
Each ball is nevertheless connected in the graph.
Thus, locality is relative to the dependence structure of the prior, rather than necessarily to proximity on the image grid.

\paragraph{Fig.~\ref{fig:exp1-1-synthetic}(b) The certified radius is nondecreasing over a bounded noise interval.}
The second panel compares the exact $r_{\mrm{ext}}$ and certified radii $r_{\mrm{crt}}$ at the absolute error tolerance $\epsilon=10^{-3}$.
At low noise, the target patch's own observation is accurate, so little additional context is needed.
As noise increases, observations of correlated patches become more useful, and both radii are nondecreasing over the specified interval in this realization.
$r_{\mrm{crt}}$ is conservative: its maximum over the sweep is 16 hops, compared with 4 hops for each exact radius.
The dotted line marks $\sigma_{\mrm{abs}}\approx1.128962$, the endpoint specified in Cor.~\ref{crl:bounded-range}.
The corollary guarantees monotonicity of the certified radius on $[0,\sigma_{\mrm{abs}}]$; the exact-radius behavior is an empirical observation in this example.
Beyond this interval, the radii eventually decrease as the observations provide progressively less information about the clean signal.
Indeed, $J_\sigma\sim\sigma^{-2}Q^{-1}$ as $\sigma\to\infty$, so both the global and local posterior means approach the zero prior mean and their absolute difference shrinks.
Note that this does not mean that reconstruction improves: locality error measures the discrepancy between the two denoisers, not their error relative to the clean image.

\paragraph{Fig.~\ref{fig:exp1-1-synthetic}(c) Relative sensitivity reveals how the denoiser distributes influence across patches.}
The third panel shows that the sensitivity radius $r_{\mrm{sns}}$ can remain broad even when the absolute locality error becomes small.
As noise increases from low levels, a larger graph neighborhood is needed to capture the same fraction of the total derivative energy, consistent with the initial growth of $r_{\mrm{ext}}$ and $r_{\mrm{crt}}$.
At extreme noise, however, the absolute response to every observation weakens, while its relative contribution approaches a constant:
\begin{align}
    \frac{(J_\sigma)_{ml}^2}{\sum_j (J_\sigma)_{mj}^2}\longrightarrow
    \frac{((Q^{-1})_{ml})^2}{\sum_j((Q^{-1})_{mj})^2}
    \qquad\text{as }\sigma\to\infty.
\end{align}
At extreme noise, the posterior covariance approaches the prior covariance $Q^{-1}$, while the factor $\sigma^{-2}$ in the Jacobian makes its absolute entries vanish.
This common scale cancels under normalization, leaving a pattern determined by the prior's covariances, which can extend to distant patches.
These results show why a broad sensitivity radius can coexist with a small absolute locality error: the radius describes the distribution of influence, while the error also depends on its absolute strength.
Note that in this synthetic setting, we do not claim that the sensitivity radius must increase monotonically with noise, or that observations outside the corresponding neighborhood can be removed without affecting denoising accuracy.

This interpretation provides a link to denoisers trained on real-image datasets: their input--output Jacobians also describe responses to infinitesimal perturbations, so the sensitivity radius $r_{\mrm{sns}}$ can be defined in the same way to characterize local functional context even when the population prior and exact local conditional means are unknown.
Furthermore, due to \eqref{eq:denoiser-covariance}, to the extent that learned denoisers approximate the population denoisers, their Jacobian-based sensitivity provides a measure of statistical dependence relevant to denoising.

\subsection{Experiment 1-2: Denoising context in DiTs trained on real-image datasets}\label{app:exp1-2}

We ask whether the denoising-relevant statistical dependence in image datasets exhibits a similar pattern to the exponentially local GMRF model studied in Appendix~\ref{app:exp1-1}.
We use trained DiTs as estimators of the population denoiser and its Jacobian; through the posterior-covariance identity \eqref{eq:denoiser-covariance}, their sensitivity fields provide an empirical view of this dependence.
We use a DiT with approximately 20M parameters, comprising $L=10$ transformer layers, token dimension $d_h=320$, $\tilde{H}=10$ attention heads, and MLP hidden width 1280.
We denote its clean-patch prediction by $D_{m,t}^{\mrm{dit}}$ and use the common patch representation and training protocol at the start of Appendix~\ref{app:experimental-details}.
Following the normalized-observation convention in Sec.~\ref{sec:setup}, $D_{m,t}^{\mrm{dit}}(y)$ denotes the prediction obtained by feeding $ty$ and $t$ to the trained network; thus, all input derivatives below are with respect to $y$.

\paragraph{Experiment-specific training.}
We train models on CelebA, iCartoonFace, ImageNet64, and STL-10 using the datasets and splits described above.
We analyze checkpoints at epochs 400, 400, 800, and 800, respectively.
The non-ImageNet runs use batch size 32 and FP32 on single GPUs, mostly RTX 2080 Ti GPUs; the ImageNet checkpoint uses batch size 256 and BF16 on an H200.

\subsubsection{Sensitivity context set based on input-output Jacobian}
Our primary measurement is the Jacobian of output patch $m$ with respect to input patch $l$, extending the scalar-patch sensitivity in Appendix~\ref{app:exp1-1}.
For real-image priors, the population derivative $\partial D_{m,t}^\ast(y)/\partial y_l$ is unavailable, so we estimate it using the trained denoiser's Jacobian,
\begin{align}
    (J_t(y))_{ml}:=\frac{\partial D_{m,t}^{\mrm{dit}}(y)}{\partial y_l}
    \in\mbb{R}^{d_p\times d_p}.
\end{align}
Since each patch contains multiple spatial and color coordinates, we measure its sensitivity by the squared Frobenius norm of this block, which sums the squared derivatives over all input--output coordinate pairs.
The squared Frobenius norm extends the squared scalar Jacobian entries used in Appendix~\ref{app:exp1-1} to vector-valued patches.
This also follows the use of input-gradient sensitivity to study denoiser locality in \citet[Sec.~3.1]{niedoba2025towards}.

For each dataset, we evaluate 1,024 held-out images at 12 noise levels from $\sigma_t=0.02$ to $49$.
In practice, we estimate the squared Frobenius norm using vector--Jacobian products with four independent Rademacher output probes per target patch, following Hutchinson-type Monte Carlo estimation \citep{avron2011randomized}.
We divide by $d_p$ to obtain sensitivity per output coordinate, average over held-out clean images and Gaussian corruption noise, and then normalize over input patches to define the \emph{normalized sensitivity weight}:
\begin{align}
    w_{ml}(t):=
    \frac{d_p^{-1}\mbb{E}_{X_0,\varepsilon}\|(J_t(Y_t))_{ml}\|_F^2}
    {\sum_{j=1}^M d_p^{-1}\mbb{E}_{X_0,\varepsilon}\|(J_t(Y_t))_{mj}\|_F^2}.
    \label{eq:exp1-2-sensitivity-energy}
\end{align}
Thus, $w_{ml}(t)$ is the fraction of patch $m$'s expected squared sensitivity attributable to input patch $l$, with $\sum_l w_{ml}(t)=1$ whenever the total sensitivity is positive.

For real images, the prior's dependence graph is unknown so we cannot compute the exact radius of a local context as in the synthetic case.
Instead, we describe local context in terms of spatial distance and patch count.
Specifically, we use the Chebyshev grid distance, a patch hop, defined as $d_{\mrm{grid}}(m,l)=\max\{|i_m-i_l|,|j_m-j_l|\}$, where $i_m,j_m$ denote the row and column indices of patch $m$, and define $B_r(m)=\{l:d_{\mrm{grid}}(m,l)\le r\}$.
These neighborhoods are square regions clipped to the image boundary, with radius measured in patch hops.
Using the normalized derivative energy in \eqref{eq:exp1-2-sensitivity-energy}, we define the \emph{sensitivity radius} by
\begin{align}
    r_{\mrm{sns}}(m,t;q)
    &:=\min\left\{r\in\mathbb Z_{\ge0}:\sum_{l\in B_r(m)}w_{ml}(t)\ge q\right\}.
\end{align}
The radius is the smallest centered Chebyshev ball containing at least a fraction $q$ of the normalized derivative energy.
This ball measures the spatial extent of the context, but the functionally relevant local context set need not be ball-shaped.
To allow arbitrary shapes, we also define a context set containing the fewest patches that together account for at least a fraction $q$ of the normalized derivative energy and the \emph{sensitivity count} as
\begin{align}
    \mcl{B}_{m,t}(q)
    \in\operatorname*{arg\,min}_{B\subseteq[M]}
    \left\{|B|:\sum_{l\in B}w_{ml}(t)\ge q\right\},\qquad 
    K_q(m,t):=|\mcl{B}_{m,t}(q)|.
    \label{eq:exp1-2-context-set}
\end{align}
We construct this set by adding patches in descending order of $w_{ml}(t)$ until their cumulative energy reaches $q$, breaking ties by patch index.
The selected patches may be scattered or disconnected; the target patch is eligible for selection.
For the population curves, we compute each target's radius and count from the image-averaged energy in \eqref{eq:exp1-2-sensitivity-energy}, then average these quantities over all 256 target patches.
For the displayed sensitivity heatmaps, we apply the same construction to the normalized derivative energy of the individual noisy image.

Together, these quantities distinguish spatially broad context from sparse influence carried by a few distant patches.
The posterior-covariance identity \eqref{eq:denoiser-covariance} connects these measurements to statistical dependence remaining after observing the noisy image.
Under the Jacobian approximation described above, these sensitivity-based contexts estimate the patches carrying most of the expected squared posterior cross-covariance with the target.
They thus characterize the dataset's denoising-relevant dependence through a learned predictor.

\subsubsection{Attention context set based on attention mass}
Additionally, we average the normalized attention matrices computed from the residual-stream features of $D_{m,t}^{\mrm{dit}}$ over heads and layers, and use the resulting attention mass in place of $w_{ml}(t)$.
Specifically, let $A_{ml,t}^{(\ell,h)}(y)$ denote the softmax attention weight from query patch $m$ to source patch $l$ in head $h$ of layer $\ell$, evaluated when the network receives $ty$ and time $t$.
These are the spatial weights in \eqref{app-eqn:attention-weights}, computed from each layer's features after layer normalization and AdaLN modulation, as described in Appendix~\ref{app:residual-stack}.
Averaging over the same held-out images and corruption noise gives the \emph{mean attention weight}:
\begin{align}
    w_{ml}^{\mrm{attn}}(t)
    :=\mbb{E}_{X_0,\varepsilon}\left[
    \frac{1}{L\tilde{H}}\sum_{\ell=1}^{L}\sum_{h=1}^{\tilde{H}}
    A_{ml,t}^{(\ell,h)}(Y_t)\right].
    \label{eq:exp1-2-attention-mass}
\end{align}
Each attention row sums to one over source patches, so $\sum_l w_{ml}^{\mrm{attn}}(t)=1$ without further normalization.
We average the softmax weights themselves, without squaring them or weighting them by residual gates or value/output projections; this measures how attention distributes mass across token positions along the residual stream.
Note that both $w_{ml}^{\mrm{attn}}(t)$ and $w_{ml}(t)$ assign normalized weights to input patches $l$ for a given target patch $m$, but they measure different quantities: the former describes how attention distributes mass within the network, whereas the latter measures the relative squared sensitivity of the final denoised prediction to each input patch. 
In this sense, although both weights lie in $[0,1]$ and sum to one, equal values do not imply equal influence.

Using the same neighborhoods and coverage threshold as for the sensitivity radius, we define the \emph{attention radius} by
\begin{align}
    r_{\mrm{attn}}(m,t;q)
    &:=\min\left\{r\in\mathbb Z_{\ge0}:
    \sum_{l\in B_r(m)}w_{ml}^{\mrm{attn}}(t)\ge q\right\}.
\end{align}
Mirroring \eqref{eq:exp1-2-context-set}, we additionally define the attention context set and \emph{attention count} as
\begin{align}
    \mcl{B}_{m,t}^{\mrm{attn}}(q)
    \in\operatorname*{arg\,min}_{B\subseteq[M]}
    \left\{|B|:\sum_{l\in B}w_{ml}^{\mrm{attn}}(t)\ge q\right\},\qquad
    K_q^{\mrm{attn}}(m,t):=|\mcl{B}_{m,t}^{\mrm{attn}}(q)|.
    \label{eq:exp1-2-attention-context-set}
\end{align}
The radius measures the smallest centered square capturing a fraction $q$ of the mean attention weight, while the count measures the fewest patches carrying that fraction, allowing arbitrary spatial arrangements.
We obtain the attention context set by selecting patches in descending order of $w_{ml}^{\mrm{attn}}(t)$, with the same tie-breaking rule as in \eqref{eq:exp1-2-context-set}.
These quantities are computed after averaging the attention maps over heads, layers, and images; the population curves then average the resulting radii or counts over all 256 target patches.
For an individual-image heatmap, we average only over heads and layers before selecting its context set.

\subsubsection{Interpretation of Fig.~\ref{fig:exp1-2-real}}

\paragraph{Estimated dependence fields resemble the exponentially local GMRF setting.}
Figure~\ref{fig:exp1-2-real} reveals a pattern of denoising-relevant statistical dependence similar to that of the synthetic GMRF prior.
Sensitivity is concentrated at lower noise, while its radius generally increases as noise grows, echoing the broadening of relative sensitivity in Fig.~\ref{fig:exp1-1-synthetic}.
Through the posterior-covariance interpretation in \eqref{eq:denoiser-covariance}, this suggests that predicting a patch relies on a broader set of statistical dependencies when the observation is noisier.
Figures~\ref{fig:exp1-2-metrics} and~\ref{fig:exp1-2-extra-maps} corroborate this overall broadening across additional datasets and at both 90\% and 95\% coverage curves.
The analytic GMRF setting thus provides a useful model for interpreting the estimated dependence fields of real-image datasets, beyond its role as a tractable synthetic example.

\begin{figure}[!htbp]
    \centering
    \includegraphics[width=\textwidth]{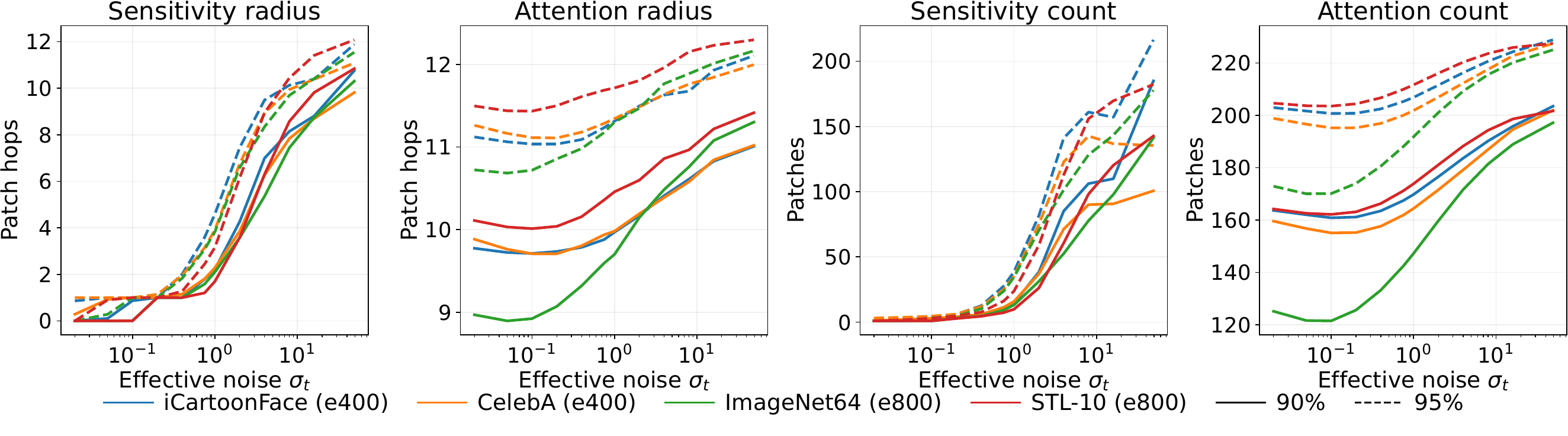}
    \caption{Sensitivity and attention radii and patch counts on four datasets. Solid/dashed curves show 90\%/95\% coverage, computed from population maps over 1,024 held-out images and averaged over all 256 target patches.}
    \label{fig:exp1-2-metrics}
\end{figure}

\paragraph{The context set need not be spatially contiguous.}
The estimated dependence fields in the heatmaps show that patches carrying most of the sensitivity can form irregular or disconnected regions rather than a centered square.
At lower noise, influence is concentrated near the target and selected nearby structures; at higher noise, the displayed context sets spread over a larger part of the image.
The radius and patch count therefore describe complementary aspects of context: its spatial extent and the number of patches carrying a specified fraction of influence.

\begin{figure}[!htbp]
    \centering
    \includegraphics[width=\textwidth]{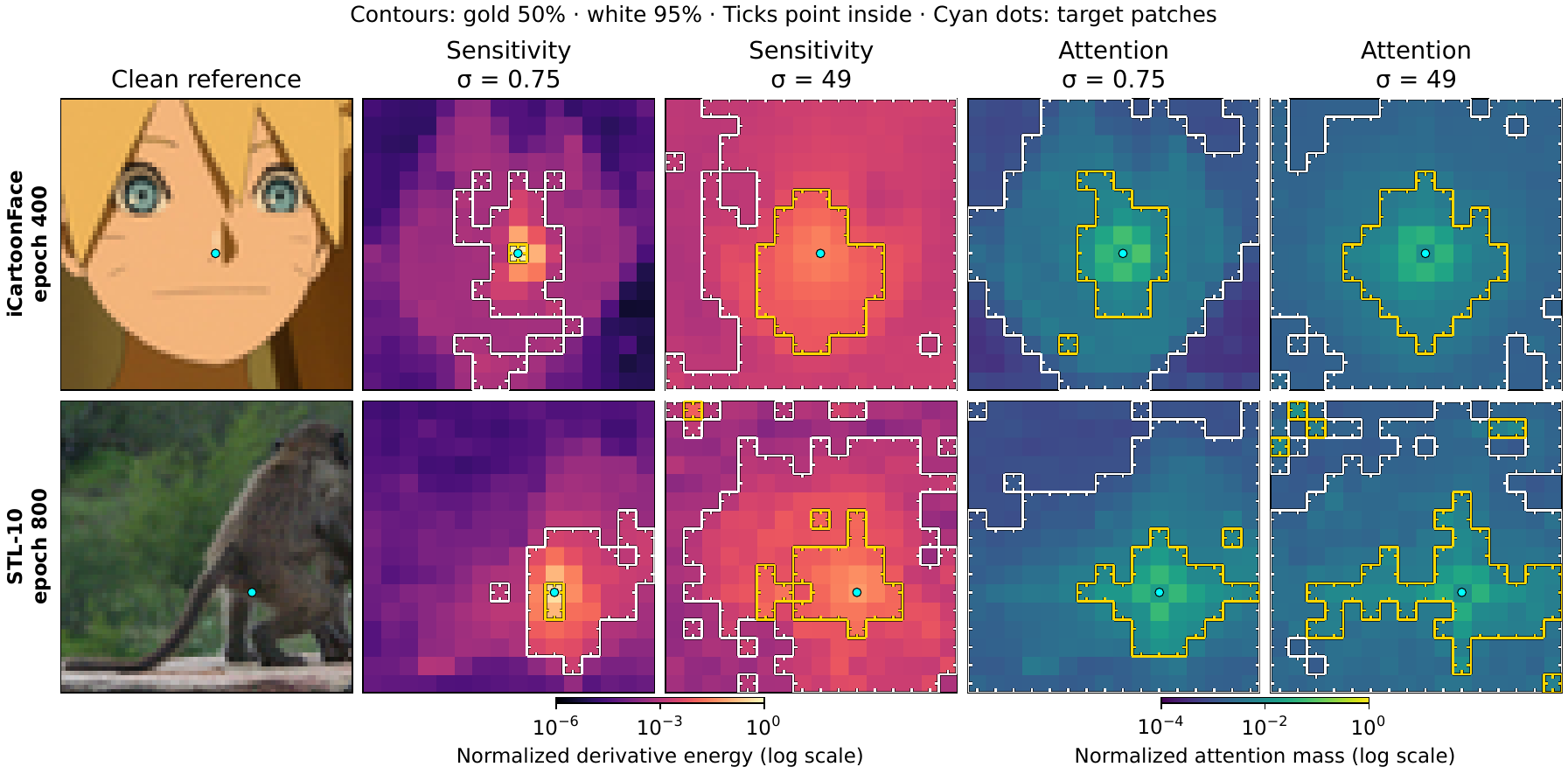}
    \caption{Additional context maps for iCartoonFace and STL-10 at $\sigma_t=0.75,49$, complementing the CelebA and ImageNet examples in Fig.~\ref{fig:exp1-2-real}. Colors are logarithmic; gold/white contours mark 50\%/95\% coverage, ticks point inward, and cyan dots mark target patches.}
    \label{fig:exp1-2-extra-maps}
\end{figure}

\paragraph{Attention exhibits a related pattern of context formation.}
Head- and layer-averaged attention also tends to spread over a broader region at higher noise, although its mass is already more diffuse than sensitivity at low noise.
Attention weights and input--output Jacobians are different quantities: the former are softmax-normalized logits between intermediate tokens, whereas the latter measure how the final denoised prediction responds to the original input.
There is therefore no requirement that their spatial patterns agree.
Their observed agreement in overall broadening is consistent with attention adapting its aggregation to the dataset's denoising-relevant dependence structure and contributing to the formation of noise-dependent local context, as suggested by the aggregator interpretation in Sec.~\ref{sec:attention-aggregator} and Appendix~\ref{app:residual-stack}.
The additional radius/count curves and heatmaps in Figs.~\ref{fig:exp1-2-metrics} and~\ref{fig:exp1-2-extra-maps} support this interpretation.
Note that the comparison concerns a shared qualitative trend rather than equality of context sets: 95\% attention mass and 95\% squared derivative energy describe different distributions.

\subsection{Experiment 2: Replacing MLPs with shared-memory predictors}
\label{app:exp2}


\subsubsection{Motivation}

Our aggregator--predictor framework decomposes the local conditional mean into a local aggregator and a patch predictor (Def.~\ref{dfn:agg-pre}).
Sections~\ref{sec:attention-aggregator} and~\ref{sec:patch-predictors} associate these operations with attention and the DiT MLP, respectively.
This interpretation motivates replacing the MLP with more interpretable predictors based on retrieval from a reusable bank of learned representations, as introduced in Sec.~\ref{sec:patch-predictors}, with the aim of attaining comparable denoising and generative performance at a similar trainable-parameter budget.
If the proposed alternatives were to fail to do so, this would indicate that their representation or training omits useful properties of an ordinary MLP for denoising.
Conversely, a successful replacement provides evidence that the framework supports a useful, explicitly structured alternative.
Additionally, this comparison lets us examine the effects of three design choices on empirical performance: patch-specific versus patch-shared storage, softmax cross-attention versus unnormalized GELU retrieval, and a base memory shared across layers versus independent layerwise memory.
We take a baseline DiT and replace only its patch predictor with an alternative predictor, leaving everything else fixed; the following section provides the details.

\subsubsection{Settings}

\paragraph{Dataset and common backbone.}
We use the CelebA splits, preprocessing, and common training protocol described at the start of Appendix~\ref{app:experimental-details}, without horizontal-flip augmentation.
All models use token width $d_h=256$, depth $L=8$, and $\tilde H=8$ spatial-attention heads.
The patch embedding, spatial-attention architecture, AdaLN conditioning, residual structure, and image-output head are common across families.
Each model is trained jointly from initialization; the common backbone specification does not mean that attention weights are frozen or shared between runs.

\paragraph{Predictor configurations and training.}
We use the implementations in Appendix~\ref{app:predictor-implementation}, with the residual composition in Appendix~\ref{app:residual-stack}.
Table~\ref{tab:shared-predictor-grid} lists all 21 configurations.
For each shared-memory family, we compare $H=1$ and $H=8$ predictor heads at the same stored width, with projected width $d_h/H$ per head.
Thus, the total projected query, key, and value width stays at 256 rather than growing eightfold with the number of heads.
We vary stored width $d_h'$ and choose MLP hidden widths to obtain three approximate parameter budgets around 8.6--8.9M, 10.2M, and 13.0--13.4M parameters.
The table reports exact counts, including the common backbone and all predictor parameters; matching is approximate, rather than exact equality or equal computation.

For training, we use the common objective in Eq.~\eqref{eq:shared-predictor-loss} and train all configurations to 400 epochs, with 5,086 optimizer updates per recorded epoch and 2,034,400 updates in total.
The global batch size is 32, and the default learning rate is constant except for the LPLD recovery intervention below.
All runs use training seed 0.

\begin{table}[!htbp]
    \centering
    \small
    \setlength{\tabcolsep}{8pt}
    \begin{tabular}{lrrrr}
        \toprule
        Predictor & $R$ & $d_h'$ & Parameters, $H=1$ & Parameters, $H=8$ \\
        \midrule
        LPLD & 32 & 128 & 8,935,472 & 8,935,472 \\
        LPLD & 32 & 256 & 10,213,424 & 10,213,424 \\
        LPLD & 32 & 512 & 12,965,936 & 12,965,936 \\
        \midrule
        SCA & 1024 & 256 & 8,642,616 & 8,642,672 \\
        SCA & 1024 & 512 & 10,215,480 & 10,215,536 \\
        SCA & 1024 & 1024 & 13,361,208 & 13,361,264 \\
        \midrule
        SMLP & 1024 & 256 & 8,638,512 & 8,652,848 \\
        SMLP & 1024 & 512 & 10,213,424 & 10,242,096 \\
        SMLP & 1024 & 1024 & 13,363,248 & 13,420,592 \\
        \midrule
        MLP & 640 & 256 & 8,639,536 & --- \\
        MLP & 1024 & 256 & 10,215,472 & --- \\
        MLP & 1792 & 256 & 13,367,344 & --- \\
        \bottomrule
    \end{tabular}
    \caption{\textbf{Complete patch-predictor grid.}
    Each numeric entry in the last two columns specifies one trained configuration.
    All models use $d_h=256$, $L=8$, and $\tilde H=8$; predictor head width is 256 for $H=1$ and 32 for $H=8$.
    $R$ counts image-shaped memories for LPLD, shared prototype pairs for SCA/SMLP, and hidden units for MLP.}
    \label{tab:shared-predictor-grid}
\end{table}

\paragraph{LPLD computation and noise groups.}
LPLD's key path applies spatial attention to each of the $R$ full memory grids at every layer and noise level.
This is substantially more expensive than projecting a patch-shared prototype bank: distinct sampled times require distinct conditioned memory computations.
We therefore draw four independent times per global batch and share each time among eight images, retaining independent Gaussian noise for each image.
This allows the memory-key computation to be reused within a time group, while leaving per-image marginal objective in Eq.~\eqref{eq:shared-predictor-loss} unchanged.

During training, the original $H=8,d_h'=256$ LPLD run developed severe generation artifacts and degraded denoising during training.
We exclude that run from the displayed comparisons and use a separate recovery that resumes its epoch-225 checkpoint, including optimizer and EMA state, at learning rate $10^{-4}$.
The recovery reaches cumulative epoch 400 after 175 additional epochs and has FID 29.197 at that endpoint.
Only completed post-restart measurements are drawn for its trajectory; they are not spliced onto the original curve.

\paragraph{Evaluation.}
All measurements use EMA weights and the same 19,962 held-out test images.
We evaluate the velocity objective above and the raw, unclipped clean-image MSE across 32 logarithmically spaced normalized noise levels $\sigma\in[0.05,4]$.
The denoising score is the trapezoidal integral of this MSE with respect to $\log\sigma$, divided by $\log(4/0.05)$, so it is an average error over the specified noise range.
At each noise level we use $t=1/(1+\sigma)$ to map the normalized observation to the model's input convention.
FID-50k compares 50,000 generated samples with the common real-test feature statistics.
Generation uses 50 uniform time intervals from noise to data, with Heun updates and an Euler update on the last interval, and the same denominator floor of 0.05.
We evaluate the standard training trajectories every 25 epochs.
Figure~\ref{fig:shared-predictor-efficiency} uses epoch 400 for every displayed metric and model in the current snapshot, rather than selecting each model's best intermediate FID.

\subsubsection{Results and interpretation}

Figure~\ref{fig:shared-predictor-trajectories} shows that all reported families learn substantial denoising improvements, but their attained errors and generation trajectories differ.
Held-out losses generally decrease smoothly, whereas FID is less monotone.
The following comparisons use the common epoch-400 snapshot in Fig.~\ref{fig:shared-predictor-efficiency}.

\begin{figure}[!htbp]
    \centering
    \includegraphics[width=\textwidth]{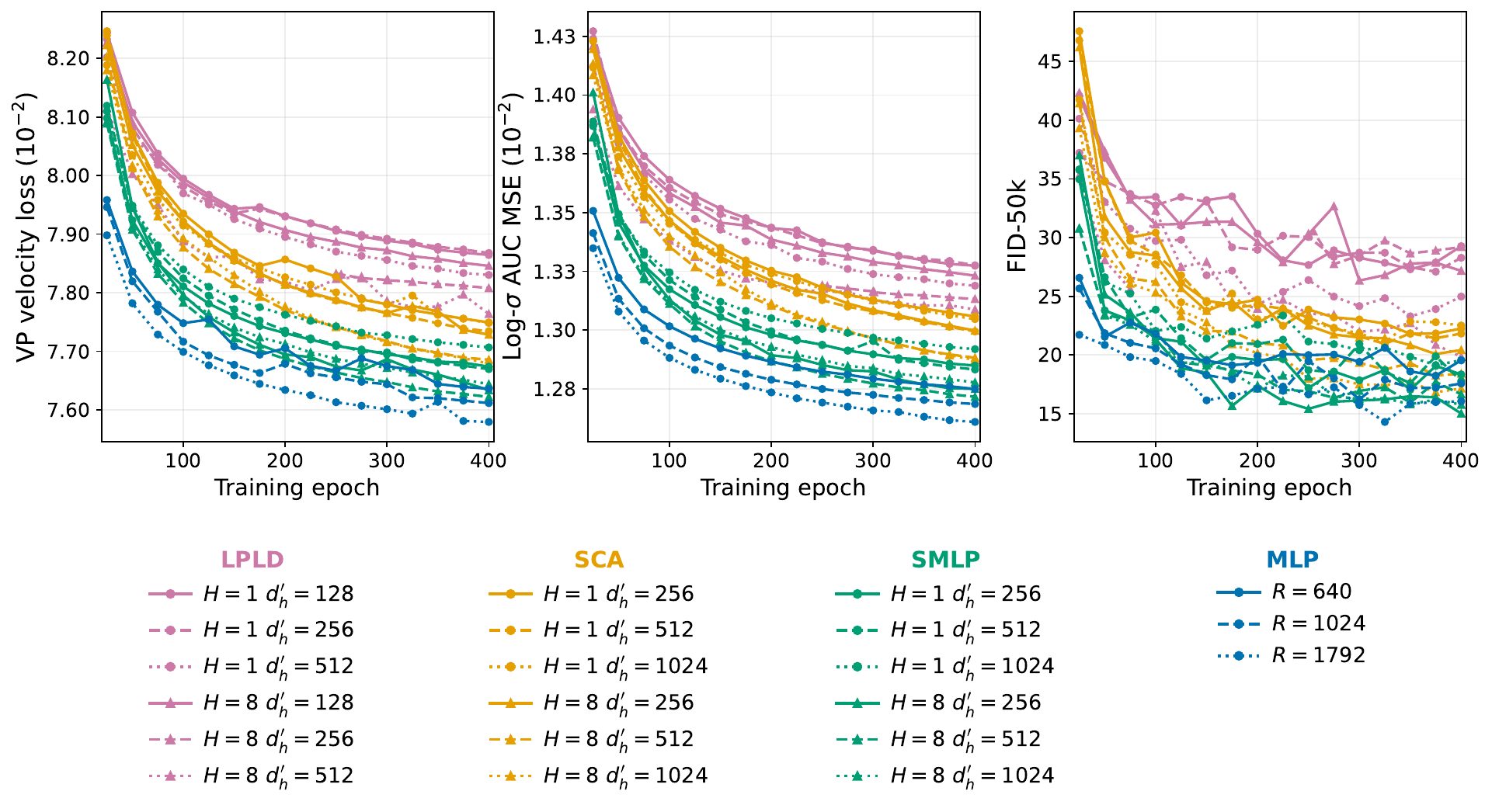}
    \caption{\textbf{Training trajectories for the patch-predictor comparison.}
    Held-out velocity loss, denoising log-$\sigma$ AUC MSE, and FID-50k on CelebA64.
    Colors identify predictor families; circles and triangles indicate $H=1$ and $H=8$.
    Solid, dashed, and dotted curves denote increasing $d_h'$ within each shared-memory family, or increasing hidden width $R$ for MLP.}
    \label{fig:shared-predictor-trajectories}
\end{figure}

\paragraph{(1) Different patches can retrieve from a shared patch memory.}
The patch-shared SCA and SMLP banks are more competitive than LPLD's position-specific image-shaped memory in the corresponding parameter-budget groups.
At the largest budget with $H=8$, LPLD has denoising AUC 0.013087 and FID 19.848, compared with 0.012873 and 17.145 for SCA, and 0.012777 and 15.767 for SMLP.
Thus, allocating a separate stored value vector to every patch position is not necessary for competitive denoising in this setting.
Sharing prototypes across positions permits a much larger bank---1,024 entries rather than 32---at a similar total parameter count.
Overall, the weaker LPLD results show that the most literal learned-image retrieval construction is not automatically the easiest or most effective mechanism for patch predictor.

\paragraph{(2) GELU retrieval is more effective here than softmax retrieval.}
At every tested $(H,d_h')$ pair, SMLP has lower held-out velocity loss, denoising AUC, and FID than SCA.
For example, at $H=8,d_h'=512$, SCA and SMLP have FIDs 18.413 and 16.652, respectively, with denoising AUCs 0.012882 and 0.012716.
Although SCA already includes the positive, query-dependent headwise gain $\eta_{\ell,h,t}(s)$ specified in Appendix~\ref{app:predictor-implementation}, its within-head softmax coefficients remain nonnegative and coupled by normalization.
In contrast, GELU allows unnormalized, potentially signed coefficients, providing a more flexible way to combine the shared values.
The observed advantage is consistent with this flexibility being useful during training.

\paragraph{(3) Multiple predictor heads are usually beneficial.}
Holding the stored width fixed, $H=8$ improves both held-out denoising measures over $H=1$ at every tested width in all three shared-memory families.
For SCA and SMLP, it also improves FID at every width.
For example, at $d_h'=1024$, increasing $H$ from one to eight reduces FID from 22.508 to 17.145 for SCA and from 19.556 to 15.767 for SMLP.
LPLD FID improves from 29.269 to 27.167 at $d_h'=128$ and from 24.976 to 19.848 at $d_h'=512$.
This suggests that multiple projections of the shared bank can be useful while keeping the total projected width fixed.

\paragraph{(4) The base memory can be shared across layers.}
The central comparison is between SMLP, whose base key--value bank is reused at every layer, and the ordinary MLP, whose vectors are independently stored in each layer.
For $H=8$ SMLP at $d_h'\in\{256,512,1024\}$, epoch-400 FIDs are 15.017, 16.652, and 15.767, compared with 19.550, 17.574, and 16.060 for the respective MLP controls.
At the smallest budget, their denoising AUCs are 0.0127499 and 0.0127496, and their velocity losses are 0.076336 and 0.076360.
At the middle and largest budgets, SMLP has approximately 0.24\% and 1.33\% higher denoising AUC, respectively, despite its lower FID.
Thus the shared-memory predictor is competitive in both denoising and generation, with a small denoising tradeoff at the larger budgets.
We emphasize that the result is a constructive demonstration that independently stored memory at every layer is not required for this level of performance.

\paragraph{Scope.}
We close by noting the scope of this experiment. 
It is a controlled study of unconditional face generation at $64\times64$ resolution, with modest model sizes and one training seed per configuration. 
These limits constrain how broadly the findings of our experiment can be applied.
We therefore do not claim that all DiT MLPs should share memory. 
Instead, the results support considering an explicit shared memory architecture when a common representation bank would be useful for inspection or future intervention. 
Our experiment shows that this architecture is feasible at least in the setting studied here.

\subsection{Experiment 3: Component compatibility in cross-dataset adaptation}
\label{app:exp3}

\paragraph{Motivation and design.}
The aggregator--predictor framework separates forming a patch's local context from predicting its clean value using learned representations.
Section~\ref{sec:attention-aggregator} associates the first operation with attention and Sec.~\ref{sec:patch-predictors} associates the second with the MLP.
Experiment 1 finds that DiT sensitivity and attention maps exhibit related noise-dependent broadening on real images; Experiment 2 shows that explicit retrieval from a shared memory supports competitive denoising and generation quality in a controlled CelebA64 study.
We now ask whether component compatibility across image domains is consistent with the proposed attention and MLP roles.
This question also matters when target data or adaptation compute are limited: fine-tuning methods must decide which pretrained parameters to update, and the best choice can depend on the source and target distributions \citep{kumari2023customdiffusion, choi2024autoft}.

The attention decomposition in \eqref{eqn:attention} gives a more specific hypothesis.
Within an isolated branch, QK controls spatial weights and VO transforms and transports features; the MLP admits a key--value representation for tokenwise updates.
This motivates the hypothesis that source QK may remain compatible with adapted VO and MLP parameters when source and target have similar spatial context, with larger shifts associated with reduced compatibility.
Source-QK compatibility may also be directional: a broadly trained source may support contextual relations needed by a narrower target, whereas the reverse transfer need not hold.
We probe these hypotheses by adapting QK, VO, and MLP jointly to each target, then replacing one or more adapted groups with their source weights without further optimization.
The eight resulting hybrids show how much each subset of jointly learned changes matters for reproducing the jointly adapted model.

\paragraph{Training and endpoint selection.}
We use the 20M-class DiT of Appendix~\ref{app:exp1-2} on $64\times64$ RGB images: $L=10$, $d_h=320$, $\tilde H=10$, and MLP hidden width 1280.
The twelve directions are CelebA$\to$AAHQ/FFHQ/AFHQ-v2/STL-10, STL-10$\to$AAHQ/FFHQ/AFHQ-v2/CelebA, and ImageNet64$\to$AAHQ/FFHQ/AFHQ-v2/CelebA.
Source training uses seed 0; adaptation starts from EMA checkpoints at source epochs 400, 800, and 800, respectively, and is repeated with two adaptation seeds.
Only the attention QK and VO projections along with MLP block weights are updated jointly.
Patch, position, and time embeddings, AdaLN, the output head, and other parameters stay fixed.
Freezing AdaLN keeps its learned time conditioning while holding the denoising schedule constant; it assumes that this aspect of time dependence transfers across datasets.

Training uses the common VP objective, global batch 32, AdamW with learning rate $2\times10^{-4}$, and a fresh optimizer for adaptation.
For source and target training-set sizes $N_{\mathrm{src}},N_{\mathrm{tgt}}$, one reference epoch is $s_{\mathrm{ref}}=\min(\lfloor N_{\mathrm{src}}/32\rfloor,\lfloor N_{\mathrm{tgt}}/32\rfloor)$ updates; adaptation stops by 200 reference epochs, with checkpoints every 25.
For each direction, we select the checkpoint minimizing the joint model's two-seed mean FID-50k on the monitored target split, with earlier epochs breaking ties, and use that same reference epoch for both seeds (Table~\ref{tab:exp3-endpoints}).
Each eight-state hybrid cube uses the selected checkpoints, EMA weights, and the same 1,024 Gaussian initial latents per state, sampled in FP32 with 50-step VP Heun.
Note that since the three architectural components are adapted jointly, these eight-state swaps measure compatibility among jointly adapted components; they do not measure what independently fine-tuned restricted models could learn.

\paragraph{Scope.}
Component transplantation diagnoses compatibility among parameter groups obtained by joint adaptation; it neither determines their performance when optimized in isolation nor uniquely assigns computation to components in a multilayer residual network.
Separate restricted-training comparisons are reported in Appendix~\ref{app:exp3-trajectories}.
Spatial MMD, EA, and SA are descriptive proxies for dataset geometry and endpoint agreement, not direct measurements of denoising context.
The twelve transfer directions share datasets and pretrained sources, and the component contrasts share cube terms; their correlations are descriptive rather than independent-sample estimates of a population relationship.

\begin{table}[!htbp]
\centering
\small
\begin{tabular}{lrr@{\qquad}lrr}
\toprule
Transfer & Epoch & FID & Transfer & Epoch & FID \\
\midrule
CelebA$\to$AAHQ & 75 & 40.767 & STL-10$\to$AAHQ & 150 & 40.592 \\
CelebA$\to$FFHQ & 100 & 25.252 & STL-10$\to$FFHQ & 100 & 28.231 \\
CelebA$\to$AFHQ-v2 & 200 & 44.821 & STL-10$\to$AFHQ-v2 & 200 & 49.119 \\
CelebA$\to$STL-10 & 200 & 68.050 & STL-10$\to$CelebA & 200 & 17.262 \\
ImageNet64$\to$AAHQ & 150 & 34.493 & ImageNet64$\to$FFHQ & 100 & 22.640 \\
ImageNet64$\to$AFHQ-v2 & 200 & 38.840 & ImageNet64$\to$CelebA & 100 & 14.620 \\
\bottomrule
\end{tabular}
\caption{Selected joint-adaptation endpoints. Epochs use the smaller source/target training set as reference; FID is the two-seed mean FID-50k used for checkpoint selection. STL-10 and ImageNet64 source checkpoints are both from epoch 800.}
\label{tab:exp3-endpoints}
\end{table}

\subsubsection{Spatial MMD for dataset similarity}\label{app:exp3-spatial-mmd}
To describe spatial similarity between unpaired real-image datasets, we compare distributions of whole-image patch self-similarity descriptors using maximum mean discrepancy (MMD).
For an individual clean image $x_0$, let $u_m(x_0)$ be its final DINOv2 ViT-B/14 feature at patch position $m$ after square resize to $224\times224$, yielding a $16\times16$ patch grid \citep{oquab2023dinov2}.
For each feature channel, subtract its mean over the 256 patch positions and divide by its spatial sample standard deviation plus $10^{-6}$; then normalize each patch vector to unit length, obtaining $\widehat u_m(x_0)$.
For patch positions $m,l\in[256]$, we form $K_{x_0}(m,l)=\widehat u_m(x_0)^\top\widehat u_l(x_0)$ and let $v(x_0)\in\mbb{R}^{32640}$ collect the upper off-diagonal entries of $K_{x_0}$ in a fixed raster order.
Thus each descriptor retains \emph{where} pairs of patches have similar features.
Patch self-similarity has been used to describe image structure \citep{tumanyan2022splicing}, and spatially normalized patch features have been studied in diffusion representation alignment \citep{singh2025spatialalignment}; applying MMD to these descriptors for cross-dataset comparison is our construction.

For descriptors $v,v'$, we use the Gaussian kernel $\kappa^{\mrm{sp}}(v,v')=\exp[-\|v-v'\|_2^2/(2\cdot32640\,\tau^2)]$.
We fix $\tau^2$ to the median of $\|v-v'\|_2^2/32640$ over all distinct pairs in a balanced calibration bank of 128 held-out images from each of six datasets, separately for normalized and raw patch descriptors.
Let $A=\{a_i\}_{i=1}^{n_A}$ and $B=\{b_j\}_{j=1}^{n_B}$ be independent evaluation banks of real-image descriptors. For distinct datasets, $n_A=n_B=1000$; for a diagonal cell, the banks are disjoint halves of one dataset and $n_A=n_B=500$. Conditional on the calibrated kernel bandwidth, the following is an unbiased estimator of the population squared MMD between the corresponding distributions of $v(X_0)$ \citep{gretton2012kernel}:
\begin{align}
\widehat{\mathrm{MMD}}_{\mathrm{sp}}^2(A,B)
&=\frac{\sum_{i\ne j}\kappa^{\mrm{sp}}(a_i,a_j)}{n_A(n_A-1)}
+\frac{\sum_{i\ne j}\kappa^{\mrm{sp}}(b_i,b_j)}{n_B(n_B-1)}
-\frac{2\sum_{i,j}\kappa^{\mrm{sp}}(a_i,b_j)}{n_A n_B}.
\label{eq:exp3-spatial-mmd}
\end{align}
Smaller values indicate more similar distributions of these fixed-grid descriptors, not necessarily more similar physical geometry or semantic content.
The estimate may be slightly negative for finite banks.
Figure~\ref{fig:exp3-dataset-geometry} shows the spatially normalized version, with each diagonal comparing two disjoint halves of the same dataset.

\begin{figure}[!htbp]
\centering
\includegraphics[width=0.8\textwidth]{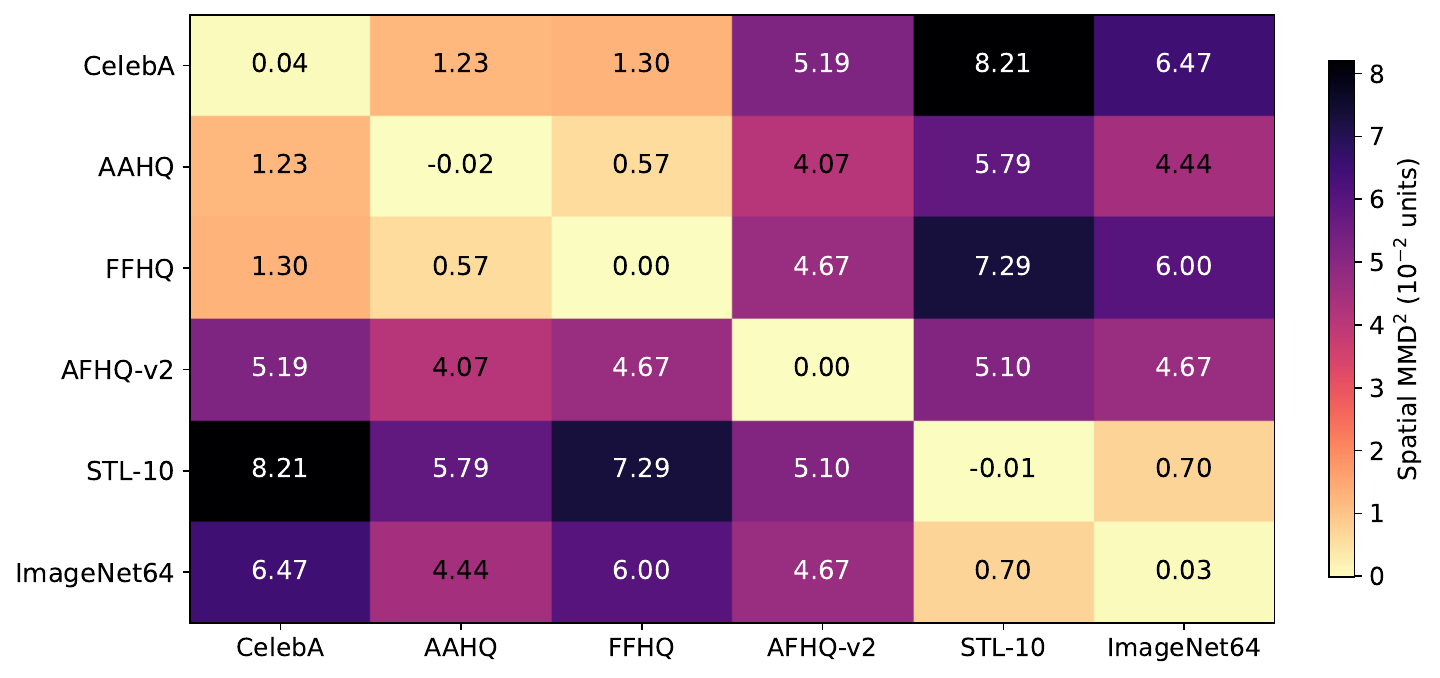}
\caption{\textbf{Spatial self-similarity distance between real datasets.} Unbiased Gaussian-kernel MMD$^2$ of spatially normalized, fixed-grid DINOv2 patch descriptors. Cell and colorbar values are in units of $10^{-2}$. Off-diagonal cells compare 1,000 unpaired images from each dataset; diagonal cells compare disjoint 500-image halves. Lower means more similar descriptor distributions.}
\label{fig:exp3-dataset-geometry}
\end{figure}

\subsubsection{Endpoint agreement and component contrasts}\label{app:exp3-endpoint-agreement}
For each transfer direction and adaptation seed, write $\mathbf b=(b_Q,b_V,b_M)\in\{0,1\}^3$ for a hybrid state in QK/VO/MLP order: 0 restores source weights and 1 retains jointly adapted weights. A bit string such as $011$ abbreviates $\mathbf b=(0,1,1)$.
Let $x_0^{(\mathbf b)}(\xi_i)$ be the clean image generated by hybrid $\mathbf b$ from the $i$th shared Gaussian sampling initialization $\xi_i$; $x_0^{(000)}(\xi_i)$ is the source output and $x_0^{(111)}(\xi_i)$ is the joint-adaptation output.
Visual comparisons use the same latents but can be ambiguous, so we also measure their agreement with the paired joint endpoint.
Let $f_{\mathrm{g}}(x_0)$ be the unit-normalized final DINOv2 class token and let $v(x_0)$ be the spatial descriptor above.
For $N=1024$ matched latents, define
\begin{align}
\mcl{E}_{\mathbf b}^{\mathrm{g}}&=\frac1N\sum_{i=1}^N\|f_{\mathrm{g}}(x_0^{(\mathbf b)}(\xi_i))-f_{\mathrm{g}}(x_0^{(111)}(\xi_i))\|_2^2,
&\mathrm{EA}_{\mathbf b}&=1-\frac{\mcl{E}_{\mathbf b}^{\mathrm{g}}}{\mcl{E}_{000}^{\mathrm{g}}},\\
\mcl{E}_{\mathbf b}^{\mathrm{sp}}&=\frac1N\sum_{i=1}^N\frac{\|v(x_0^{(\mathbf b)}(\xi_i))-v(x_0^{(111)}(\xi_i))\|_2^2}{32640},
&\mathrm{SA}_{\mathbf b}&=1-\frac{\mcl{E}_{\mathbf b}^{\mathrm{sp}}}{\mcl{E}_{000}^{\mathrm{sp}}}.
\label{eq:exp3-endpoint-agreement}
\end{align}
We compute EA only when $\mcl{E}_{000}^{\mathrm{g}}\ge0.01$ and SA only when $\mcl{E}_{000}^{\mathrm{sp}}\ge10^{-6}$.
Both scores set the source to 0 and the joint endpoint to 1; negative values are retained if a hybrid is farther from the joint endpoint than the source.
EA describes global feature agreement and SA describes agreement of patch self-similarity at fixed image coordinates.
SA and the dataset MMD use the same spatial descriptor but answer different questions: SA compares matched generated images, while MMD compares distributions of unpaired real images.

For either score $\mathrm{Agr}_{\mathbf b}\in\{\mathrm{EA}_{\mathbf b},\mathrm{SA}_{\mathbf b}\}$, finite differences across cube states quantify component dependence.
The penalty for reverting one adapted group from the joint model is $R_Q=\mathrm{Agr}_{111}-\mathrm{Agr}_{011}$, $R_V=\mathrm{Agr}_{111}-\mathrm{Agr}_{101}$, or $R_M=\mathrm{Agr}_{111}-\mathrm{Agr}_{110}$; we also use $R_{\min}=\min(R_Q,R_V,R_M)$ to ask whether all three reversions are costly.
The gain beyond adding the three isolated changes is
\begin{align*}
\mathrm{JNA}(\mathrm{Agr})&=\mathrm{Agr}_{111}-\mathrm{Agr}_{100}-\mathrm{Agr}_{010}-\mathrm{Agr}_{001}+2\mathrm{Agr}_{000},\\
\mathrm{TWI}(\mathrm{Agr})&=\mathrm{Agr}_{111}-\mathrm{Agr}_{110}-\mathrm{Agr}_{101}-\mathrm{Agr}_{011}
+\mathrm{Agr}_{100}+\mathrm{Agr}_{010}+\mathrm{Agr}_{001}-\mathrm{Agr}_{000}.
\end{align*}
Here JNA measures joint non-additivity relative to the three isolated gains, while TWI isolates the three-way interaction after subtracting pairwise interactions.
Positive $R_{\min}$ means each single reversion lowers endpoint agreement; positive JNA means the joint gain exceeds the sum of isolated gains.
All contrasts are computed within each seed's cube before averaging seeds.
Figure~\ref{fig:exp3-cube-agreement} gives the two agreement scores for every hybrid state and transfer direction; interpretation of their component patterns follows in the next section.

\begin{figure}[!htbp]
\centering
\includegraphics[width=\textwidth]{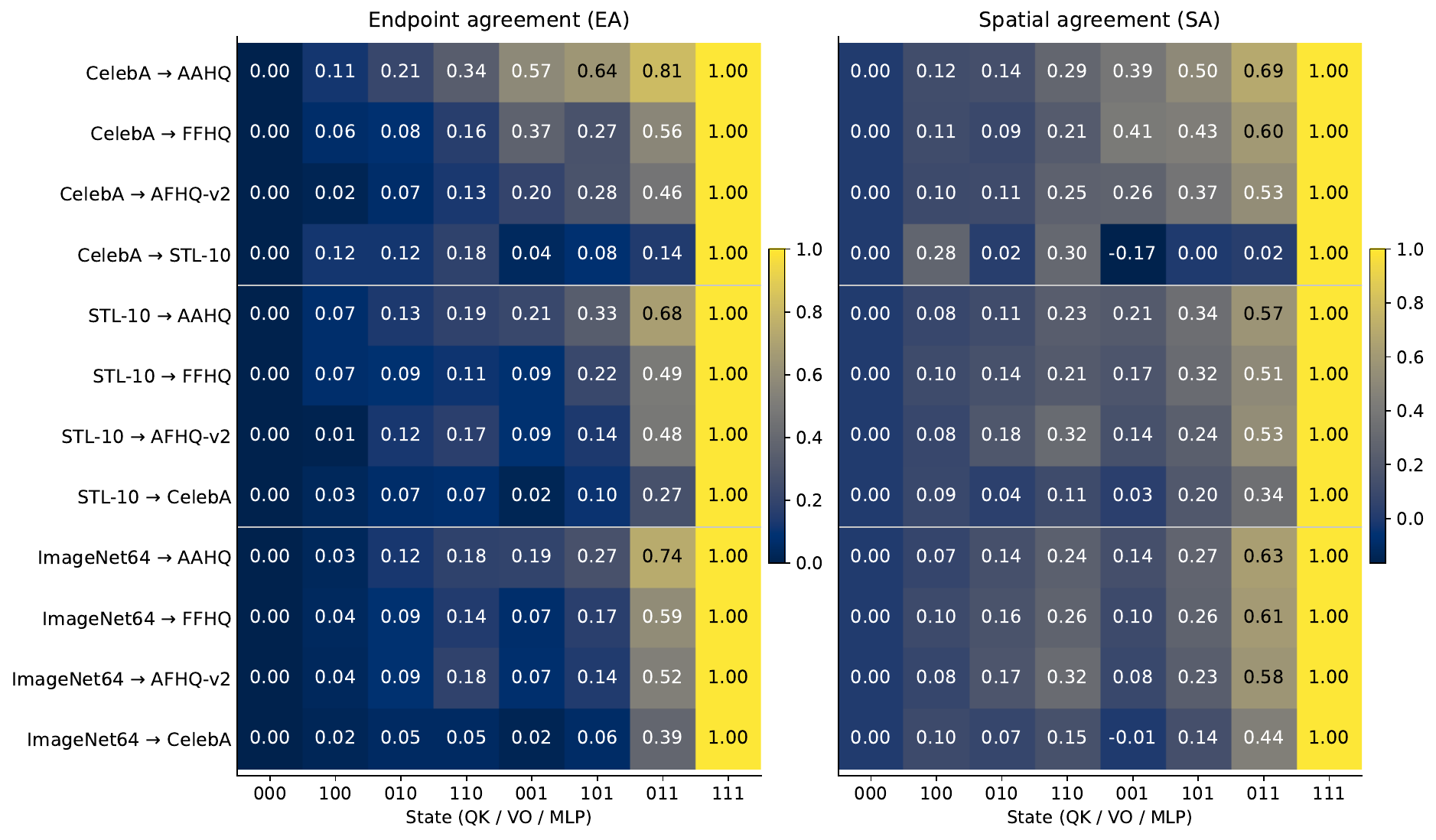}
\caption{\textbf{Agreement of all hybrid states with the jointly adapted endpoint.} Global-feature endpoint agreement EA (left) and relative spatial endpoint agreement SA (right), averaged over two adaptation seeds. Each row is a transfer direction, and columns give QK/VO/MLP bits (0: source, 1: jointly adapted). The same 1,024 noise latents are used for all eight states of each cube. Each score is 0 for the source and 1 for the joint endpoint; negative values are possible.}
\label{fig:exp3-cube-agreement}
\end{figure}

\subsubsection{Results and interpretation}
\label{app:exp3-results}

\paragraph{(1) Higher spatial similarity is associated with QK compatibility; larger shifts accompany coadaptation.}
Figure~\ref{fig:transfer-quantitative}(a) shows that spatial MMD$^2$ correlates with the QK reversion difference $\mathrm{Agr}_{111}-\mathrm{Agr}_{011}$ ($\rho_{\mathrm{EA}}=0.739$, $\rho_{\mathrm{SA}}=0.844$): source QK remains more compatible with adapted VO+MLP when dataset geometry is similar.
The selected CelebA$\to$AAHQ/FFHQ images retain recognizable faces in state 011, whereas CelebA$\to$STL-10 has little recognizable structure until joint adaptation (111; Figs.~\ref{fig:transfer-flagship} and~\ref{fig:exp3-extra-celeba64}).
Across all twelve nested state differences, this QK difference ranks first by the weaker of its two absolute correlations (Table~\ref{tab:exp3-nested-mmd}).
Figure~\ref{fig:exp3-coordination} extends the pattern to joint non-additivity, the minimum single-component reversion penalty, and three-way interaction: all increase with spatial distance ($\rho_{\mathrm{EA}}/\rho_{\mathrm{SA}}=0.634/0.655$, $0.739/0.844$, and $0.641/0.774$, respectively).
These observations are consistent with stronger component coadaptation across larger spatial shifts, subject to the dependence limitations noted above.

\begin{table}[!htbp]
\centering
\small
\setlength{\tabcolsep}{3pt}
\renewcommand{\arraystretch}{1.05}
\caption{Spatial distance and nested state differences. Spearman correlations with dataset spatial MMD$^2$ use 12 two-seed direction means for each score. Each row changes one component with its partners fixed. Rows are sorted by decreasing $\min(|\rho_{\mathrm{EA}}|,|\rho_{\mathrm{SA}}|)$. }
\label{tab:exp3-nested-mmd}
\begin{tabular*}{\linewidth}{@{\extracolsep{\fill}}llrr@{}}
\toprule
Interpretation & State difference & $\rho_{\mathrm{EA}}$ & $\rho_{\mathrm{SA}}$ \\
\midrule
\textbf{QK reversion penalty}$^\dagger$ & $\mathrm{Agr}_{111}-\mathrm{Agr}_{011}$ & \textbf{0.739} & \textbf{0.844} \\
MLP adaptation cost (VO adapted)$^\dagger$ & $\mathrm{Agr}_{010}-\mathrm{Agr}_{011}$ & 0.718 & 0.809 \\
VO adaptation cost (QK adapted) & $\mathrm{Agr}_{100}-\mathrm{Agr}_{110}$ & 0.771 & 0.683 \\
MLP-only adaptation cost & $\mathrm{Agr}_{000}-\mathrm{Agr}_{001}$ & 0.746 & 0.673 \\
MLP adaptation cost (QK adapted) & $\mathrm{Agr}_{100}-\mathrm{Agr}_{101}$ & 0.644 & 0.673 \\
VO reversion penalty & $\mathrm{Agr}_{111}-\mathrm{Agr}_{101}$ & 0.641 & 0.673 \\
VO-only adaptation cost & $\mathrm{Agr}_{000}-\mathrm{Agr}_{010}$ & 0.455 & 0.424 \\
QK adaptation cost (VO adapted) & $\mathrm{Agr}_{010}-\mathrm{Agr}_{110}$ & 0.676 & 0.378 \\
MLP reversion penalty & $\mathrm{Agr}_{111}-\mathrm{Agr}_{110}$ & 0.602 & 0.375 \\
QK adaptation gain (MLP adapted) & $\mathrm{Agr}_{101}-\mathrm{Agr}_{001}$ & 0.266 & 0.676 \\
QK-only adaptation gain & $\mathrm{Agr}_{100}-\mathrm{Agr}_{000}$ & 0.081 & 0.179 \\
VO adaptation cost (MLP adapted) & $\mathrm{Agr}_{001}-\mathrm{Agr}_{011}$ & 0.168 & 0.018 \\
\bottomrule
\end{tabular*}
\end{table}

\begin{figure}[!htbp]
    \centering
    \includegraphics[width=\textwidth]{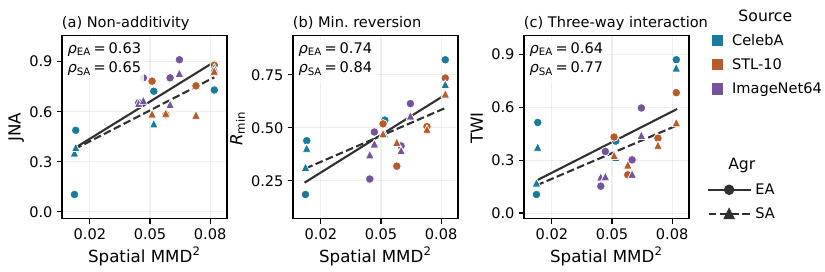}
    \caption{\textbf{Spatial distance and functional coordination.}
    Spatial MMD$^2$ versus (a) joint non-additivity, (b) minimum reversion penalty, and (c) three-way interaction, defined in Appendix~\ref{app:exp3-endpoint-agreement}.
    Points average two adaptation seeds; colors, markers, and lines follow Fig.~\ref{fig:transfer-quantitative}.}
    \label{fig:exp3-coordination}
\end{figure}

\paragraph{(2) VO dependence increases with functional coadaptation.}
In Fig.~\ref{fig:transfer-quantitative}(b), joint non-additivity correlates strongly with the VO reversion difference $\mathrm{Agr}_{111}-\mathrm{Agr}_{101}$ ($\rho_{\mathrm{EA}}=0.832$, $\rho_{\mathrm{SA}}=0.930$).
Table~\ref{tab:exp3-vo-nonadditivity} shows that this association exceeds those for reverting QK or MLP in both scores.
The generated images illustrate the same dependence: reverting VO in state 101 leaves fragmented texture for CelebA$\to$STL-10 and distorted outputs for STL-10/ImageNet64$\to$CelebA, compared with the coherent joint endpoints in 111 (Figs.~\ref{fig:transfer-flagship}, \ref{fig:transfer-stl10}, and~\ref{fig:transfer-imagenet64}).

\paragraph{(3) Broad-domain source QK retains partial compatibility on face targets.}
Figure~\ref{fig:transfer-quantitative}(c) shows partial source-QK agreement across AAHQ, FFHQ, and AFHQ for all three sources ($\mathrm{EA}_{011}=0.465$--$0.815$, $\mathrm{SA}_{011}=0.507$--$0.686$).
For STL-10$\to$CelebA, retained agreement $(\mathrm{EA}_{011},\mathrm{SA}_{011})=(0.267,0.340)$ exceeds the reciprocal CelebA$\to$STL-10 values $(0.143,0.021)$; ImageNet64$\to$CelebA also retains partial agreement $(0.387,0.444)$.
The selected natural-source images show varying face or animal structure with source QK in 011, while 111 generally yields clearer endpoints; some STL-10$\to$CelebA examples remain weak in 011 (Fig.~\ref{fig:transfer-stl10}).

\begin{table}[!htbp]
\centering
\small
\setlength{\tabcolsep}{3pt}
\renewcommand{\arraystretch}{1.05}
\caption{Component dependence and joint non-additivity. Spearman correlations between $\mathrm{JNA}_{\mathrm{Agr}}$ and each displayed difference use 12 two-seed direction means, separately for $\mathrm{Agr}=\mathrm{EA},\mathrm{SA}$. Note that since all quantities share cube terms, the differences between correlations are descriptive.}
\label{tab:exp3-vo-nonadditivity}
\begin{tabular*}{\linewidth}{@{\extracolsep{\fill}}llrr@{}}
\toprule
Reverted component & State difference & $\rho_{\mathrm{EA}}$ & $\rho_{\mathrm{SA}}$ \\
\midrule
VO & $\mathrm{Agr}_{111}-\mathrm{Agr}_{101}$ & \textbf{0.832} & \textbf{0.930} \\
QK & $\mathrm{Agr}_{111}-\mathrm{Agr}_{011}$ & 0.601 & 0.573 \\
MLP & $\mathrm{Agr}_{111}-\mathrm{Agr}_{110}$ & 0.664 & 0.063 \\
\bottomrule
\end{tabular*}
\end{table}

\subsubsection{Additional generated-image cubes}
\label{app:exp3-generated-cubes}
Figures~\ref{fig:exp3-extra-celeba64}--\ref{fig:transfer-imagenet64} extend Fig.~\ref{fig:transfer-flagship} with more latents and all eight component states per target.
They use the same layout and selected endpoints (Table~\ref{tab:exp3-endpoints}), with EMA weights, FP32 sampling, and 50-step VP Heun.

\begin{figure}[!htbp]
    \centering
    \includegraphics[width=\textwidth]{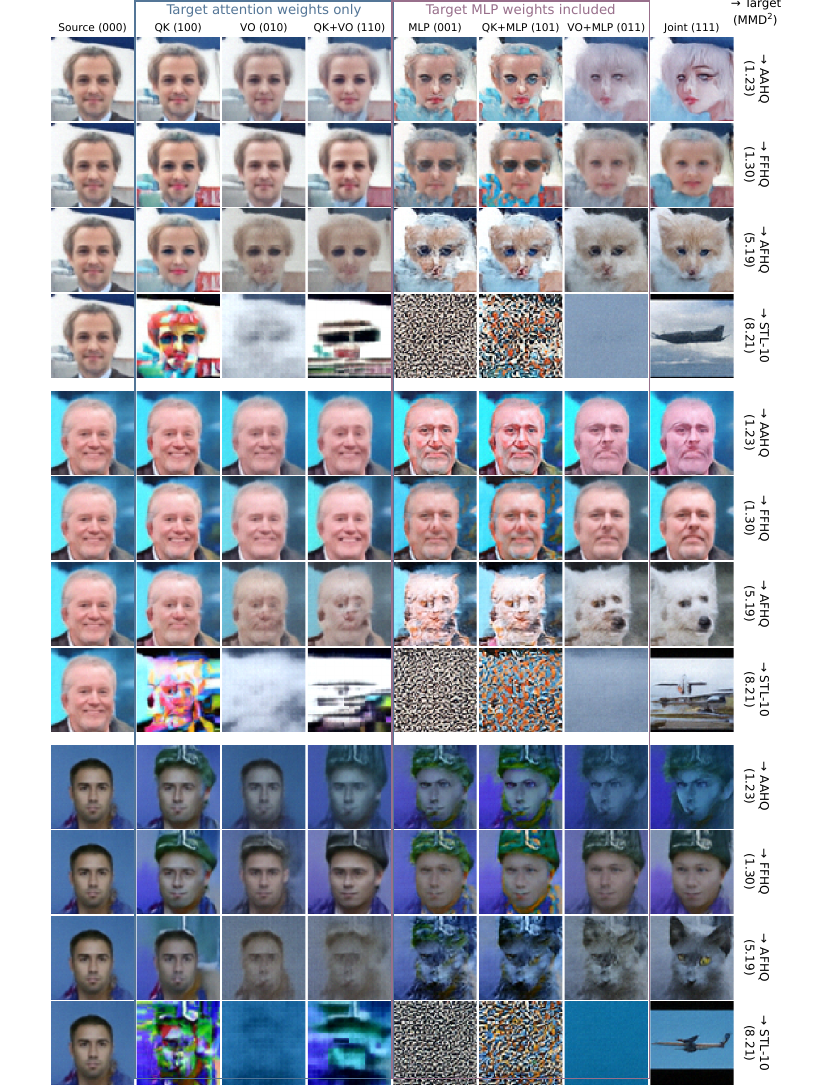}
    \caption{CelebA source: three additional latents across four targets.}
    \label{fig:exp3-extra-celeba64}
\end{figure}

\begin{figure}[p]
    \centering
    \includegraphics[width=\textwidth]{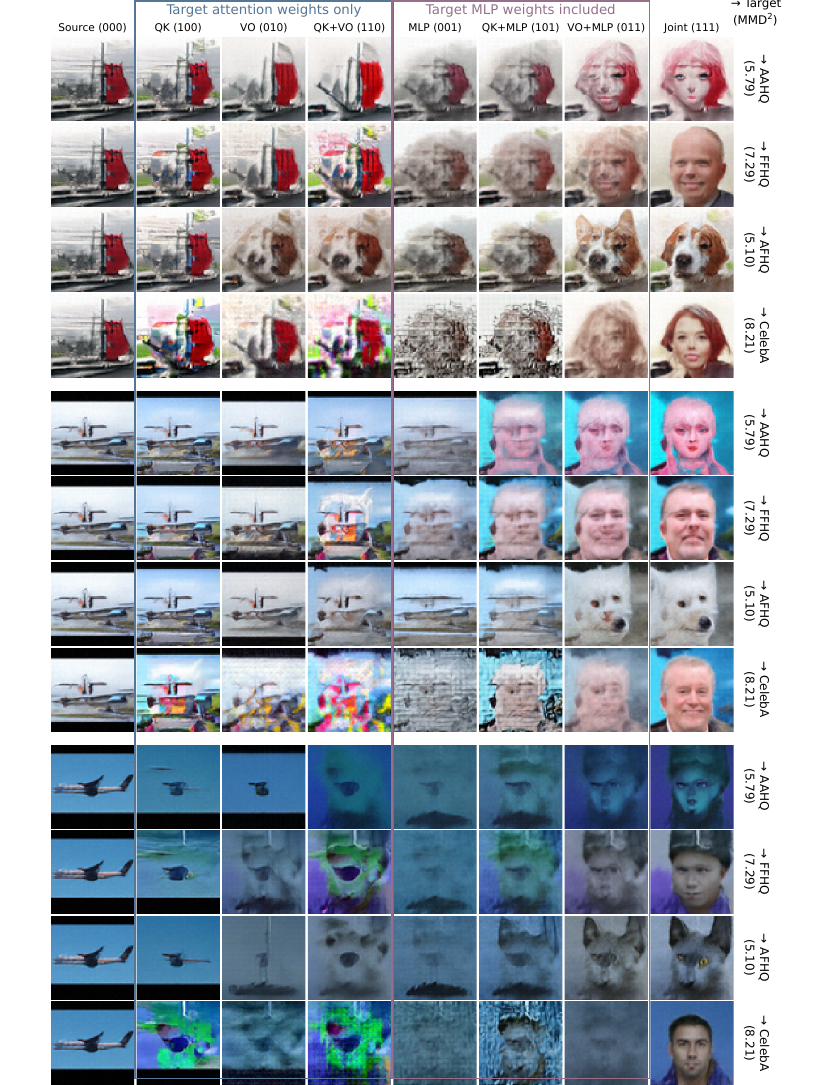}
    \caption{STL-10 source: three latents across four targets.}
    \label{fig:transfer-stl10}
\end{figure}

\begin{figure}[p]
    \centering
    \includegraphics[width=\textwidth]{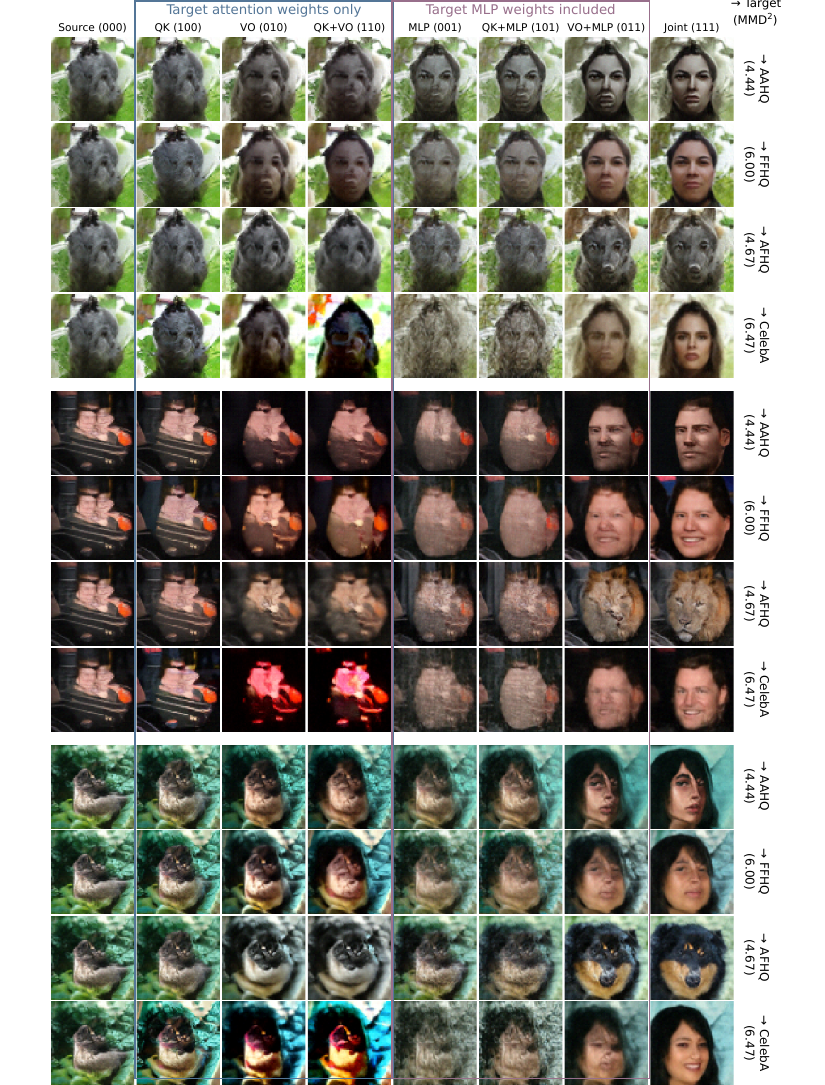}
    \caption{ImageNet64 source: three latents across four targets.}
    \label{fig:transfer-imagenet64}
\end{figure}

\clearpage
\subsubsection{Training trajectories and restricted fine-tuning}
\label{app:exp3-trajectories}
We compare joint QK/VO/MLP fine-tuning (111) with training all parameters from scratch and with independent fine-tuning restricted to QK or QK+VO, using the same target-domain update budget (Figs.~\ref{fig:exp3-trajectories-celeba64}--\ref{fig:exp3-trajectories-imagenet64}).
Joint fine-tuning is comparable to or better than scratch training in held-out denoising and usually in FID-50k; STL-10$\to$CelebA is an exception for FID at the selected joint endpoint.
QK-only and QK+VO-only fine-tuning generally lag further behind on larger domain shifts under the tested training protocol and update budget.
Each target has one row, with denoising AUC, VP velocity loss, and FID-50k in columns; lower is better.
The two loss axes use $10^{-2}$ units, as in Fig.~\ref{fig:shared-predictor-trajectories}.
Solid and dashed curves denote adaptation seeds 0 and 1, and dotted scratch curves use seed 0.
Vertical lines mark the selected joint endpoints; source pretraining is excluded from the target-domain update count.

\begin{figure}[!htbp]
    \centering
    \includegraphics[width=\textwidth]{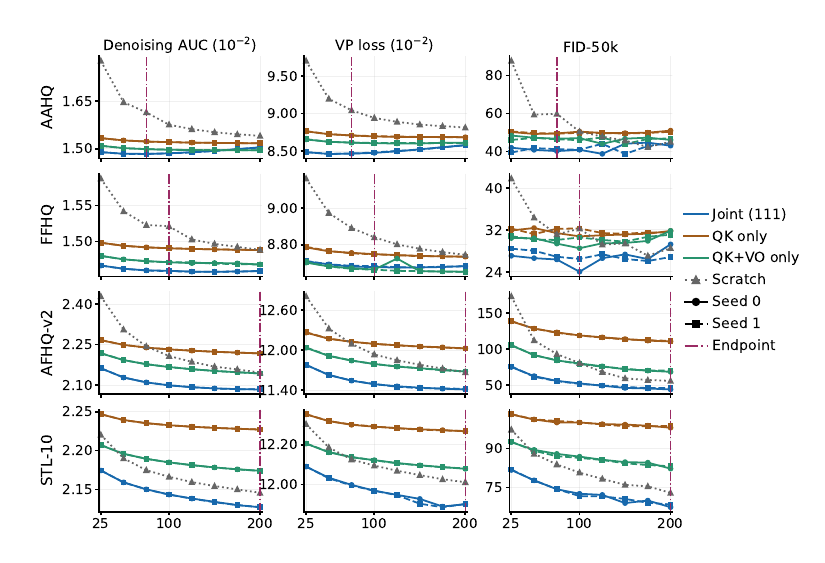}
    \caption{\textbf{CelebA-source adaptation trajectories.} Rows are AAHQ, FFHQ, AFHQ-v2, and STL-10 targets.}
    \label{fig:exp3-trajectories-celeba64}
\end{figure}

\begin{figure}[!htbp]
    \centering
    \includegraphics[width=\textwidth]{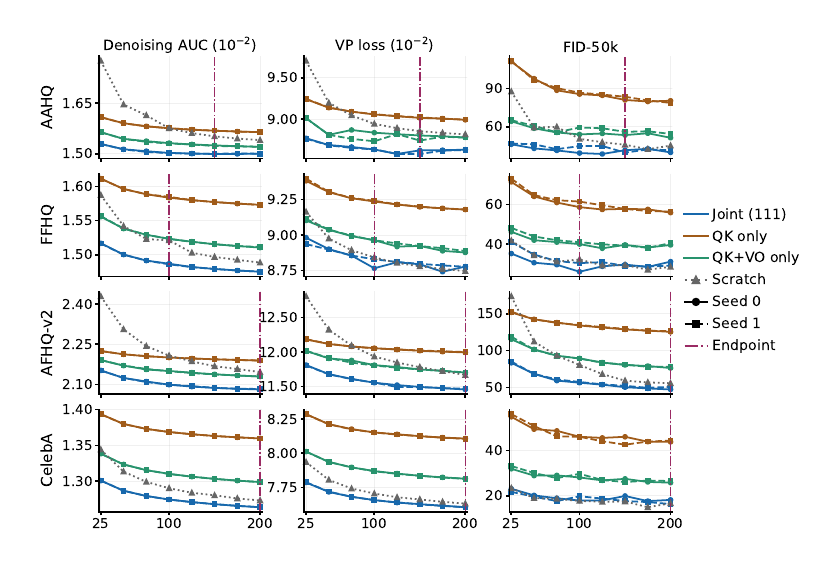}
    \caption{\textbf{STL-10-source adaptation trajectories.} The source checkpoint is from epoch 800.}
    \label{fig:exp3-trajectories-stl10}
\end{figure}

\begin{figure}[!htbp]
    \centering
    \includegraphics[width=\textwidth]{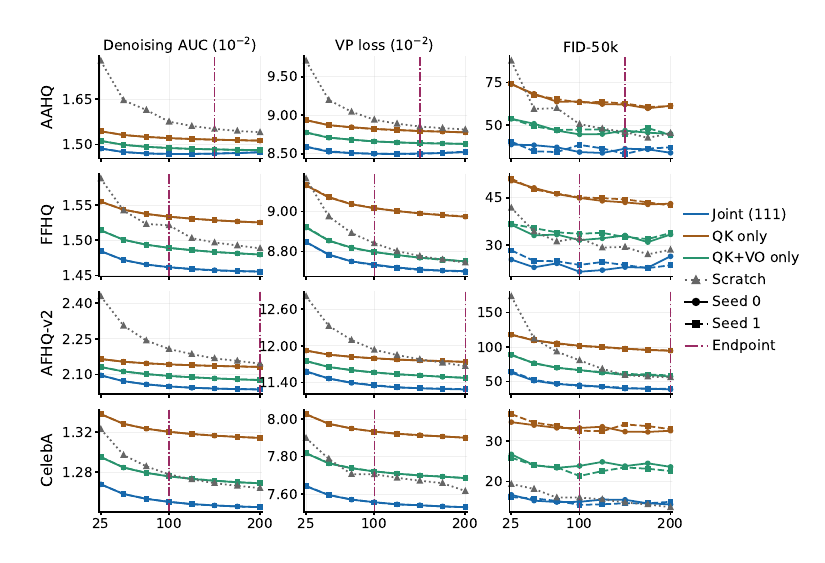}
    \caption{\textbf{ImageNet64-source adaptation trajectories.} The source checkpoint is from epoch 800; rows are AAHQ, FFHQ, AFHQ-v2, and CelebA targets.}
    \label{fig:exp3-trajectories-imagenet64}
\end{figure}

\clearpage
\endgroup
\end{document}